\documentclass[]{article}
\usepackage[utf8]{inputenc} 
\usepackage[T1]{fontenc}    
\usepackage{url}            
\usepackage{amsfonts}       
\usepackage{nicefrac}       
\usepackage{microtype}      
\usepackage[dvipsnames]{xcolor}         
\usepackage{pifont}
\usepackage{subcaption}
\usepackage{fullpage}
\usepackage{amsmath}
\usepackage{amsthm}
\usepackage{wrapfig}
\usepackage{diagbox}
\usepackage{algorithm}
\usepackage{algpseudocode}

\usepackage{mleftright}
\usepackage[authoryear,round]{natbib}
\definecolor{ForestGreen}{rgb}{0.1333,0.5451,0.1333}
\definecolor{DarkRed}{rgb}{0.8,0,0}
\definecolor{Red}{rgb}{0.9,0,0}
\usepackage[linktocpage=true,
pagebackref=true,colorlinks,
    urlcolor=black,
linkcolor=DarkRed, citecolor=ForestGreen,
bookmarks,bookmarksopen,bookmarksnumbered]
{hyperref}

\usepackage{graphicx}       
\usepackage{amssymb}
\usepackage{tikz}
\usetikzlibrary{automata}
\usetikzlibrary{arrows.meta, positioning, shapes.geometric}
\usetikzlibrary{decorations.pathreplacing,}
\usetikzlibrary{shapes, shadows, fit, calc, backgrounds}
\definecolor{gold}{RGB}{212,175,55}

\usepackage[capitalize,noabbrev]{cleveref}
\usepackage{aliascnt}

\usepackage{dsfont}
\theoremstyle{plain}
\Crefname{assumption}{Condition}{Conditions}
\crefname{assumption}{Condition}{Conditions}

\theoremstyle{plain}

\newtheorem{theorem}{Theorem}

\newtheorem{definition}{Definition}
\newtheorem{lemma}{Lemma}

\newtheorem{remark}{Remark}

\theoremstyle{definition}

\newtheorem{example}{Example}

\newcommand{\R}{\mathbb R}

\renewcommand{\tilde}{\widetilde}

\usepackage{bm}
\usepackage{dsfont}
\usepackage{nicematrix}

\usepackage{todonotes}
\usepackage{authblk}

\title{Learning through Internalization}
\date{\today}
\author[1]{Nikolaos Tsilivis}
\author[2]{\quad Nirmit Joshi}
\author[3]{\quad Marko Medvedev}
\author[1]{\quad Julia Kempe}
\author[2]{\quad Nathan Srebro}

\affil[1]{\small New York University. \texttt{nt2231@nyu.edu}}
\affil[2]{\small Toyota Technological Institute at Chicago.}
\affil[3]{\small University of Chicago.}

\begin{document}

\maketitle

\begin{abstract}

    We study internalization processes, by which neural-network-based systems absorb an explicit computational procedure into their own weights, and how they facilitate learning. We investigate how transformers internalize the simulation of semiautomata by internalizing chain-of-thought (CoT) tokens, which classes of semiautomata are harder to internalize, and expose the flip side of internalization, that is, a progressive degradation of out-of-distribution performance. We then provide the first provable analysis of successful internalization: for the task of learning parities, we show that a simplified one-layer transformer provably first learns the target with explicit CoT supervision and then internalizes the autoregressive generation as CoT tokens are progressively removed, learning to directly compute the parity. Learning such a representation directly from data without CoT supervision is computationally hard. Finally, we discuss how learning through internalization relates to the \textit{Positive Distribution Shift} phenomenon recently introduced by~\citet{Med+26}.
\end{abstract}

\section{Introduction}\label{sec:intro}

Humans often find it difficult to learn some tasks directly. When a teenager wants to learn how to drive a car, they benefit from a teacher who meticulously explains the many intermediate steps of the process and does not ask them to start practicing driving right off. Driving students begin by learning each of these subtasks (e.g. the role of each pedal, how to press the right button or clutch, how to steer the wheel) and only later do they become ready to operate the vehicle on the road. In fact, only after repetition and effort do some manage to drive without actively thinking about it. For these people, driving has become automatic, and they do not have to always think about all the intermediate steps. They have internalized their execution and have mastered driving, which was the task they ultimately cared about; all the other subtasks were merely stepping stones to it. And while doing things faster does not always mean doing them better, since doing things automatically can make people more prone to mistakes under varying conditions, we still recognize this internalization process as something positive and desirable. The whole process (explicit subtask learning, composition and finally internalization) permits the learning of difficult tasks that seem otherwise impenetrable, while the final stage of internalization allows humans to reduce execution time, enabling them to move on to other pursuits.

In machine learning, there have been many instances where a neural network has benefited from such a process. We define \textbf{internalization} as the process by which a neural-network-based system that initially implements a procedure over many steps has its weights reconfigured so that it implements the same procedure in far fewer steps. For example, neural-network based systems can learn how to play board games by internalizing the output of an expensive next-move search in their own weights, as in Samuel’s original work on checkers~\citep{Sam59} or Google’s AlphaGo Zero~\citep{Sil+17}.

\begin{figure}
    \centering
    \begin{minipage}{0.4825\textwidth}
        \begin{tikzpicture}[
            scale=0.55,
            transform shape,
            >=Latex,
            font=\small,
            box/.style={
                draw,
                thick,
                minimum width=1.3cm,
                minimum height=1.0cm,
                align=center
            },
            cloud/.style={
                draw,
                thick,
                rounded corners=8pt,
                minimum width=2.2cm,
                minimum height=1.5cm,
                align=center
            },
            lab/.style={font=\small\itshape},
            every node/.style={inner sep=2pt}
        ]
        
        \node[lab] at (4.5,5.75) {Explicit to Implicit CoT};
        \node[lab] at (4.5,5.4) {\citep{DCS24}};

        \node[box] (tf1) at (0,3) {TF};
        \node[box] (tf2) at (3,3) {TF};
        \node[box] (tf3) at (6,3) {TF};
        \node[box] (tf4) at (10,3) {TF};
        
        \draw[->,thick] (0,2.1) -- (tf1.south);
        \node at (0,1.8) {\texttt{10101}};
        
        \draw[->,thick] (3,2.1) -- (tf2.south);
        \node at (3,1.8) {\texttt{10101\textcolor{red}{0}}};
        
        \draw[->,thick] (6,2.1) -- (tf3.south);
        \node at (6,1.8) {\texttt{10101\textcolor{red}{00}}};
        
        \draw[->,thick] (10,2.1) -- (tf4.south);
        \node at (10,1.8) {\texttt{10101}};
        
        \draw[->,thick] (tf1.north) -- ++(0,0.6);
        \node at (0,4.5) {$\texttt{\textcolor{red}{0}}$};
        
        \draw[->,thick] (tf2.north) -- ++(0,0.6);
        \node at (3,4.5) {$\texttt{\textcolor{red}{0}}$};
        
        \draw[->,thick] (tf3.north) -- ++(0,0.6);
        \node at (6,4.5) {$\texttt{\textcolor{gold}{1}}$};
        
        \draw[->,thick] (tf4.north) -- ++(0,0.6);
        \node at (10,4.5) {$\texttt{\textcolor{gold}{1}}$};
        
        \draw[->,thick]
            ($(0,4.25)+(0.25,0.25)$)
            to[out=0,in=270]
            (3.,1.35);
        
        \draw[->,thick]
            ($(3,4.25)+(0.25,0.25)$)
            to[out=0,in=270]
            (6.,1.35);
        
        \draw[->,thick] (7.0,3.0) -- (8.8,3.0);
        \node[lab] at (7.9,3.35) {internalization};
        
        \begin{scope}[on background layer]
            \filldraw[
                rounded corners=4pt,
                draw=red!70!black,
                fill=red!45,
                fill opacity=0.10,
                line width=0.8pt
            ]
            (-.9,.75) rectangle (6.75,4.85);
    
            \filldraw[
                rounded corners=4pt,
                draw=blue!70!black,
                fill=blue!45,
                fill opacity=0.10,
                line width=0.8pt
            ]
            (9.25,.75) rectangle (10.75,4.85);
        \end{scope}
        \end{tikzpicture}
    \end{minipage}
    \begin{minipage}{0.51\textwidth}
        \includegraphics[width=\linewidth]{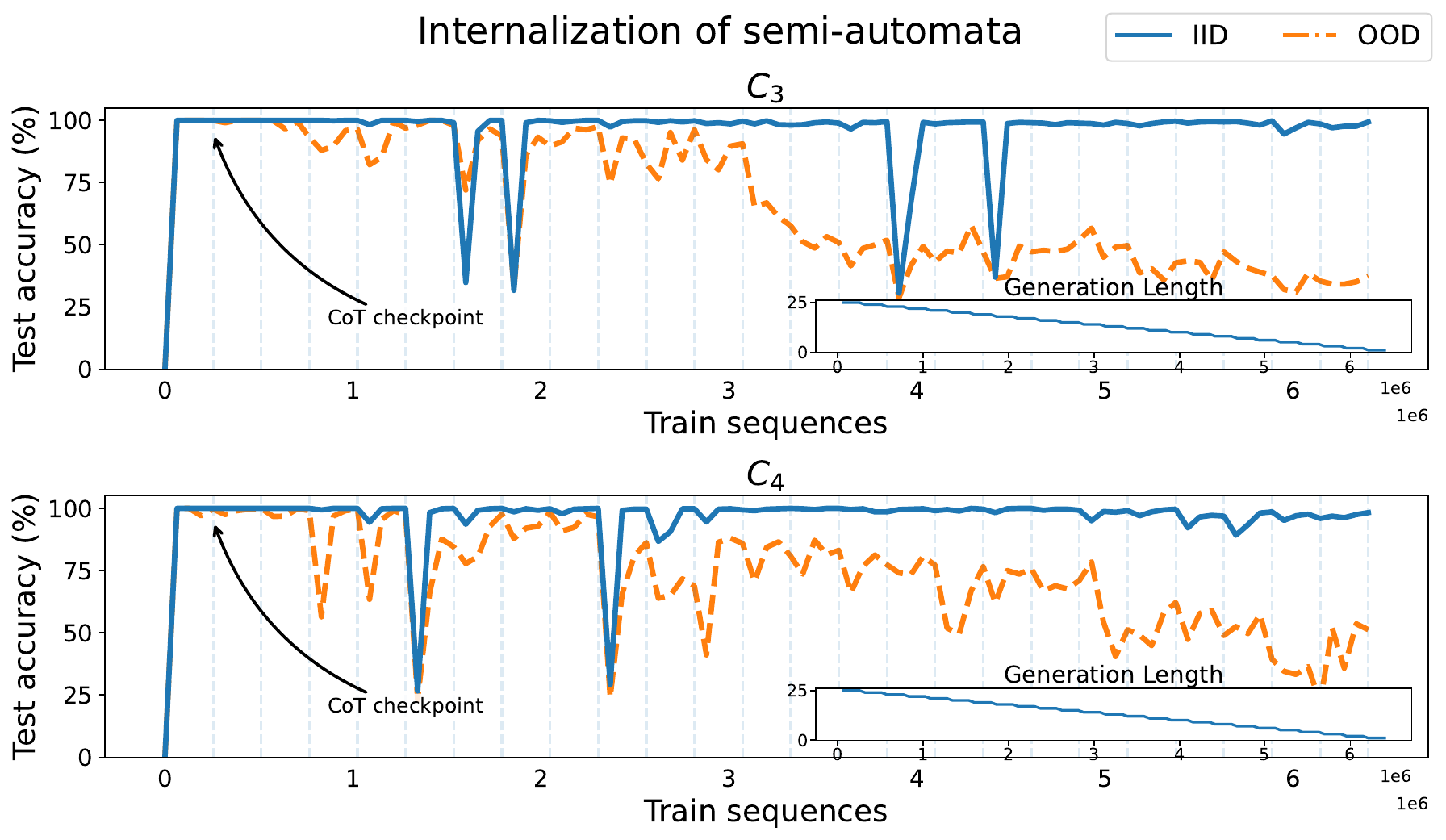}
    \end{minipage}
    \caption{\textbf{Internalization of chain-of-thought in autoregressive transformers.} \textit{Left}: An illustration of the outcome of the internalization process. \textit{Right}: Two examples of internalization of semiautomaton simulation over $T=25$ steps. As the transformer internalizes more chain-of-thought tokens, its out-of-distribution test accuracy (induced by a different sampling process for input tokens) degrades. Each dashed line corresponds to one additional token dropped from the chain-of-thought.}
    \label{fig:ad_figure}
\end{figure}
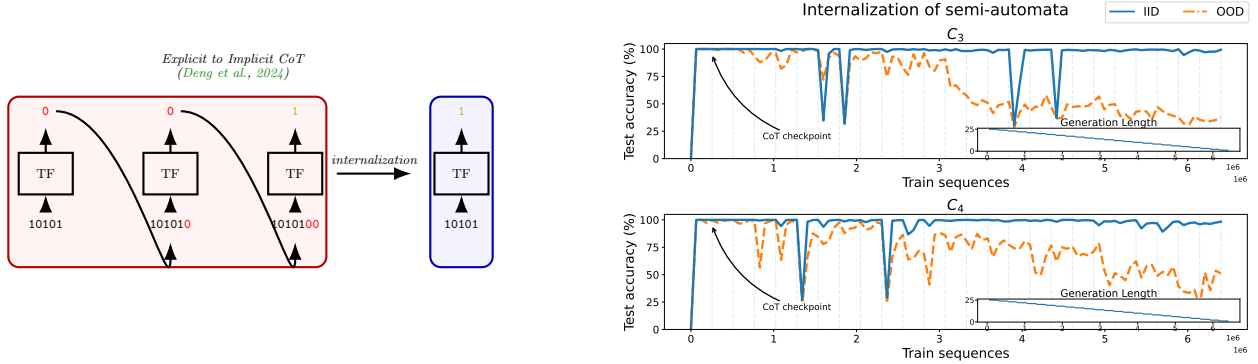

More recently, several works have explored internalization schemes for machine reasoning with autoregressive large language models (LLMs)~\citep{Deng+23,DCS24,Yu+24}. While training LLMs to solve complex problems ``directly'', i.e., without chain-of-thought (CoT) supervision, might often be difficult, these works present an alternative: a transformer first learns to autoregressively generate answers based on step-by-step CoT supervision, and is then later optimized to internalize the outcome of the lengthy generation in its weights, thereby generating responses with fewer, or perhaps even just one, forward pass (Figure 1, left). In particular, \citet{DCS24} introduced a curriculum method for the internalization part that progressively removes CoT tokens during training, and showed that the whole process (explicit CoT training and then gradual internalization) can be very effective in arithmetic tasks, such as number multiplication, which otherwise appear impenetrable if attempted to be learned directly. 

While this approach appears powerful and has seen success on several tasks~\citep{DCS24}, other works have reported mixed results on more sophisticated reasoning problems~\citep{Yu+24, Lin+25}, with ~\citet{Lin+25} identifying that ``internalized reasoning’’ can overfit to specific patterns and may fail to generalize more broadly. Motivated by the promising empirical results on internalization of CoT, in our work we pose and attempt to answer several questions regarding internalization. We ask:
\begin{center}
    \emph{Which tasks are easy to internalize, and which tasks are difficult? How do network representations evolve during internalization? Does internalization lead to representations that mirror explicit reasoning but unfold vertically across network layers, or does it instead lead to shortcut solutions that exploit the model's parallel computation? What are the principles that enable efficient learning through internalization?}
\end{center}

\paragraph{Our contribution.} First, in \Cref{sec:prelim}, we define what constitutes an internalization process throughout the paper, and explain how internalization can sometimes be used for learning purposes. In \Cref{sec:automata}, we empirically study internalization of CoT in a canonical sequential reasoning setting: simulating a semiautomaton for \(T\) timesteps and outputting its final state. Then, in \Cref{sec:parity_one_layer_tf}, we provide a theoretical result on the sparse parity task that shows that internalizing CoT can enable a model to learn representations that are computationally hard to find directly. Finally, in~\Cref{sec:pds-learning-internalization}, we explore connections between learning through internalization and a recently introduced \textit{Positive Distribution Shift} learning framework~\citep{Med+26}. Our main results are summarized below.
\begin{itemize}
	\item We identify several \textit{method specific factors}, such as architecture or curriculum choices, that affect the probability of successful internalization. For example, we observe that, across different semiautomata tasks, internalization of CoT is consistently more successful with wider rather than deeper architectures at a fixed parameter budget. This perhaps contradicts the mental picture that implicit reasoning tokens are ``represented vertically'' in the network post-internalization.
	\item We identify \textit{task specific factors} that control whether a semiautomaton is easy or hard to internalize. In particular, amongst the class of modular counter semiautomata those with \textit{prime} order appear to be the most difficult to internalize. We conjecturally relate this to the ability of transformers to compactly represent counter semiautomata whose order has small prime-power factors (\Cref{thm:log_depth_log_width_prime_informal}).
	\item We expose the flip side of internalization: much like humans, transformers that internalize operations often tend to perform worse in OOD settings (Figure~\ref{fig:ad_figure}, right). In particular, the internalization process for semiautomata can lead the model to learn shortcut solutions~\citep{Liu+23}, which can generalize worse than the representation outputted by CoT training.
	\item We provide the first provable analysis of learning through internalization: for the sparse parity task, we prove that a stylized one-layer transformer with linear attention and a fixed nonlinear MLP first learns the target parity of hidden support, using explicit CoT, and then internalizes this computation as CoT tokens are removed one at a time (\Cref{thm:provable-parity-main}), thereby directly computing the parity. Learning such a representation directly from data without internalization of CoT is computationally hard.
\end{itemize}
Finally, throughout the paper, and in greater detail in~\Cref{sec:pds-learning-internalization}, we formally explain how learning through internalization provides a viable strategy for sidestepping the hardness of ``direct’’ learning. We explicitly connect this perspective to the Positive Distribution Shift phenomenon~\citep{Med+26}. Motivated by this, we explore a new internalization scheme for Multi Layer Perceptrons that enables learning hard parity functions (\Cref{subsec:internalization-MLP}). For autoregressive transformers, we further empirically demonstrate that training on a single mixture distribution of ``short'' and ``long'' CoTs may already suffice for learning, without requiring an explicit curriculum that progressively removes tokens (\Cref{subsec:mixture}).

\section{Internalization Processes}\label{sec:prelim}

We start off by defining what we call an internalization process throughout the paper. We then describe the main learning phenomenon of study: instead of learning a distribution directly, we first learn a separate distribution with a ``slow'' system and then later internalize the ``slow'' system within a ``faster'' one.

\subsection{Definition of internalization} 

Let $\Sigma$ be an alphabet, $\mathcal{X} \subseteq \Sigma^\ast$ be an input space, $\Theta$ be an encoding (parameter) space, and let $\mathcal{Y} \subseteq \Sigma^\ast$ be an output space.
Let $\mathcal{H} = \left\{ h_\theta \middle | \theta \in \Theta \right\} \subseteq \mathcal{Y}^{\mathcal{X}}$ be a hypothesis class of functions each indexed by $\theta \in \Theta$, and, for any $\theta \in \Theta$, let $\mathfrak{h}_\theta$ denote an algorithm that given input $x \in \mathcal{X}$ computes $h_\theta(x)$.
Furthermore, for any $\theta \in \Theta$, denote the runtime of $\mathfrak{h}_\theta$ on input $x \in \mathcal{X}$ by $\mathrm{Time}_{\mathfrak{h}_\theta}(x)$. Let $P: \mathcal{Y}^{\mathcal{X}} \to \mathcal{Y}^{\mathcal{X}}$ be a procedure that, given access to a hypothesis $h_\theta$, produces a new mapping $P(h_\theta): \mathcal{X} \to \mathcal{Y}$. For any $\theta \in \Theta$, let $\mathfrak{P}(\mathfrak{h}_\theta)$ be an algorithm that given input $x \in \mathcal{X}$ computes $P(h_\theta)(x)$ and denote its runtime by $\mathrm{Time}_{\mathfrak{P}(\mathfrak{h}_\theta)}(x)$. We assume that $\mathfrak{P}(\mathfrak{h}_\theta)$ has only query access to $\mathfrak{h}_\theta$ and does not know $\mathcal{H}$ nor can it modify $\theta$.

\begin{definition}[Internalization]\label{def:internalization}
    Fix $\theta \in \Theta$. If there exists $\theta^\prime \in \Theta$ with $\theta^\prime \neq \theta$ such that
    \begin{enumerate}
        \item \textbf{Output agreement:}
        \begin{equation}
            h_{\theta^\prime}(x) = P(h_\theta)(x) \qquad \forall x \in \mathcal{X}.
        \end{equation}
        \item \textbf{$\mathfrak{h}_{\theta^\prime}$ is faster than $\mathfrak{P}(\mathfrak{h}_\theta)$:} For all $x \in \mathcal{X}$, it holds:
        \begin{equation}
            \mathrm{Time}_{\mathfrak{h}_{\theta^\prime}}(x) < \mathrm{Time}_{\mathfrak{P}(\mathfrak{h}_\theta)}(x),
        \end{equation}
    \end{enumerate}
    then we say that $\mathfrak{h}_{\theta^\prime}$ \textbf{internalizes} $\mathfrak{P}(\mathfrak{h}_\theta)$.
\end{definition}

In words, if we have a program (e.g., a neural network) $\mathfrak{h}_\theta$ and some other system that uses this program (potentially many times) to produce a new program $\mathfrak{P}(\mathfrak{h}_\theta)$, then we say that $\mathfrak{h}_{\theta^\prime}$ internalizes $\mathfrak{P}(\mathfrak{h}_{\theta})$ if there exists a new, fast, program $\mathfrak{h}_{\theta^\prime}$, induced by a hypothesis $h_{\theta^\prime}$ belonging to the same class of hypotheses $\mathcal{H}$ as the original $h_\theta$, that agrees with the previous slow $\mathfrak{P}(\mathfrak{h}_\theta)$ on all inputs. 
Some remarks are in order.
\begin{remark}
    We assume that the initial slow program and the final fast one are induced by hypotheses belonging to the same class. We believe this is analogous to the human case, where a single brain is reconfigured during internalization. This choice also distinguishes internalization from other notions previously studied in machine learning, such as model distillation~\citep{HVD15}, which is typically viewed as distilling a source class $\mathcal{F}$ into a target class $\mathcal{G}$~\citep{Boi24}.
\end{remark}
\begin{remark}
    In our definition, we prohibit the program $\mathfrak{P}(\mathfrak{h}_\theta)$ from interacting with $\mathcal{H}$ or updating $\theta$, and we allow it only query access to $\mathfrak{h}_\theta$. That is, $\mathfrak{P}(\mathfrak{h}_\theta)$ is an inference-type algorithm. Without these assumptions, the definition would permit the slow program to be an iterative algorithm initialized at $\mathfrak{h}_\theta$, such as gradient descent initialized at $\theta$, which is a more general notion than the one we seek to isolate in this work. Note that internalization, as defined here, needs not be a learning process. It does not assume access to a training dataset, nor do we evaluate absolute performance with respect to a test distribution. We will describe later how it can be used for learning.
\end{remark}

\begin{example}[Internalizing autoregressive chain-of-thought]\label{ex:int-cot}
We follow the notation of~\citet{Jos+25}.
Let $\Sigma$ be a finite token alphabet, let $\mathcal{X} = \Sigma^\ast$ be the space of prompts, and let $\mathcal{Y} = \Sigma$ be the space of output tokens. Let $\mathcal{H} = \left\{ h_\theta \middle | \theta \in \Theta \right\} \subseteq \mathcal{Y}^\mathcal{X}$ be a class of next-token generators indexed by $\theta \in \Theta$ for some set $\Theta$.

Fix a $\theta \in \Theta$. For the generator $h_\theta$, define its apply-and-append map
$\bar h_\theta: \mathcal{X} \to \mathcal{X}$ by $\bar h_\theta(x) := \texttt{append}\bigl(x, h_\theta(x)\bigr)$. For a fixed generation length $T \in \mathbb{N}_+$, define
\begin{equation}\label{eq:P(h)_autoregressive}
    P(h_\theta)(x) = h_\theta^{\mathrm{e2e}-T}(x) : =  \underbrace{\bar h_\theta \circ \bar h_\theta \ldots \circ \bar h_\theta}_{T \text{ times}}(x) [-1],
\end{equation}
where $z[-1]$ denotes the last token of a string $z$. This is the mapping $P(h_\theta)$ from prompt to final token, ignoring the generated chain-of-thought. Also, define corresponding programs $\mathfrak{h}_\theta, \mathfrak{P}(\mathfrak{h}_\theta)$ for functions $h_\theta, P(h_\theta)$, respectively. Assume that the runtime of $\mathfrak{h}_\phi$ on a prompt depends only on the prompt length and not on $\phi$, that is $\mathrm{Time}_{\mathfrak h_{\phi}}(x) = \mathrm{Time}_{\mathfrak h_{\theta}}(x)$ for all $x, \phi, \theta$, and that $\mathfrak P(\mathfrak h_\theta)$ computes $P(h_\theta)(x)$ from an input $x$ by performing the $T$ calls to $\mathfrak{h}_\theta$ explicitly.

Suppose there exists $\theta^\prime \neq \theta$ such that 
\begin{equation}
    h_{\theta^\prime}(x) = P(h_\theta)(x), \qquad \forall x \in \mathcal X.
\end{equation}
Let $r(n)$ denote the runtime of $\mathfrak h_\phi$ on a prompt of length $n$. Then, we have $\mathrm{Time}_{\mathfrak h_{\theta'}}(x)=r(|x|)$,
whereas the explicit autoregressive implementation satisfies
$\mathrm{Time}_{\mathfrak P_T(\mathfrak h_\theta)}(x) \geq \sum_{t=0}^{T-1} r(|x|+t)$. Hence, if $T>1$ and $r(n)>0$ for all $n$, then
\begin{equation}
    \mathrm{Time}_{\mathfrak h_{\theta'}}(x) < \mathrm{Time}_{\mathfrak P_T(\mathfrak h_\theta)}(x).    
\end{equation}
Thus, according to~\Cref{def:internalization}, $\mathfrak{h}_{\theta^\prime}$ \textit{internalizes} $\mathfrak{P}(\mathfrak{h}_\theta)$.
In words, the updated generator $h_{\theta^\prime}$ produces in one forward pass exactly the same final answer that the original generator $h_\theta$ would have produced after $T$ autoregressive steps.
\end{example}


While the above definition and example ground our discussion of internalization (see also~\Cref{sec:internalization_more} for further discussion), they do not explain how to internalize, nor do they concern themselves with learning any of the hypotheses. Next, we describe how an ``indirect'' approach to learning (learning of the slow inference system and then internalization) can sometimes be used instead of ``direct'' learning of a hard distribution.

\subsection{Learning through Internalization}\label{sec:internalization-learning-perspective-intro}

Indeed, an internalization process, as defined above, can itself be a part of a greater learning process.

Let $\mathcal{X}$ be an input space, let $\mathcal{Y}$ be an output space, let 
$f^\star:\mathcal{X}\to\mathcal{Y}$ be an unknown target function, and let 
$\mathcal{D}$ be a distribution over $\mathcal{X}$. For a hypothesis 
$h:\mathcal{X}\to\mathcal{Y}$, define its error by $L_{\mathcal{D}}(h) := \mathbb{P}_{x\sim\mathcal{D}}\left[h(x)\neq f^\star(x)\right]$. The usual end-to-end learning approach is to collect samples 
$(x,f^\star(x))$ with $x\sim\mathcal{D}$ and use them to find a hypothesis 
$h_\theta\in\mathcal{H}$ with small error $L_{\mathcal{D}}(h_\theta)$. 
However, for many natural hypothesis classes and target functions, such as parities, this direct 
learning problem may be computationally difficult 
\citep{Val84,KeVa94}.

An alternative to such an ``end-to-end'' approach would be to capitalize on internalization: instead of trying to learn the final 
map directly, one may first learn a predictor with slower inference time $P(h_\theta):\mathcal{X}\to\mathcal{Y}$, induced by an $h_\theta \in \mathcal{H}$, with small error $L_{\mathcal{D}}(P(h_\theta))$. This slower inference system may be easier to obtain because we can exploit additional supervision or search.  The second step is to then find a hypothesis  $h_{\theta'}\in\mathcal{H}$ that internalizes $P(h_\theta)$. Thus, the original goal of ``end-to-end'' learning can be split into two: learning of a ``slower'' system and then internalization.
For example, \citet{DCS24} showed that such a learning strategy can sidestep the difficulties of learning certain reasoning tasks with transformers: explicit chain-of-thought supervision was used to train the slow transformer-based autoregressive system, followed by an internalization process consisting of curriculum training with progressively amputated chain-of-thought data, eventually yielding a transformer that can respond correctly in a single forward pass.

\paragraph{Positive Distribution Shift.}
This strategy of learning through internalization can be related to a learning framework recently formalized by~\citet{Med+26}. In particular,~\citet{Med+26} introduced a \textit{Positive Distribution Shift (PDS)} framework to reconcile theoretical hardness of learning with empirical success of deep learning practice. The perspective of PDS here is that even though learning certain distributions can be ``hard'' (e.g., difficult to learn directly with transformers and standard gradient descent methods, or impossible with computationally efficient algorithms), one can succeed in learning by coming up with another distribution that is not only itself easier to learn from, but, upon learning it, one is certain that it has learned the original ``hard'' distribution, too. 

In this terminology, the distribution used for learning the slow system $P(h_\theta)$, together with any other distributions accessed during internalization, forms a ``positive distribution shift'' relative to the potentially hard ``end-to-end'' distribution $\mathcal{D}$. For example, the sequence of distributions in the curriculum of \citet{DCS24} constitutes a ``positive distribution shift'' to the end-to-end distribution associated with the hard reasoning problems they study. We note that the form of PDS studied by \citet{Med+26} concerns ``standard'' supervised learning and predominantly considers covariate shifts in a single shifted distribution, whereas the distribution shift considered here also affects the target (due to chain-of-thought supervision) and consists of a sequence of distributions. In Section~\ref{sec:pds-learning-internalization}, we introduce a setting in which learning through internalization facilitates the learning of hard distributions with feed forward networks, and we explicitly connect both this approach and autoregressive learning of semiautomata to the strict notion of PDS.

\section{Internalizing Semiautomata with Transformers}\label{sec:automata}

In this section, we consider a specific case of learning through internalization. In particular, we investigate whether and under what conditions transformers can internalize the simulation of \emph{semiautomata}.

A semiautomaton consists of a state space, an input space, and a transition function that, given the current state and an input symbol, determines the next state. We are interested in predicting the state of the semiautomaton after $T$ transition steps, specified by inputs $\sigma_1, \ldots, \sigma_T$ and an initial state $q_0$, and consider various types of training that differ in the type of supervision they provide beyond the input tokens:
\begin{enumerate}
    \item \textit{End-to-end} (E2E) training, where only the final state $q_T$ is appended to the input,
    \item \textit{Chain-of-thought} (CoT) training, where sequences contain all the intermediate states $q_1, \ldots, q_{T-1}$ together with the initial and final states,
    \item \textit{Internalization} via a curriculum, where we first perform CoT training and then progressively start omitting intermediate states from the training sequences~\citep{DCS24}. We consider internalization via a \textit{left} and a \textit{right} curriculum~\citep{DCS24}, which differ in whether the CoT starts getting amputated from the beginning or the end of the sequences, as well as an \textit{inductive} curriculum that we introduce, which trains on sequences with progressively larger leaps between states.
\end{enumerate}

In our experiments, we mainly consider modular counter semiautomata $C_n$, which have inputs and states from $\mathbb{Z}_n = \left\{0, \ldots, n-1\right\}$ and the state at time $t+1$ is computed as $q_{t+1} = (q_t+\sigma_{t+1}) \bmod n$. We also consider the non-commutative semiautomaton associated with the permutation group $S_3$. The input distribution draws each symbol uniformly at random from the input alphabet. To probe out-of-distribution generalization during training, we also consider distributions that bias the total sum of the input tokens. See Appendix~\ref{sec:app:experimental-details-semiautomata} for more details and formal definitions.

Simulating semiautomata presents a clean setting for studying the internalization of CoT in transformers. Recall that, during internalization, we start training with explicit CoT containing $T-1$ intermediate states/sequential reasoning steps. Such a CoT sequence can be compactly represented by a \textit{single} transformer layer (attention on the current state and the corresponding input symbol, while the MLP implements the transition function). However, the transformer needs to change its representation during internalization and could, in principle, converge to either of two qualitatively different ones: either to a representation that mirrors explicit reasoning but is unfolded vertically across $\Theta(T)$ transformer layers, or to a shallower representation that exploits the parallel computation inherent to the transformer. Indeed, we know that such shallow \textit{shortcut} solutions, utilizing only $\Theta(\log T)$ transformer layers of reasonable width, exist for all semiautomata~\citep{Liu+23}. We ask: which type of representation is favored by SGD-type internalization? Can we internalize with shallow transformers, or does successful internalization of $T$ steps demand $\Theta(T)$ transformer layers to begin with?


\subsection{Internalizing chain of thought over model width}\label{ssec:method-specific-factors}

\begin{figure}
    \centering
    \includegraphics[scale=0.445]{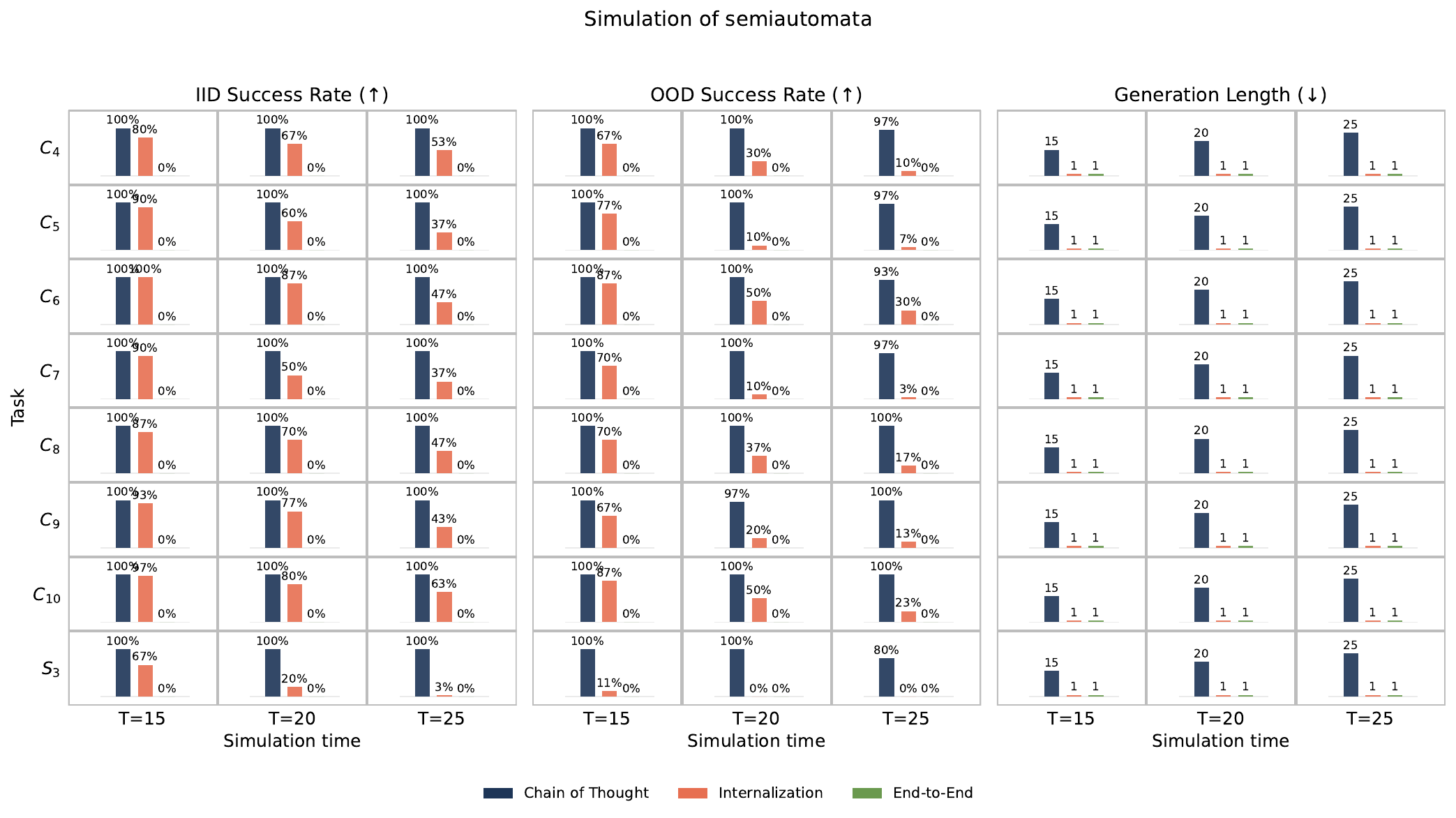}
    \caption{\textbf{Comparison between chain-of-thought (CoT), end-to-end (E2E) and (left) internalization training for semiautomata simulation with transformers.} We report in-distribution (ID) and out-of-distribution (OOD) success rates that are defined as the percentage of runs that resulted in >95\% test accuracy (ID or OOD, respectively), as well as the length of the generated responses.}
    \label{fig:cot_e2e_internalization_comparison}
\end{figure}

We first present results of training relatively shallow transformers (depth $L=4$) compared to the simulation time $T \in \{15, 20, 25\}$ on a range of different semiautomata (\Cref{fig:cot_e2e_internalization_comparison}). We compare CoT and E2E training with internalization via left curriculum. In our experiments, the number of possible input sequences far exceeds the number of training sequences, so we can rule out learning via ``memorization''. We found that internalization runs were effectively binary in achieving solid test accuracy\footnote{In this section, test accuracy refers to final-state accuracy after autoregressive generation.} or being on par to random chance, and for this reason we opted to report success rates (fraction of runs with test accuracy > 95\%) instead of aggregated accuracy numbers. We ran each configuration for 30 seeds -- see \Cref{sec:app:experimental-details-semiautomata} for details. We list our observations:

\begin{itemize}
    \item \textbf{E2E training fails.} We observe that transformers are not able to learn to simulate semiautomata with E2E training. The end-to-end tasks exhibit long-range dependencies that appear difficult to learn even for relatively ``simple'' local operations, such as $\bmod n$ addition. 
    \item \textbf{CoT training succeeds.} On the other hand, the task becomes easy with chain-of-thought training. There, the model leverages the additional supervision and generalizes on novel inputs; both in-distribution and out-of-distribution. Indeed, we observe very small drops (if any) from ID to OOD success rates across tasks for CoT training. The downside, of course, is the high inference time, since the model has to generate all the intermediate states before generating the final one.
    \item \textbf{Internalization succeeds with shallow transformers.} Remarkably, we find that the internalization process with the left curriculum often produces networks that generalize perfectly in distribution, despite them not generating any intermediate tokens; this holds even for non-commutative semiautomata. In fact, the internalization process can succeed even with depth-1 transformers (see, for example, \Cref{fig:scatter_C2_C3_C5}), demonstrating that an ability to explicitly represent all the intermediate reasoning steps is not necessary for ``implicit'' reasoning in the model. However, we observe a tradeoff: the OOD success rates of the internalization runs are significantly lower than those of CoT training. We return to this point in \Cref{ssec:automata_OOD}. 
\end{itemize}
\paragraph{Width helps more than depth.}
We next study the role of depth in internalization.
\begin{figure}
    \centering
    \includegraphics[scale=0.48]{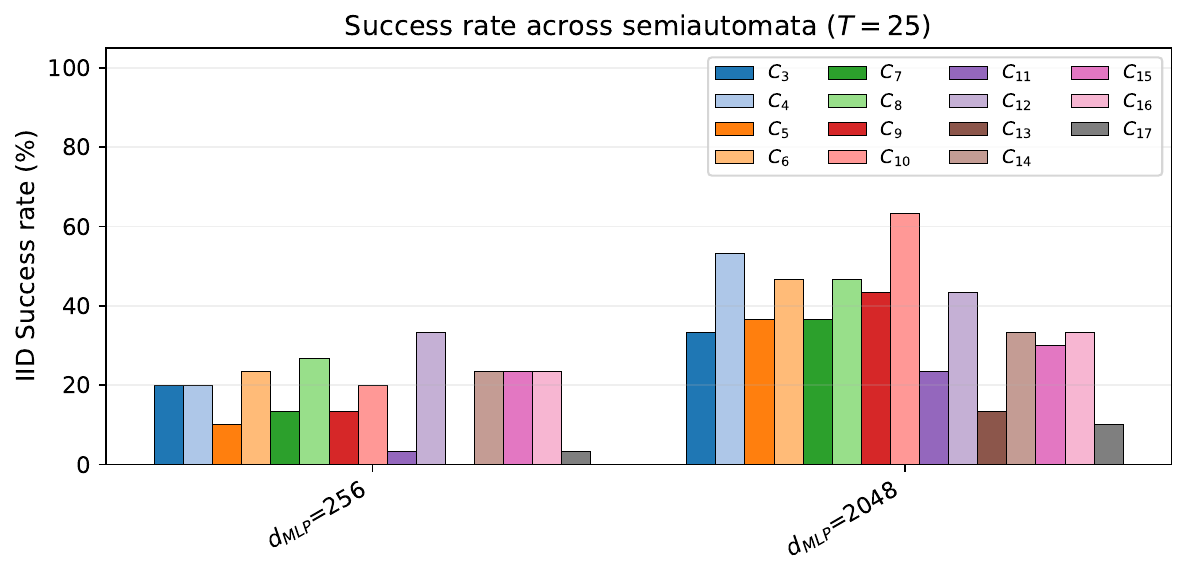}
    \hfill
    \raisebox{0.35cm}{
    \includegraphics[scale=0.38]{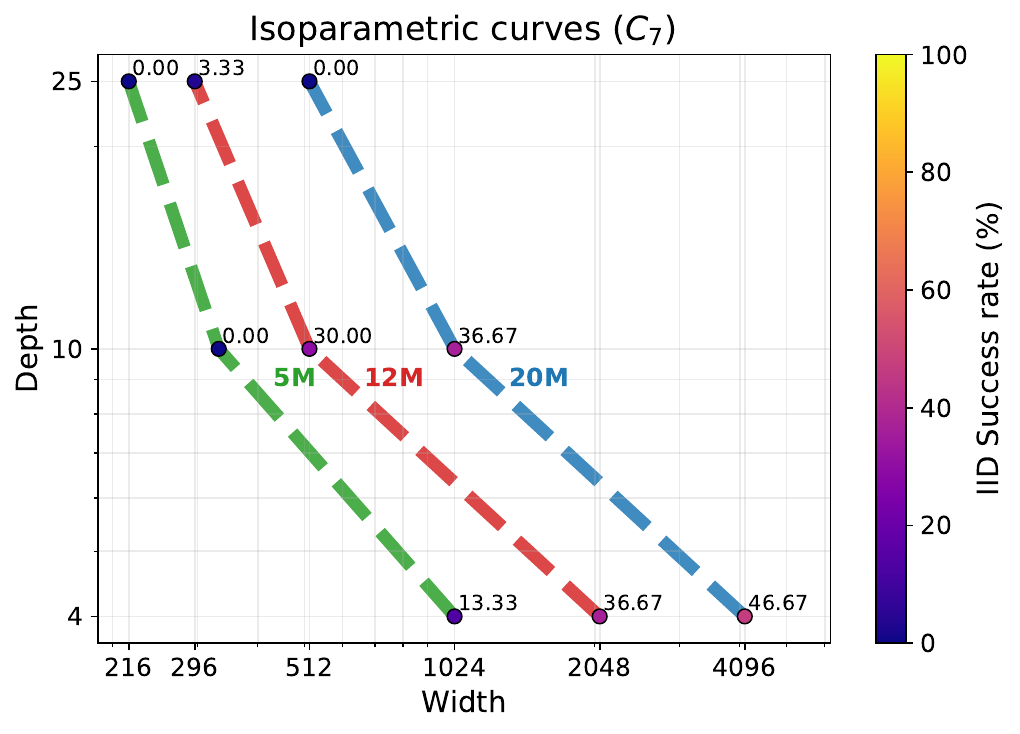}}
    \caption{\textit{Left:} \textbf{The probability of internalization depends on the order of the cyclic semiautomaton.} We show IID success rates across various counter semiautomata for two different width configurations (depth is fixed to $L=4$). We observe that increasing the MLP width of the transformer improves the odds of successful internalization, but to a lesser extent for $C_n$ with $n$ being a prime number. \textit{Right:} \textbf{The probability of internalization increases with width rather than depth.} We show in-distribution success rates across various transformer configurations. Each line corresponds to an isoparametric curve (scales: 5M, 12M, and 20M parameters).}
    \label{fig:isoparams_and_cprime}
\end{figure}
In \Cref{fig:isoparams_and_cprime} (right), we show internalization results for simulating the $C_7$ semiautomaton over $T=25$ steps where we vary the depth (number of layers) and width (maximum of embedding dimension and MLP width) of the transformer while keeping the number of total parameters (approximately) fixed. We observe that, across different parameter scales, increasing depth (4 $\to$ 10 $\to$ 25) is less beneficial for internalization than increasing width. In fact, deeper models ($L=25$) do not appear well suited for internalization. See also Figure~\ref{fig:isoparams_more} for similar plots on $C_{11}$ and $C_{13}$. These findings support the hypothesis that the intermediate states/reasoning steps are internalized ``horizontally'' over the width of the model rather than being explicitly represented ``vertically'' over its depth. 

\subsection{The flip side of internalization: worse OOD performance}\label{ssec:automata_OOD}

Despite the previous positive internalization results with shallow transformers, we observe what appears to be the flip side of internalization: the decline of OOD performance.
Indeed, \Cref{fig:cot_e2e_internalization_comparison} shows that across tasks and simulation times, internalized runs are far less probable to have good OOD test accuracy in comparison to transformers produced by CoT training.

In Figure~\ref{fig:ad_figure} (right), we show two internalization runs for $C_3$ and $C_4$, respectively, that demonstrate how the OOD performance progressively gets worse as we remove more tokens from the chain-of-thought supervision during internalization.

Furthermore, in \Cref{fig:scatter_C2_C3_C5}, we present a scatter plot of internalization runs for various random seeds, showing the final test accuracy (in-distribution and out-of-distribution) of the transformer on $C_2, C_3,$ and $C_5$. For all three tasks, we observe that succeeding in internalizing in distribution does not necessarily imply good out-of-distribution performance. For $C_2$, we observe that increasing depth has a positive effect on OOD performance, while for $C_3$ and $C_5$ this OOD weakness persists even for deeper networks. These results imply that in our settings transformers fail to exactly internalize in the strict sense of~\Cref{def:internalization}.
\begin{wrapfigure}{r}{0.28\textwidth}
    \centering
    \includegraphics[scale=0.31]{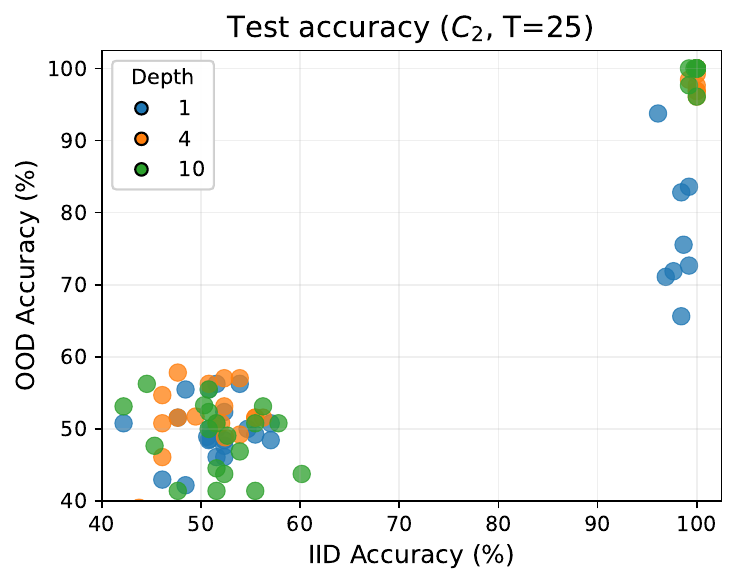}
    \includegraphics[scale=0.31]{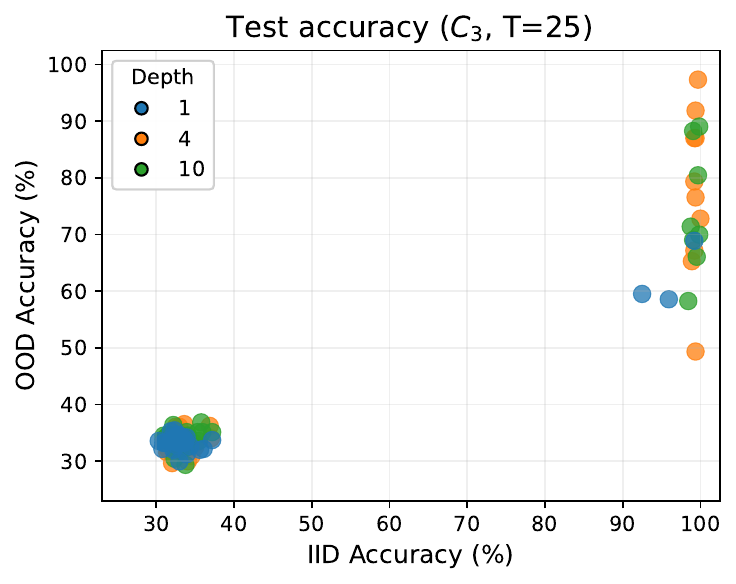}
    \includegraphics[scale=0.31]{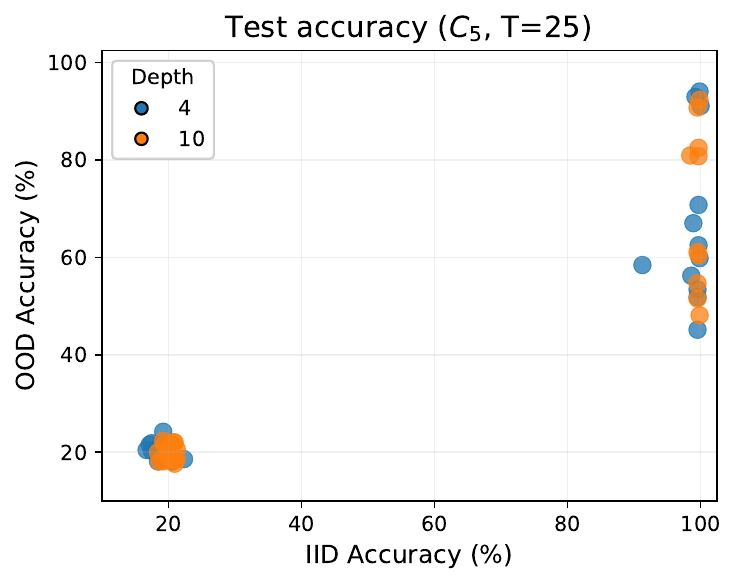}
    \caption{Scatter plots of test accuracy at the end of internalization for several seeds across transformers with various depths.}
    \vspace{-15mm}
    \label{fig:scatter_C2_C3_C5}
\end{wrapfigure}

We hypothesize that the decline in OOD performance during internalization occurs because the shallow transformer is forced to implement the $\bmod n$ operation over more than two input tokens (as initially done in the CoT representation), and this may be difficult for a single MLP to robustly learn from uniformly sampled sequences.

\subsection{Prime \texorpdfstring{$C_n$}{} are harder to internalize}\label{subsec:prime-hard-to-internalize}
Semiautomata such as $C_2, C_7, C_{18}$ all share the same state space structure, but they differ in the transition function, i.e., the local operation rule, which depends on the order of the semiautomaton. The exact value of the order does not seem to affect whether a transformer can learn under CoT or E2E training. But, could the ``complexity'' of the local operation matter for successful internalization?

In \Cref{fig:isoparams_and_cprime} (left), we plot success rates on all counter semiautomata from $C_3$ all the way to $C_{17}$ for two different width configurations. We observe a noticeable drop in the probability of internalization for semiautomata $C_n$, where $n$ is a prime number. For example, semiautomata $C_{11}$ and $C_{13}$ appear to be more difficult to internalize than $C_{14}, C_{15}$ and $C_{16}$, even though they have a comparable number of possible inputs and states. This difference persists even for wider networks (MLP width: 256 $\to$ 2048). Why might this be the case?

In order to internalize a semiautomaton, the transformer needs at the very least to be able to represent the final end-to-end mapping from $T$ inputs and initial state to final state. In the regime we are studying, where the number of layers is much smaller than the horizon of the simulation, prior work has established that generic semiautomata with $|Q|$ states can be represented by a transformer of $O\left(|Q|^2\right)$ MLP width~\citep{Liu+23}. While this suggests that a transformer might need larger width to internalize counter semiautomata of larger order, it unfortunately tells us very little on how to differentiate between semiautomata whose orders differ only sligthly, e.g., $C_{10}$ vs $C_{11}$.

We show next that a transformer can represent some counter semiautomata $C_n$ much more compactly than what was known before, using an MLP width that depends on the sizes of the \textit{prime-power factors} (prime factors with their multiplicity taken into account) of $n$. In particular, this implies that semiautomata $C_n$, where $n$ has small prime-power factors, can be represented very efficiently by a small transformer.

\begin{theorem}[Representation of $C_n, \, n = \prod_{j = 1}^k p_j^{v_j}$, with $O(\log T)$ depth and $O(\sum_{j = 1}^k p_j^{v_j})$ width; informal]\label{thm:log_depth_log_width_prime_informal}
    Consider a counter semiautomaton $C_n, \, n = \prod_{j = 1}^k p_j^{v_j}$. Then, for any $T$, there exists a transformer that simulates $C_n$ over $T$ steps using $O\left(\log T \right)$ layers, embedding dimension $O\left(k \right)$ and MLP width $O\left(\sum_{j = 1}^k p_j^{v_j}\right)$.
\end{theorem}

See Theorem~\ref{thm:log_depth_log_width_construction} in the appendix for a formal statement and proof. The proof exploits the fact that addition of two numbers $\bmod n$ can be performed ``in parallel'' in the basis of each prime-power factor of $n$ due to the Chinese Remainder Theorem. \Cref{thm:log_depth_log_width_construction} 
implies that if $n$ is a $B$-\textit{powersmooth} number, meaning that $p_j^{v_j}\le B$ for every prime-power factor in its prime factorization, then $C_n$ can be simulated with MLP width $O(B\log n)$. Indeed, there are at most $\log_2 n$ distinct prime factors, so $\sum_j p_j^{v_j}\le B\log_2 n$. In particular, if $B\le(\log n)^c$ for some fixed constant $c$, the resulting MLP width is $O((\log n)^{c+1})$.
For this class of counter semiautomata, this construction implies an exponential improvement on the size of the transformer, compared to the (general) result of~\citet{Liu+23}.

Based on this result, we conjecture that semiautomata such as $C_{10}$ and $C_{12}$ are easier to learn through internalization than $C_{11}$ because the size of their final representation is smaller and hence they require fewer samples to learn. Speculating even further, we could hypothesize that the more ``parallelizable'' the local operation of a semiautomaton is, the easier it might be to internalize. We defer a deeper investigation of this fascinating question to future work.

We discuss additional optimization and architectural choices and how they they affect the internalization of semiautomata in Appendix~\ref{sec:app:experimental-details-semiautomata}.

\section{Provable Internalization of Parities in 1-layer Transformer with Linear Attention}
\label{sec:parity_one_layer_tf}
In this section, we present our main theoretical result: a simplified one-layer transformer can provably learn a hidden parity with explicit CoT supervision and then internalize the computation, outputting the parity in a single generation step (\Cref{thm:provable-parity-main}).

We study the problem of learning a hidden parity over an unknown support
\( 
S_\star=\{i_1,\ldots,i_k\}\subseteq[d],
\)
where the input
\( x=(x_1,\ldots,x_d)\sim \textnormal{Unif}(\{0,1\}^d)
\), and the final target is
\( 
y=\bigoplus_{i\in S_\star}x_i=x_{i_1}\oplus\cdots\oplus x_{i_k}.
\)
This class is hard to learn in the statistical query model~\citep{Kea98}, which is often viewed as capturing limitations of gradient-based learning, and is also believed to be hard to learn in the noisy PAC setting. \citet{KiSu24} showed that, with CoT supervision, transformers can provably learn parity efficiently, but the learned model still explicitly generates the CoT tokens autoregressively at inference time. Here, we show that this computation can also be internalized: after a CoT-shortening curriculum (from left), the model learns to directly output the parity. 
This gives a provable example where learning through internalization enables learning a task that is computationally hard for ``direct'' learning.

\paragraph{CoT and Curriculum.}
In addition to the final label, we use a CoT sequence given by the prefix parities
\( 
    z_t=x_{i_1}\oplus\cdots\oplus x_{i_t}, \, t\in[k],
\) so that \(z_k=y\). The curriculum begins with the full sequence $a^{[1]}=(x_1,\dots,x_d,z_1,\dots,z_k)$ and progressively removes tokens from left such that in Phase $t$, the input sequence is
\begin{equation}\label{eq:a1-main}
     a^{[t]}= (x_1,\ldots,x_d,z_t,z_{t+1},\ldots,z_k).
\end{equation}
After all CoT tokens are removed, the model must directly output the parity given input \((x_1,\ldots,x_d)\).


\paragraph{Architecture.}
We analyze a simplified one-layer linear-attention transformer, similar in spirit to prior work of \citet{KiSu24}. The architecture uses a single attention head that collapses the key-query parameterization into a single matrix \(W_{QK}\), and fixes the value matrix and MLP. For a content sequence
\( 
a=(a_1,\ldots,a_D)\in\mathbb{R}^D,
\)
we represent each token in dimension \(D+1\), with one scalar content coordinate and \(D\) one-hot positional coordinates:
\[
X(a)
=
\begin{pmatrix}
a_1 & 1 & 0 & \cdots & 0\\
a_2 & 0 & 1 & \cdots & 0\\
\vdots & \vdots & \vdots & \ddots & \vdots\\
a_D & 0 & 0 & \cdots & 1
\end{pmatrix}
\in\mathbb{R}^{D\times(D+1)}.
\]
The attention scores are bilinear (without softmax attention) that have been also popularized as a variant to softmax \cite[e.g.,][]{katharopoulos2020transformers,shen2021efficient}:
\[
S=X(a)W_{QK}X(a)^\top,
\]
where \(W_{QK}\in\mathbb{R}^{(D+1)\times(D+1)}\). We fix the value matrix to be the identity, \(W_V=I_{D+1}\), and apply the causal mask to obtain the attention output
\[
H(X(a))=(M\odot X(a)W_{QK}X(a)^\top)X(a),
\]
where \(M_{ij}=\mathbf{1}\{j\le i\}\). For any position $i$, the fixed MLP reads the first coordinate of the residual stream \(H_i(X(a))\in\mathbb{R}^{D+1}\), which corresponds to the content coordinate, and applies a nonlinearity \(\sigma:\mathbb{R}\to\mathbb{R}\):
\[
f_W(a)_i=\sigma\!\left((H_i(X(a)))_1\right).
\]
We choose \(\sigma\) to be a piecewise-linear parity activation, equal to \(0\) on even integers and \(1\) on odd integers, with linear interpolation between consecutive integers; see \Cref{eq:non-linearity}. Following the same simplification as in \citet{KiSu24}, we assume that \(W_{QK}\) ignores the scalar content coordinate and is parameterized only through the positional block:
\[
W_{QK}
=
\begin{pmatrix}
0 & 0\\
0 & W
\end{pmatrix},
\qquad
W\in\mathbb{R}^{D\times D}.
\]
See \Cref{subsec:architecture} for how this choice greatly simplifies the architecture: for any position \(i\in [D]\),
\begin{equation}\label{eq:output-at-i}
    f_W(a)_i
=\sigma(H_i(X(a))_1)=
\sigma\!\left(\sum_{j\le i}W_{ij}a_j\right).
\end{equation}
\paragraph{Training Algorithm.}
We now specify how the model is trained in plain language, with pseudocode deferred to \Cref{alg:internalize-parity-shortening}. In Phase \(1\), we use the full teacher-forced input sequence
\(
a^{[1]}=(x_1,\ldots,x_d,z_1,\ldots,z_k)
\)
from \Cref{eq:a1-main}
and train on an empirical i.i.d. batch for the following population objective:
\[
L_1(W)
:= \frac{1}{2}\mathbb{E}_{a}\left[\sum_{l=1}^{k}\left( f_W(a^{[1]})_{d+l-1}-a^{[1]}_{d+l}
\right)^2 \right]
=\frac{1}{2}\mathbb{E}_{a}\left[\sum_{l=1}^{k}\left( f_W(a^{[1]})_{d+l-1}-z_{l}
\right)^2 \right]\,.
\]
This amounts to CoT training on the output part of the sequence, as in supervised fine tuning (SFT). In subsequent phases \(t=2,\ldots,k\), we use the left-amputated sequence
\(
a^{[t]}=(x_1,\ldots,x_d,z_t,\ldots,z_k)
\)
from \Cref{eq:input-phase-t} and, after drawing a fresh batch, train on the empirical version of the following objective:
\[
L_t(W)
=\frac{1}{2}\mathbb{E}_{a}\left[
\left(
f_W(a^{[t]})_d-a^{[t]}_{d+1}
\right)^2 \right]
=\frac{1}{2}\mathbb{E}_{a}\left[
\left(
f_W(a^{[t]})_d-z_t
\right)^2 \right].
\]
Note that we only train on the first output position for stages $t=2, \ldots, k$. A fresh batch of size \(B\) is drawn in each phase, and we train on the same batch for \(T\) iterations within that phase, with step size \(\eta\). At the end of each phase, we round the entries of \(W\) to the nearest integer.\footnote{Such rounding was also used in \cite{KiSu24}, though for a different technical reason, and is also used empirically to reduce memory storage requirements \citep{wu2020integer,jacob2018quantization} and for efficient fine-tuning \citep{dettmers2022gpt3}. In our analysis, it seems to be needed only for a minor technical reason, and the entries before rounding are already close to integers. Experiments show full success without any rounding; cf. \Cref{fig:attention-parity}.}
Finally, we use an initialization scheme with mostly zero entries, except \(W_{j,j}^{(0)}=-1\) for \(j=d+l\) for all \(l\in [k-1]\). This encodes the bias that the model should attend to the token at position \(d+l\), since in Phase \(1\), we know that \(z_{l+1}=z_l\oplus x_{i_{l+1}}\), and attending to \(z_l\) is important.\footnote{We can also work with the all-zero initialization \(W^{(0)}=0\), in which case the final internalization guarantee in \Cref{eq:only-final} from \Cref{thm:provable-parity-main} still holds. This is because the final output at the \(d^\mathrm{th}\) position involves only the \(d^\mathrm{th}\) row \(W_{d,:}\), and by \Cref{eq:output-at-i}, it does not depend on the initialization of other rows. Initializing these entries with \(-1\) is only needed for the guarantee of outputting explicit CoT in \Cref{eq:full-CoT}.}

We formally state this training algorithm in~\Cref{alg:internalize-parity-shortening} and show the following performance guarantee.
\begin{theorem}\label{thm:provable-parity-main}
There exist universal constants \(C,c>0\) such that running \Cref{alg:internalize-parity-shortening} with
\[
\eta=\frac{c}{d},
\qquad
T=\lceil Cd\rceil,
\qquad
B\ge Cd^3\log\frac{kT}{\delta},
\]
satisfies the following with probability at least \(1-\delta\).
\begin{itemize}
    \item After Phase \(1\), the model \(W^{(1)}\) computes all
    explicit CoT tokens autoregressively. For every
    \(x \in \{0,1\}^d\), define the generated tokens recursively by
    \[
        \widehat z_l(x)
        :=
        f_{W^{(1)}}\!\left(
            \mathtt{append}\!\left(
                x,
                (\widehat z_1(x), \ldots, \widehat z_{l-1}(x))
            \right)
        \right)_{d+l-1},
        \qquad l \in [k],
    \]
    where the generated prefix is empty when \(l=1\).
    Then, for every \(l \in [k]\) and every \(x \in \{0,1\}^d\),
    \begin{equation}\label{eq:full-CoT}
        \widehat z_l(x)
        =
        x_{i_1} \oplus \dots \oplus x_{i_l}.
    \end{equation}
    \item At the end of the \(k^\mathrm{th}\) phase of the curriculum, the final returned model \(f_{W^{(k)}}\) has internalized $P\left(f_{W^{(1)}}\right) = f^{\mathrm{e2e}-k}_{W^{(1)}}$ (in the sense of \Cref{def:internalization} and \Cref{eq:P(h)_autoregressive}) and, in particular, it holds for every \(x\in \{0,1\}^d\)
    \begin{equation}\label{eq:only-final}
        f_{W^{(k)}}(x)_d
        =
        \bigoplus_{i\in S_\star}x_i.
    \end{equation}
\end{itemize}
\end{theorem}

The overall runtime of the algorithm is \(\textnormal{poly}(d,\log(1/\delta))\), and it finds a representation \(W^{(k)}\) that directly computes the target parity in one forward pass. For the proof of \Cref{thm:provable-parity-main}, see \Cref{sec:app:proof-provable-parity-main}.

\begin{remark}
We note that the guarantee in \Cref{thm:provable-parity-main} ensures exact internalization for all inputs \(x\in \{0,1\}^d\), whereas in \Cref{ssec:automata_OOD}, we report poor OOD performance on related tasks. While the theoretical result involves simplifications, we attribute the main source of this mismatch to the fixed choice of the MLP for local operations. In our empirical study, these local operations also need to be adjusted, which we believe is primarily responsible for the OOD performance degradation. We leave it to future work to theoretically prove internalization results that exhibit worse OOD performance.
\end{remark}

\section{Learning through Internalization and Positive Distribution Shift}\label{sec:pds-learning-internalization}

\begin{figure}
    \centering
    \includegraphics[width=0.7\linewidth]{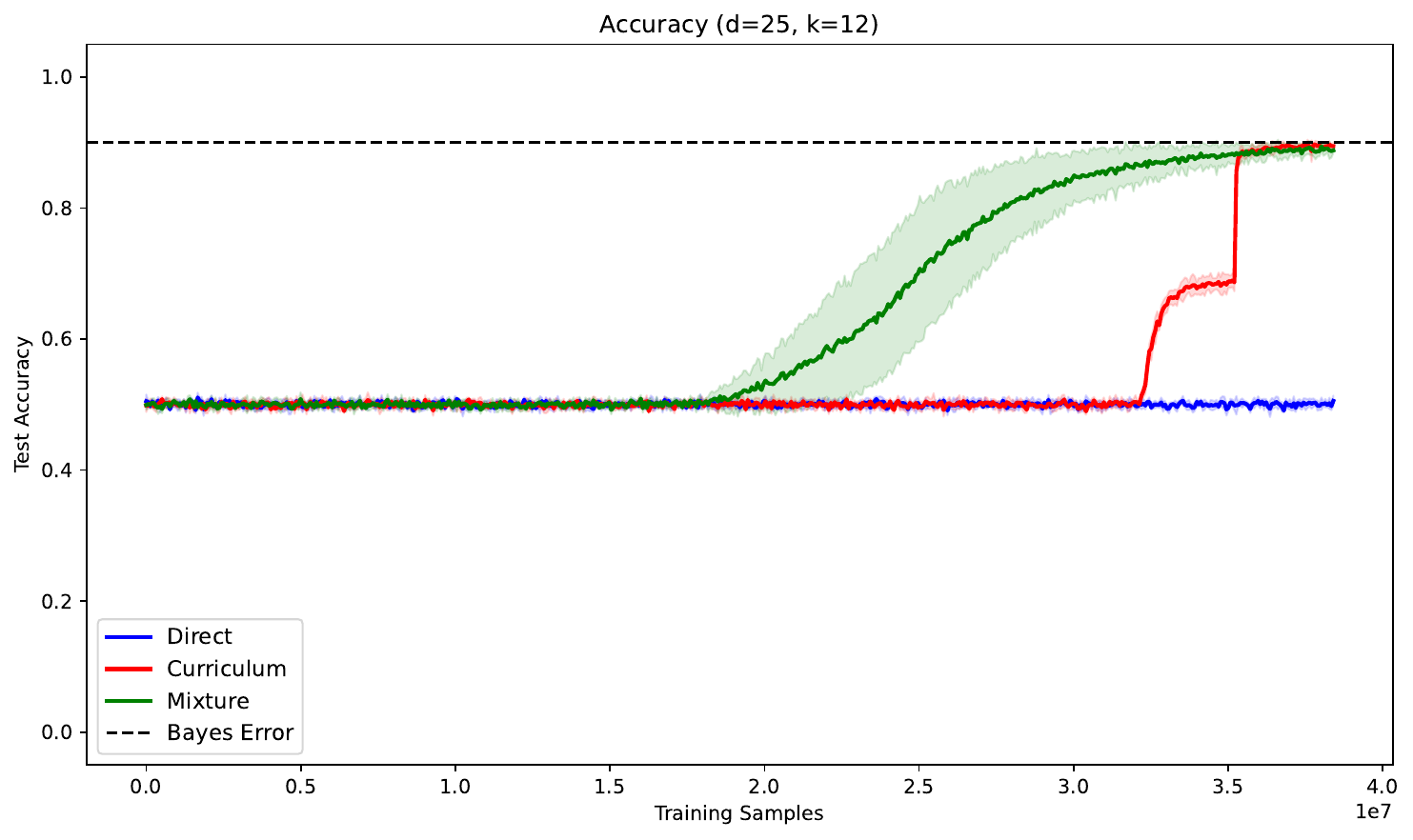}
    \includegraphics[width=0.485\linewidth]{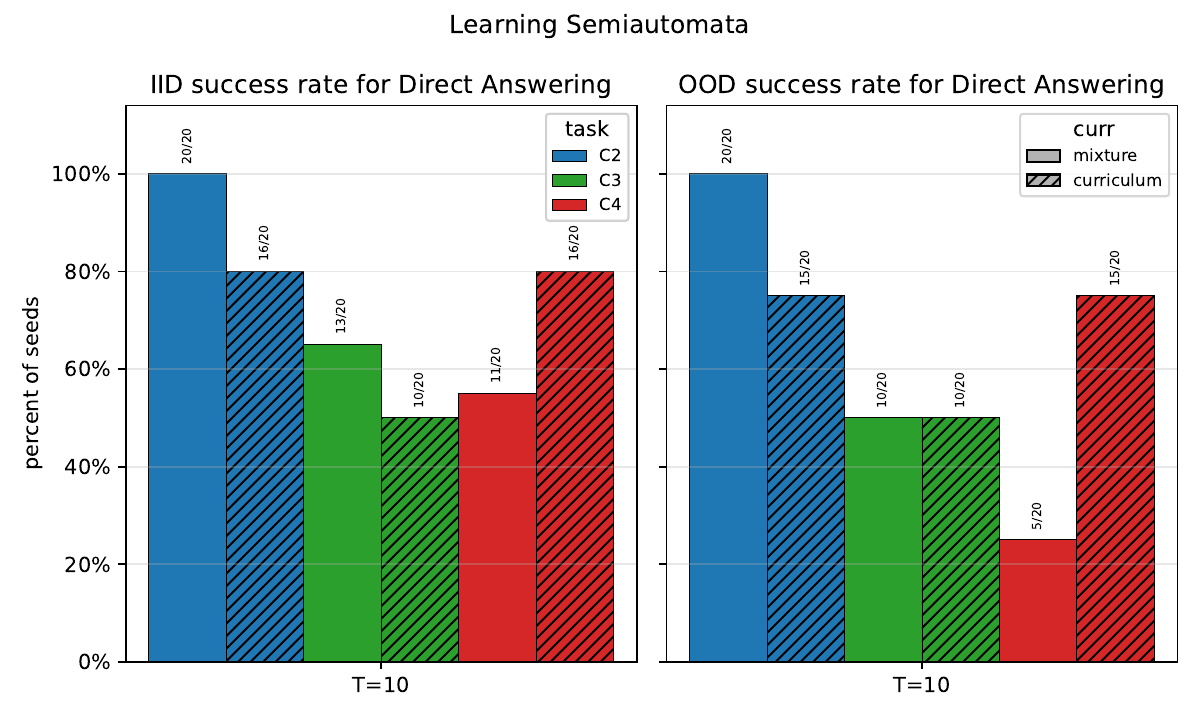}
    \includegraphics[width=0.485\linewidth]{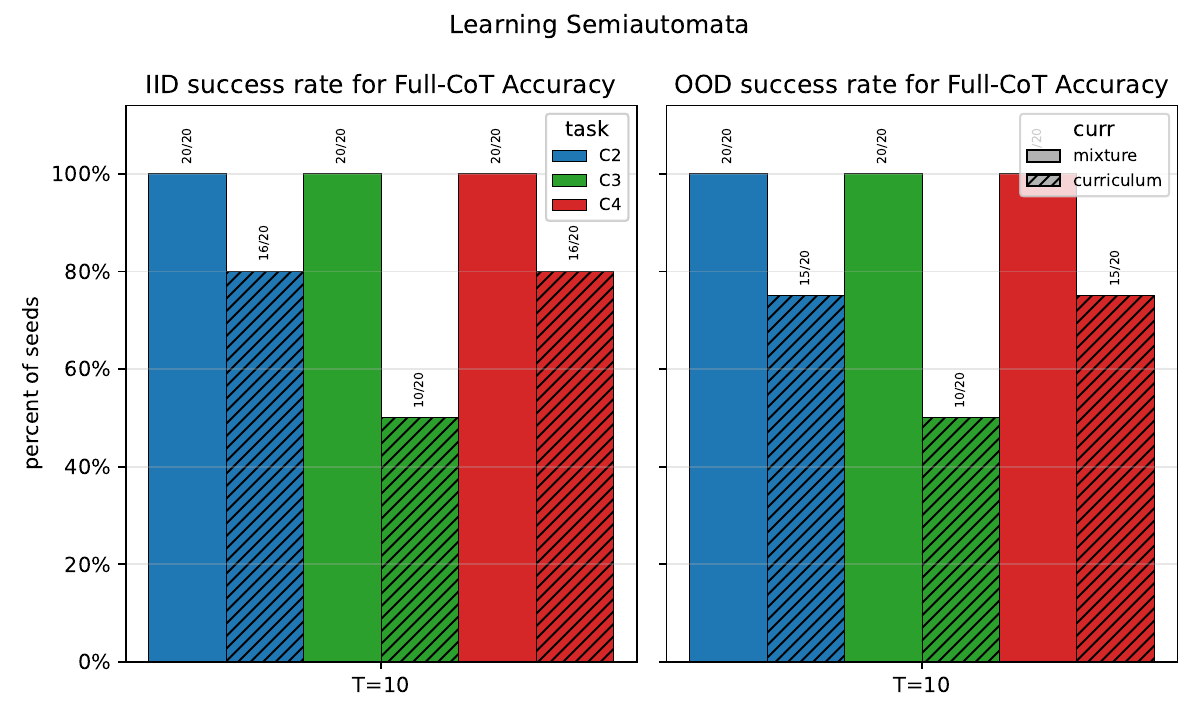}

    \caption{\textit{Top:} \textbf{Learning through Internalization in MLPs.}
    We plot the test accuracy for curriculum and PDS learning of parities described in
    \Cref{subsec:internalization-MLP} and \Cref{subsec:mixture}. That is, we plot the
    dependence of test accuracy on the number of training samples for three different
    training strategies for learning an unknown noisy parity over $k=\frac{d}{2}$ bits
    (label noise $\eta=0.1$) over the input space of $d$ bits expanded with $k-1$ hints
    $\tilde{x} = (x_1,\dots,x_d,z_1,\dots,z_{k-1})$, with $t$-th hint being the parity
    of first $t$ indices of the support, and we plot the unhinted accuracy, i.e. accuracy
    of predicting $y$ from $( x_1, \ldots, x_d, 0, \ldots, 0 )$. We consider
    (i) direct training (in blue) where all the hints are zeroed out,
    (ii) curriculum training (in red) where in stage $l$ we have $l-1$ hints zeroed out
    from the right, and (iii) mixture (in green) over the inputs from the $k-1$ stages of
    the curriculum, where we draw a stage uniformly at random. We see that direct training
    (blue) fails to learn, while both the curriculum (red) and the mixture (green) succeed
    in learning. 
    \textit{Bottom:} \textbf{Learning $C_n$ from a single mixture distribution over various CoT lengths with a hint for the length of answer.}
    We train a transformer on
    $\mathcal D_{\mathrm{mix}}=\frac{1}{T+1} \sum_{l=0}^{T} \mathcal D_{l}$,
    where $\mathcal D_l$ is the distribution over examples containing CoT of length $l$,
    and each example includes a hint for the length of the answer. We plot the percentage
    of seeds that achieve $>0.95$ accuracy for (i) on the left, direct answering - the model
    is given the hint for answer length equal to $1$, and (ii) on the right, the \emph{same}
    model is allowed to generate full CoT. We see that when training with mixture, we can
    learn to both directly answer and to generate the full CoT with a single model.}
    \label{fig:mixture}
\end{figure}


Finally, in this section, we first introduce and study a novel internalization scheme that enables learning hard functions with feed forward neural networks that is inspired by internalization of chain-of-thought. Here, the slower system does not involve multiple calls to the same neural network, as in the autoregressive case, nor does it search for better outputs according to some fixed rule, but rather maintains an additional set of weights that operate on \textit{hints} provided together with the input. The role of the internalization process is then to transfer the computation onto the weights that operate on the original input space, by relying less and less on hints. We describe the setup for the case of a hidden noisy parity target function, but note that similar setups can be defined for any computable function, with the hints consisting of the gates of a boolean circuit that computes the function.

In the rest of the section, we use this example to explain the connection between learning through internalization and the Positive Distribution Shift framework introduced in \cite{Med+26}. Motivated by this connection, we show that learning through a curriculum-based internalization procedure can often be replaced by a more straigthforward mixture learning procedure. We validate this in both the feed forward and the autoregressive setting.

\subsection{Learning through Internalization with MLPs}\label{subsec:internalization-MLP}

\sloppy
Let $k$ denote the size of the unknown support of the parity function, as in~\Cref{sec:parity_one_layer_tf}. In this section, we consider added label noise in the target (we flip the final label with probability $\eta$).

First, we augment the $d$-bit input $x=(x_1,\dots, x_d)$, which is uniformly distributed, with $k-1$ hints $z_i$ of intermediate computations to obtain $\tilde{x} = (x_1,\dots,x_d,z_1,\dots,z_{k-1})$. The $t$-th hint, $t\in [k-1]$, is the partial parity of the first $t$ indices of the support, i.e., $z_t = \oplus_{j=1}^{t} x_{i_j}$. The augmented inputs are similar to the CoT tokens of the previous section. During the first stage of training, we learn the mapping $\tilde{x} =\left( x_1, \ldots, x_d, z_1, \ldots, z_{k-1} \right) \mapsto y$ by training a ReLU neural network of depth $L$: $
g(\tilde{x}; \left\{W_L, b_L, \ldots, W_1, b_1\right\}) = W_L \,\sigma\!\left(W_{L-1}\,\sigma\!\left(\cdots \sigma\!\left(W_1 \tilde{x} + b_1\right) \cdots \right) + b_{L-1}\right) + b_L$. Then, we take only the subnetwork of g that operates on the first $d$ bits
$h_\theta: \left\{ 0, 1 \right\}^d \mapsto \mathbb{R}$, $h_\theta(x) = W_L \,\sigma\!\left(W_{L-1}\,\sigma\!\left(\cdots \sigma\!\left(W_1[:, :\!d] x + b_1\right) \cdots \right) + b_{L-1}\right) + b_L$, as our predictor.

Similar to the autoregressive case and the curriculum of~\citet{DCS24}, we now devise a curriculum that teaches this subnetwork to internalize the parity computation as performed initially by $g$. In particular, at the $l$'th stage, $l = 2, \dots, k$, we continue training $g$ on mappings of the form $\left( x_1, \ldots, x_d, z_1, \ldots, z_{k-1-l+1}, \underbrace{0, \ldots, 0}_{l-1} \right) \mapsto y.$
That is, we progressively zero-out some of the hints provided to the model, while keeping the target fixed. At the final stage of the curriculum, we train on data with all hints zeroed out, i.e. $\left( x_1, \ldots, x_d, 0, \ldots, 0 \right) \mapsto y$. 

\sloppy
We show that this strategy indeed succeeds in learning parity target functions. In~\Cref{fig:mixture}, we show experiments for $d=25, k=12, \eta = 0.1$ and a 2-layer ReLU network. We observe that direct training on the target parity fails (as expected from any computationally efficient algorithm), while learning with the slower neural network and then internalizing succeeds in learning the target. Note that this scenario falls within~\Cref{def:internalization} if we define:  (a) the ``slow'' system $P(h_{\theta_1})(x) : = g\left(\left(x, x_{i_1}, \ldots, \oplus_{j = 1}^{k-1} x_{i_j}\right); \left\{W_{L}^{[1]}, b_L^{[1]}, \ldots, W_1^{[1]}, b_1^{[1]} \right\}\right)$, where $W_{L}^{[1]}, b_L^{[1]}, \ldots, W_1^{[1]}, b_1^{[1]}$ are the weights at the end of the first training stage, and (b) the ``fast'' system $h_{\theta_{k-1}}(x) : = W_L^{[k-1]} \,\sigma\!\left(W_{L-1}^{[k-1]}\,\sigma\!\left(\cdots \sigma\!\left(W_1^{[k-1]}[:, :\!d] x + b_1^{[k-1]}\right) \cdots \right) + b_{L-1}^{[k-1]}\right) + b_L^{[k-1]}$, where $W_{L}^{[k-1]}, b_L^{[k-1]}, \ldots, W_1^{[k-1]}, b_1^{[k-1]}$  are the weights at the end of the last stage of the curriculum.

\subsection{From Internalization Curricula to Positive Distribution Shift: Learning with a Single Mixture Distribution}\label{subsec:mixture}

To connect this to the positive distribution shift view introduced in \Cref{sec:internalization-learning-perspective-intro}, we focus on the learning aspects of the internalization procedure with MLPs introduced above. In this example, our goal is to learn the parity distribution $(\mathcal D, f^*)=( \text{Unif}(\{0,1\}^d), x_{i_1} \oplus \dots \oplus x_{iu_k} \oplus \xi)$ (where $\xi$ represents the label noise) tractably, i.e. find a predictor $h_{\theta}$ which achieves low error $L_{(\mathcal D,f^*)}(h_{\theta})$. Noisy parity is computationally hard to learn - it is known to be hard in the SQ framework \citep{Kea98, BlumFurstJacksonKearnsMansourRudich94}, there is no known tractable PAC learning algorithm for it \citep{BlumKalaiWasserman03,GoelKanadeKlivansThaler17}, and even fixed parities are difficult to learn with standard neural-networks trained with standard gradient descent \citep{Shoshani25,Abb+25}. However, as shown above, the procedure of learning  $(\mathcal D,f^*)$ from the curriculum $(\mathcal D_1,f^*),\dots, (\mathcal D_{k-1},f^*)$, where $\mathcal D_i$ is given by $\tilde x \sim \left( x_1, \ldots, x_d, z_1, \ldots, z_{k-1-l+1}, \underbrace{0, \ldots, 0}_{l-1} \right)$, by the predictor $h_{\theta}$ is (experimentally) tractable. This learning procedure with the curriculum $(\mathcal D, f^*) \to \{(D_i,f^*)\}_{i=1}^{k-1} \to h_{\theta}$ is similar to the target dependent positive distribution shift, \textit{f-PDS}, introduced in \citet{Med+26}, where we modify the training distribution depending on the target (see \Cref{app:pds-view} for the definition of f-PDS). However, that type of positive distribution shift only involves a single training distribution (and only a covariate shift on the original input space\footnote{So in that regard, the curriculum considered here does not fall under the strict definition of f-PDS; see \Cref{app:pds-view} for details.}). In addition, \citet{Med+26} broadly conjecture that a curriculum can be replaced with learning from a single distribution. Motivated by this, in top \Cref{fig:mixture}, we show (in green) that we can replace the curriculum for learning parities with MLPs from \Cref{subsec:internalization-MLP} with a mixture distribution $\mathcal D_{\mathrm{mix}}=\frac{1}{k}\sum_{l=1}^{k} \mathcal D_l$, i.e. we can efficiently learn $(\mathcal D, f^*)$ using the subnetwork of the network that learns $(\mathcal D_{\mathrm{mix}},f^*)$. Therefore, $(\mathcal D_{\mathrm{mix}},f^*)$ represents a (generalized\footnote{Since the input space changes, it is not only a covariate shift; see \Cref{app:pds-view} for details.}) f-PDS for learning $(\mathcal D,f^*)$, which can be constructed intuitively and tractably from $f^*$ and $\mathcal D$: to learn $(\mathcal D,f ^*)$, train $g$ on $(\mathcal D_{\mathrm{mix}},f^*)$ then take the subnetwork $h_{\theta}$ that operates on the first $d$ bits as our subnetwork.

Inspired by the PDS view, we ask whether the autoregressive learning of semiautomata simulation from \Cref{sec:automata} using a curriculum can also be replaced by a single distribution. Indeed, in \Cref{fig:mixture}, we show that we can learn semiautomata by training on a single mixture distribution of curriculum stages, with added hints for answer length (see \Cref{app:mixture} for more details). This presents an alternative to learning through internalization for sidestepping the hardness of learning. 


\paragraph{Benefits of Learning from a Mixture: Direct Answering and OOD}

There are three main benefits to learning semiautomata with an autoregressive model from a mixture, as in \Cref{fig:mixture}, over training with a curriculum as in \Cref{sec:automata}. First, learning from a mixture removes the need to choose the curriculum schedule details and parameters, which can impact success of learning. Conversely, learning from a single mixture distribution is simpler. Second, a mixture distribution with CoTs of varying lengths more closely models the structure of natural data where some examples contain long explanations and other contain short explanations. Our results indicate that this heterogeneity of length of answers can itself act as a positive distribution shift for learning hard tasks. Third, with the answer hint, a single model trained on the mixture supports multiple inference budgets. This lets us force the model to answer directly or let it generate a longer CoT, and as \Cref{fig:mixture} shows, the mixture trained model can recover better OOD performance in the full-CoT regime, while the curriculum trained model does not.

\section{Discussion}\label{sec:discussion}

We saw how internalization can be a viable strategy not only for reducing inference costs while maintaining strong performance, but also for enabling the learning of difficult tasks that otherwise appear hard to learn through a more ``direct'' end-to-end approach. We investigated such internalization processes in autoregressive transformers and developed a better understanding of their strengths and shortcomings. Already for counter automata, we observed that, while learning without CoT is hard and learning with explicit CoT supervision is easy, the success of internalization varies across tasks and training methods (\Cref{fig:cot_e2e_internalization_comparison}). This suggests that understanding internalization requires accounting for new factors that are captured neither by the intrinsic difficulty of directly learning the end-to-end task nor by the ease of learning with step-by-step supervision.

While the focus of this paper has primarily been on internalizing autoregressive chain-of-thought reasoning, there have been many other instances in which neural networks have benefited from similar processes, as already discussed in the introduction. One such example is AlphaGo Zero~\citep{Sil+17}, a neural-network-based system equipped with a tree-search procedure that could suggest next moves and evaluate board positions. At each round of training, the neural network was optimized to match the next-move probabilities initially proposed by the computationally expensive tree search, as well as position values derived from game outcomes. In doing so, it internalized the results of self-play into its weights, enabling even stronger search in subsequent rounds (see Example~\ref{ex:self-play}, where we formalize this). Arguably, \citet{Sam59}’s AI checkers player also learned through a similar internalization process (Example~\ref{ex:checkers}).

In general, as AI inference costs continue to increase, the internalization of expensive scaffolded procedures will likely become increasingly common. While we have explained how this approach can be beneficial both for reducing computation time and for facilitating learning, some of our findings suggest that internalization may at times have negative and unexpected effects on the broader capabilities of a system.

\paragraph{Acknowledgments.} NT would like to acknowledge useful discussions with Eran Malach and Elisabetta Cornacchia. JK thanks the Simons Foundation for support through the Collaborative Grant ``The Physics of Learning and Neural Computation''. NT acknowledges support from a grant by the Coefficient Giving. NT and JK acknowledge support by the NSF through NRT Award 1922658. 

\bibliographystyle{plainnat}  
\bibliography{refs}

\newpage

\appendix




\section{Related Work}
\paragraph{Internalized and compressed reasoning.}
Recent works have explored reducing the cost of explicit CoT by replacing, compressing, or skipping reasoning tokens using latent reasoning states, compressed CoT, or step-skipping objectives~\citep{Hao+24,Shen+25,ChengVanDurme24,Liu+24,Chen+25}. These works primarily propose mechanisms for making reasoning traces shorter or latent, focusing on the inference speed aspect of internalization. \citet{Lin+25} made observations similar to ours on the poor OOD generalization of internalized CoT representations in arithmetic tasks with LLMs. 


\paragraph{Positive Distribution Shift.} \citet{DCS24,Yu+24} show that internalized reasoning can lead autoregressive networks to learn representations that are otherwise difficult to discover through “direct” training. In particular, \citet{DCS24} devise a curriculum-based internalization method that achieves this. A similar phenomenon has previously been observed in feedforward networks for computationally hard function classes~\citep{ACL,CoMo23}. In particular, \citet{ACL,CoMo23} show that a carefully chosen curriculum helps a two-layer MLP learn a hidden parity in a way that bypasses traditional computational lower bounds. More generally, and as we explain further throughout the paper, we view these phenomena as falling under the broader learning phenomenon of \textit{Positive Distribution Shift}, recently identified by \citet{Med+26}. The statistical advantage of a distribution shift was first observed in~\citet{Med+25}.

\section{Internalization processes: discussion and examples}\label{sec:internalization_more}

Next, we provide an additional definition that seems to better align with internalization in machine learning practice as well as an additional example of an internalization process.
In particular, relaxing \Cref{def:internalization}, we define an approximate internalization process with respect to some underlying input distribution and loss function.

\begin{definition}[Distributional Internalization]\label{def:dist_internalization}
    Let $\mathcal{D}$ be a distribution over $\mathcal{X} \times \mathcal{Y}$ and let $l: \mathcal{Y} \times \mathcal{Y} \to \mathbb{R}$ be a measurable\footnote{We also assume that all hypotheses $h_\theta \in \mathcal{H}$ and all predictors $P(h_\theta)$ are measurable and that the written expectations are finite.} loss function.
    Fix $\theta \in \Theta$ and let $\varepsilon > 0$. If there exists $\theta^\prime \in \Theta$ with $\theta^\prime \neq \theta$ such that
    \begin{enumerate}
        \item \textbf{Approximate risk agreement:}
        \begin{equation}
            \left| \mathbb{E}_{\left(x, y\right) \sim \mathcal{D}} \left[ l\left(h_{\theta^\prime}(x), y\right) \right] - \mathbb{E}_{(x, y) \sim \mathcal{D}} \left[ l\left(P(h_\theta)(x), y\right) \right] \right| \leq \varepsilon.
        \end{equation}
        \item \textbf{$\mathfrak{h}_{\theta^\prime}$ is faster than $\mathfrak{P}(\mathfrak{h}_\theta)$:} For all $x \in \mathcal{X}$, it holds:
        \begin{equation}
            \mathrm{Time}_{\mathfrak{h}_{\theta^\prime}}(x) < \mathrm{Time}_{\mathfrak{P}(\mathfrak{h}_\theta)}(x),
        \end{equation}
    \end{enumerate}
    then we say that $\mathfrak{h}_{\theta^\prime}$ $\left(\varepsilon, \mathcal{D}\right)$-\textbf{distributionally internalizes} $\mathfrak{P}(\mathfrak{h}_\theta)$.
\end{definition}

Most internalization processes we consider in this work fall within this definition, as the internalization is performed using some training data following a particular distribution. Next, we present two examples of internalization.


\begin{example}[Internalizing self-play]\label{ex:self-play}
    We formalize the setting of AlphaGo Zero~\citep{Sil+17}.
    Let $S$ be a set of states and $A$ be a finite set of actions. Let $\texttt{enc}_S: S \to \Sigma^\star$ be an encoding of states using an alphabet $\Sigma$, and let $\texttt{enc}_{Y}: \Delta(A) \times [0, 1] \to \Sigma^\star$ be an encoding of tuples of distributions over actions and state values. Let $\mathcal{X} \subseteq \Sigma^\star$ and $\mathcal{Y} \subseteq \Sigma^\star$ be the input and output spaces. Let $\mathcal{H} = \left\{ h_\theta \middle | \theta \in \Theta \right\} \subseteq \mathcal{Y}^\mathcal{X}$ be a hypothesis class of functions that map (encoded) states to an (encoded) tuple of distribution over actions and winning probability. In~\citet{Sil+17}, this class was realized by a convolutional neural network with two separate heads predicting each element of the output tuple. Assume for simplicity that each state $s$ contains the information of which player is currently playing and this information is encoded in the last character of $\texttt{enc}_S(s)$.

    Fix a $\theta \in \Theta$. Let $\mathfrak{P}_{\mathrm{SP}}(\mathfrak{h}_\theta)$ be a self-play program that receives as input an input $x \in \mathcal{X}$ and computes $P_{\mathrm{SP}}(h_\theta)(x)$. This program performs the following steps:
    \begin{itemize}
        \item[-] Computes an updated distribution $\pi(\cdot | x) \in \Delta(A)$ over actions by performing a search over the tree of possible game continuations of the current state, using a fixed rule \textsc{TreeSearch}.
        \item[-] Samples\footnote{Our definition does not formally allow randomized algorithms. In order to accommodate this, we can embed, for example, the random seed $\rho$ as part of the input $x$.} a next move from $\pi(\cdot | x)$ and advances the board.
        \item[-] Repeats until it reaches a terminal state.
        \item[-] Returns an (encoded) tuple with the updated distribution of the first round (that corresponds to the input state $x$) and a bit that indicates whether the player playing at the input state $x$ ended up winning.
    \end{itemize}
    See Algorithm~\ref{alg:self-play-internalization} for a pseudocode.
    \begin{algorithm}[t]
        \caption{The slow self-play program $\mathfrak{P}_{\mathrm{SP}}(\mathfrak{h}_\theta)$.}
        \label{alg:self-play-internalization}
        \begin{algorithmic}[1]
        \Require A program $\mathfrak{h}_\theta$ and an initial encoded state $x \in \mathcal{X}$.  A tree-search program \textsc{TreeSearch}. A board advancement program \textsc{Board}.
        \State $s_1 \gets x$.
        \State $t \gets 1$.
        \State $\hat{\pi} \gets \mathfrak{h}_\theta[0]$.
        
        \While{$s_t$ is not terminal}
            \State Compute
            \begin{equation}
                \pi_t(\cdot \mid s_t)
                \gets
                \textsc{TreeSearch}(\mathfrak{h}_\theta)(s_t).
            \end{equation}
        
            \If{$t = 1$}
                \State $\hat{\pi} \gets \pi_t(\cdot \mid s_t)$.
            \EndIf
        
            \State Sample a next move
            \begin{equation}
                a_t \sim \pi_t(\cdot \mid s_t).
            \end{equation}
        
            \State Advance the game state
            \begin{equation}
                s_{t+1} \gets \textsc{Board}(s_t,a_t).
            \end{equation}
        
            \State $t \gets t+1$.
        \EndWhile
        
        \State Let $s_{\mathrm{final}} \gets s_t$.
        \State Define the terminal outcome from the perspective of the player to move at state $x$
        \begin{equation}
            z(x)
            :=
            \mathds{1}
            \left\{
            \textsc{Winner}(s_{\mathrm{final}})
            =
            \textsc{Player}(x)
            \right\}.
        \end{equation}
        
        \State \Return $\texttt{enc}_{Y} \left( \left(\hat{\pi},z(x)\right) \right)$.
        \end{algorithmic}
    \end{algorithm}
    
    Suppose there exists $\theta^\prime \neq \theta$ such that 
    \begin{equation}
        h_{\theta^\prime}(x) = P_{\mathrm{SP}}(h_\theta)(x) \qquad \forall x \in \mathcal{X}.
    \end{equation}
    If we assume that each call to $\mathfrak{h}_\phi$ has the same cost for all $\phi \in \Theta$, then, according to Definition~\ref{def:internalization}, $\mathfrak{h}_{\theta^\prime}$ internalizes $\mathfrak{P}_{\mathrm{SP}}(\mathfrak{h}_\theta)$. Note that the algorithm of AlphaGo Zero specified a loss that tried to do exactly that: find weights $\theta^\prime$ so that $\mathfrak{h}_{\theta^\prime}$ would match $\mathfrak{P}_{\mathrm{SP}}(\mathfrak{h}_\theta)$ on all states visited during a game.
\end{example}

\begin{example}[Internalizing checkers self-play]\label{ex:checkers}
    We explain informally how the learning-via-generalization algorithm of~\citet{Sam59} fits within our definition of internalization. The parametric model is a linear model with polynomial features $h_{\theta}(s) = \sum_{j} \theta_j \phi_j(s)$, where $s$ is a state (or an encoding of the state). The slow system runs a minimax search using this linear model for evaluation at the leaves. For example, for a ply of size 2, the slow system is defined as $P(h_\theta)(x) = \max_{a \in A(s)} \min_{b \in A(s_a)} h_\theta(s_{ab})$, where $A$ returns the set of possible moves for a given state and $s_z, \, z \in S^\ast$ denotes the state after actions $z_1, \ldots, z_{|z|}$. The linear model $h_\theta$ is trained to internalize the result of this $\max-\min$ evaluation into its weights.
\end{example}

\section{Representation of \texorpdfstring{$C_n$}{} semiautomata with transformers}

We present here the proof of our representational result on counter semiautomata.

\begin{definition}[Counter $\bmod n$ semiautomaton]
    Let $n \in \mathbb{N}_+$. We define the $\bmod n$ \textit{counter} semiautomaton $C_n$ as the tuple $C_n = \left(Q, \Sigma, \delta_{\bmod \! n} \right)$, where $Q = \Sigma = \mathbb{Z}_n : = \left\{ 0, 1, \ldots, n-1 \right\}$, and $\delta_{\bmod \! n}(a, b) : = \left(a+b\right) \bmod  n$ for all $a, b \in \mathbb{Z}_n$.
\end{definition}
See Figure~\ref{fig:counter_semiautomata} for an illustration of the counter semiautomata $C_2, C_3, C_5$.
\tikzset{
    ->, >=Stealth, shorten >=1pt, auto,
    state/.style={circle, draw, minimum size=0.8cm, font=\bfseries},
    lbl/.style={font=\small, inner sep=1pt}
}
\begin{figure}
    \centering

    \newcommand{\FigWa}{0.30\textwidth}
    \newcommand{\FigWb}{0.30\textwidth}
    \newcommand{\FigWc}{0.30\textwidth}

    \begin{subfigure}[t]{\FigWa}
        \centering
        \resizebox{\linewidth}{!}{%
        \begin{tikzpicture}
            \node[state] (s0) at (0,0) {0};
            \node[state] (s1) at (3,0) {1};

            \draw[gray] (s0) edge [loop left] node[lbl] {0} (s0);
            \draw[gray] (s1) edge [loop right] node[lbl] {0} (s1);

            \draw[red] (s0) edge [bend left=20] node[lbl] {1} (s1);
            \draw[red] (s1) edge [bend left=20] node[lbl] {1} (s0);
        \end{tikzpicture}%
        }
        \subcaption{$C_2$}
        \label{fig:c2}
    \end{subfigure}\hfill
    \begin{subfigure}[t]{\FigWb}
        \centering
        \resizebox{\linewidth}{!}{%
        \begin{tikzpicture}
            \foreach \i in {0,1,2} {
                \node[state] (s\i) at ({90-\i*120}:2cm) {\i};
            }

            \foreach \i in {0,1,2} {
                \draw[gray] (s\i) to[out={90-\i*120+30}, in={90-\i*120-30}, looseness=5]
                    node[lbl, midway] {0} (s\i);

                \pgfmathtruncatemacro{\dest}{mod(\i+1, 3)}
                \draw[red] (s\i) edge [bend left=20] node[lbl] {1} (s\dest);

                \pgfmathtruncatemacro{\dest}{mod(\i+2, 3)}
                \draw[blue] (s\i) edge [bend left=20] node[lbl] {2} (s\dest);
            }
        \end{tikzpicture}%
        }
        \subcaption{$C_3$}
        \label{fig:c3}
    \end{subfigure}\hfill
    \begin{subfigure}[t]{\FigWc}
        \centering
        \resizebox{\linewidth}{!}{%
        \begin{tikzpicture}
            \foreach \i in {0,...,4} {
                \node[state] (s\i) at ({90-\i*72}:3cm) {\i};
            }

            \foreach \i in {0,...,4} {
                \draw[gray] (s\i) to[out={90-\i*72+25}, in={90-\i*72-25}, looseness=6]
                    node[lbl, midway] {0} (s\i);

                \pgfmathtruncatemacro{\dest}{mod(\i+1, 5)}
                \draw[red] (s\i) edge [bend left=15] node[lbl] {1} (s\dest);

                \pgfmathtruncatemacro{\dest}{mod(\i+2, 5)}
                \draw[blue] (s\i) edge [bend left=15] node[lbl] {2} (s\dest);

                \pgfmathtruncatemacro{\dest}{mod(\i+3, 5)}
                \draw[green!60!black] (s\i) edge [bend left=15] node[lbl] {3} (s\dest);

                \pgfmathtruncatemacro{\dest}{mod(\i+4, 5)}
                \draw[orange] (s\i) edge [bend left=15] node[lbl] {4} (s\dest);
            }
        \end{tikzpicture}%
        }
        \subcaption{$C_5$}
        \label{fig:c5}
    \end{subfigure}
    \caption{Counter semiautomata for $n=2, 3, 5$. Each color corresponds to a different function $f: \mathbb{Z}_n \to \mathbb{Z}_n$ that is simply the application of the partial transition function for a certain input symbol, i.e., $\delta_{\bmod\! n}(\cdot, \sigma)$.}
    \label{fig:counter_semiautomata}
\end{figure}
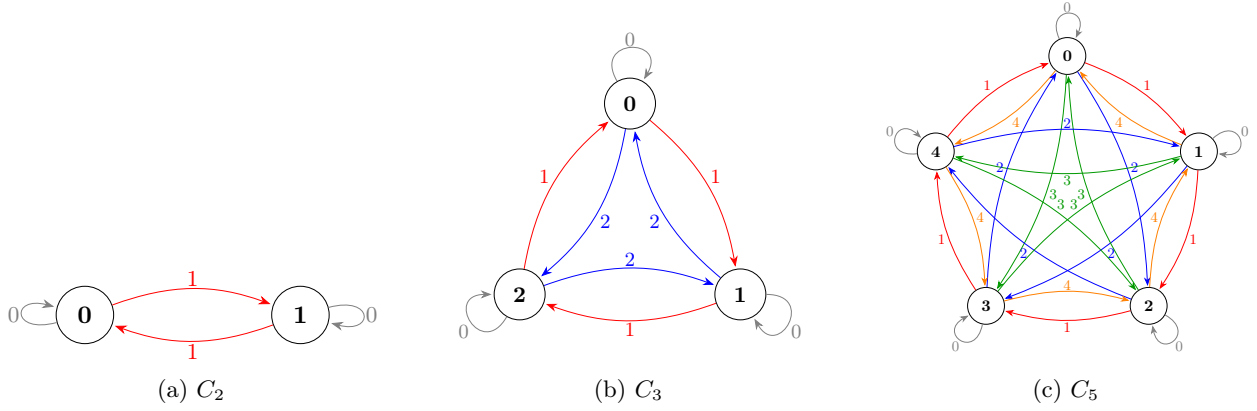

\subsection{Transformer architecture}
We consider a decoder-only transformer with $L$ layers and $H$ heads whose outputs are summed together. We are only interested in the last output of the transformer operating on a sequence. We denote the model by $f_{\mathrm{TF}}: \mathbb{R}^{d_{\mathrm{embd}} \times (T+1)} \times \Theta \to \mathbb{R}^{d_{\mathrm{embd}}}$ and it is defined as an alternating composition of Multi-Layer-Perceptron and self-attention layers:
\begin{equation}
    f_{\mathrm{TF}} = \pi_{\mathrm{last-token}} \circ f_{\mathrm{MLP}}^{(L)} \circ f_{\mathrm{attn}}^{(L)} \circ \ldots \circ f_{\mathrm{MLP}}^{(1)} \circ f_{\mathrm{attn}}^{(1)},
\end{equation}
where $\pi_{\mathrm{last-token}}(X) = X_{:, T+1}$ simply collects the last token's representation.
The input $X \in \mathbb{R}^{d_{\mathrm{embd}} \times (T+1)}$ should be understood as an embedding of the input sequence through an embedding layer $E: \Sigma^T \times Q \to \mathbb{R}^{d_{\mathrm{embd}} \times (T+1)}$. The output of the transformer gets decoded through an un-embedding layer $U: \mathbb{R}^{d_{\mathrm{embd}}} \to Q$.

\paragraph{Self-attention.} A single-headed ($H=1$) self-attention block with attention width $k$ is denoted by $f_{\mathrm{attn}}: \mathbb{R}^{d_{\mathrm{embd}} \times (T+1)} \times \mathbb{R}^{4d_{\mathrm{embd}}k} \to \mathbb{R}^{d_{\mathrm{embd}} \times (T+1)}$ and is parameterized by $\theta_{\mathrm{attn}} = \texttt{flatten}( \begin{bmatrix}
    W_Q, W_K, W_V, W_C
\end{bmatrix}) \in \mathbb{R}^{4kd_{\mathrm{embd}}}$, where $W_Q, W_K, W_V, W_C^\top \in \mathbb{R}^{k \times d_{\mathrm{embd}}}$. For all $X \in \mathbb{R}^{d_{\mathrm{embd}} \times (T+1)}$, the self-attention block is defined as:
\begin{equation}
    f_{\mathrm{attn}} (X; W_Q, W_K, W_V, W_C) = W_C W_V X \texttt{CausalSoftmax} \left( X^\top W_Q^\top W_K X \right)^\top,
\end{equation}
where $\texttt{CausalSoftmax}: \mathbb{R}^{(T+1) \times (T+1)} \to \mathbb{R}^{(T+1) \times (T+1)}$ applies a row-wise causal normalization through the soft-max function. That is, for any $A = [A_{ij}]_{i = 1, j = 1}^{T+1} \in \mathbb{R}^{(T+1) \times (T+1)}$, it holds:
\begin{equation}
    \texttt{CausalSoftmax}(A) = \begin{bmatrix}
        1 & 0 & 0 & \ldots & 0\\
        \frac{e^{A_{21}}}{e^{A_{21}} + e^{A_{22}}} & \frac{e^{A_{22}}}{e^{A_{21}} + e^{A_{22}}} & 0 & \ldots & 0 \\
        \vdots & \vdots & \vdots & \ldots & \vdots \\
        \frac{e^{A_{(T+1)1}}}{\sum_{j = 1}^{T+1}e^{A_{(T+1)j}}} & \frac{e^{A_{(T+1)2}}}{\sum_{j = 1}^{T+1}e^{A_{(T+1)j}}} & \frac{e^{A_{(T+1)3}}}{\sum_{j = 1}^{T+1}e^{A_{(T+1)j}}} & \ldots & \frac{e^{A_{(T+1)(T+1)}}}{\sum_{j = 1}^{T+1}e^{A_{(T+1)j}}}
    \end{bmatrix}.
\end{equation}
We omit a standard normalization by the square root of the attention width in the input of the \texttt{CausalSoftmax} operation.
For multi-head self-attention with $H$ heads, we have $H$ self-attention blocks $f_{\mathrm{attn}, 1}, \ldots, f_{\mathrm{attn}, H}$ each with its own parameters, and we simply sum their outputs:
\begin{equation}
    f_{\mathrm{attn}}\left(X; \left\{ W_Q^{(h)}, W_K^{(h)}, W_V^{(h)}, W_C^{(h)} \right\}_{h = 1}^H\right) = \sum_{h = 1}^{H} W_C^{(h)} W_V^{(h)} X \texttt{CausalSoftmax} \left( X^\top W_Q^{(h)\top} W_K^{(h)} X \right)^\top.
\end{equation}
Some works, or software implementations, opt to have $W_C \in \mathbb{R}^{k \times k}$ (or omit it altogether) with $k = d_{\mathrm{embd}} / H$ and then concatenate rather than sum the outputs of the self-attention heads. 

\paragraph{Multi-Layer Perceptron.} In the transformer architecture, each layer shares one MLP across positions. We consider MLPs with the ReLU activation function.

\subsection{Auxiliary lemmata}

We recall the following lemma on approximating hard arg max with the softmax non-linearity.
\begin{lemma}{(Softmax approximates arg max~\citep{Liu+23})}\label{lem:arg-softmax}
    Let $z \in \mathbb{R}^T$, and let $t^\star = \arg \max_{t \in [T]} z_t$ be the index of the largest element of $z$. Let $\texttt{softmax}: \mathbb{R}^T \to \mathbb{R}^T$ be defined as follows:
    \begin{equation*}
        \texttt{softmax}(z)_t = \frac{e^{z_t}}{\sum_{j = 1}^T e^{z_j}}.
    \end{equation*}
    Then, if there exists $\gamma > 0$ such that $z_{t} \leq z_{t^\star} - \gamma$ for all $t \neq t^\star$, it holds:
    \begin{equation*}
        \left\| \texttt{softmax}(z) - e_{t^\star} \right\|_1 \leq 2Te^{-\gamma}.
    \end{equation*}
\end{lemma}

We will also make use of the following lemma on approximating modular addition with two-layer MLPs.
\begin{lemma}[Robust ReLU implementation of modular addition]\label{lem:robust-mod-add}
Fix $N\in\mathbb{N}_+$ and $\alpha\in(0,1/16]$. There is a scalar one-hidden-layer ReLU network $A_N:\R^2\to\R$ with width at most $4N+2$ and weights bounded by $C_N=O(N)$, for a universal implicit constant, such that for all $a,b\in\mathbb{Z}_N$ and all $u,v\in\R$ satisfying
\[
    |u-a|\le\alpha,
    \qquad
    |v-b|\le\alpha,
\]
one has
\[
    A_N(u,v) = (a+b)\bmod N .
\]
\end{lemma}

\begin{proof}
It is enough to construct a continuous piecewise-affine function $g_N:\R\to\R$ such that
\[
    g_N(x)=s\bmod N
    \qquad\text{whenever}\qquad
    |x-s|\le 2\alpha,
\]
for every integer $s\in\{0,\ldots,2N-2\}$; then set $A_N(u,v)=g_N(u+v)$.

We construct $g_N$ by making it constant on each interval $[s-2\alpha,s+2\alpha]$, with value $s\bmod N$, and linearly interpolating between consecutive constant intervals.  Since $4\alpha\le 1/4$, the interpolation gaps have length at least $3/4$.  The largest jump in the target values is at most $N-1$, so all slopes are $O(N)$.  The resulting function has at most $2(2N-1)\le 4N$ ``breakpoints''.  Every continuous piecewise-affine function with $K$ breakpoints can be written as
\[
    c_0+c_1x+\sum_{r=1}^{K} c_r\max(x-\theta_r, 0),
\]
so it is representable by a one-hidden-layer ReLU network of width $K + 2 \le 4N+2$ (the linear term can be represented by two ReLU neurons). The claimed weight bound follows from the slope bound and the construction.
\end{proof}

\subsection{Main result}

We present next our result on how the semiautomaton $C_n, \, n = \prod_{j = 1}^m p_j^{v_j}$ can be represented by a transformer of $\Theta (\log T)$ layers and $\Theta\left(\sum_{j = 1}^m p_j^{v_j}\right)$ MLP width.

\begin{theorem}[Simulation of $C_n, \, n = \prod_{j = 1}^m p_j^{v_j}$, with $\Theta(\log T)$ depth and $\Theta(\sum_{j = 1}^m p_j^{v_j})$ width]\label{thm:log_depth_log_width_construction}
    Let $n \geq 2$, $q_0 \in \mathbb{Z}_{n}$ and $T \in \mathbb{N}_+$. Let $n=\prod_{j=1}^m p_j^{v_j}$ be its prime factorization, and set $N_j=p_j^{v_j}$ and $N_{\max}=\max_{1\leq j\leq m}N_j$. There exists a depth-$\lceil \log_2(T+1)\rceil$ Transformer $f_{\mathrm{TF}}$ which, together with the embedding and the un-embedding layers, implements $\delta^{(T)}_{\bmod n}(\sigma_1,\ldots,\sigma_T,q_0)=q_0+\sum_{t=1}^T\sigma_t \pmod n$. The construction has embedding dimension $d_{\mathrm{embd}}=2m+2$, MLP width $4\sum_{j=1}^m N_j+4+2m$, $H=2$ heads each of dimension $m+2$, and uses a two-layer ReLU MLP. Moreover, all weights have $\ell_\infty$ norm at most $C\max\{N_{\max},T\sqrt{\log(TN_{\max})}\}$ for a universal constant $C>0$.
\end{theorem}
\begin{proof}
    The proof introduces two basic changes from the generic $O(\log T)$-depth construction of~\citet{Liu+23}. First, instead of embedding each input symbol $\sigma \in \mathbb{Z}_n$ to its transition map $\delta_{\bmod n}(\cdot,\sigma):\mathbb{Z}_n\to\mathbb{Z}_n$, we use the prime-power factorization embedding
    \[
        \sigma \mapsto (\sigma \bmod N_1,\ldots,\sigma \bmod N_m),
        \qquad N_j=p_j^{v_j}.
    \]
    By the Chinese Remainder Theorem, $x \bmod n$ is uniquely determined by $(x\bmod N_1,\ldots,x\bmod N_m)$. Hence the MLP only needs to implement $m$ parallel modular additions, one modulo each $N_j$. Second, the initial state $q_0$ is explicitly fed in the input, hence we require one additional attention layer compared to~\citet{Liu+23} general construction.
    
    As in the construction of~\citet{Liu+23}, we prepend $T$ dummy zero tokens. Thus the transformer is run on the padded sequence
    \[
        y=(y_1,\ldots,y_M)
        :=
        (0,\ldots,0,\sigma_1,\ldots,\sigma_T,q_0),
        \qquad M:=2T+1.
    \]

    Set $L:=\lceil \log_2(T+1)\rceil$ and $\Delta_M:=1-\cos(2\pi/M)$. Choose
    \[
        \eta:=\min\left\{\frac{1}{16N_{\max}},\frac{\Delta_M}{12L},\frac{1}{16}\right\},
        \qquad
        \beta:=\frac{2}{\Delta_M}\log\left(\frac{2M}{\eta}\right).
    \]
    The quantity $\beta$ is the attention-score scale used in the definition of the query-key matrices.

    \paragraph{Embedding.}
    Let $d:=2m+2$ and $k:=m+2$. For each position $i\in[M]$, define the positional vector
    \[
        p_i:=\left(\cos\frac{2\pi i}{M},\sin\frac{2\pi i}{M}\right)^\top\in\mathbb{R}^2 .
    \]
    The initial embedding is
    \[
        X_i^{(0)}
        :=
        \left(
            y_i\bmod N_1,\ldots,y_i\bmod N_m,
            \underbrace{0,\ldots,0}_{m\text{ auxiliary coordinates}},
            p_i^\top
        \right)^\top
        \in\mathbb{R}^{2m+2}.
    \]
    Let $r_j^{(\ell)}(i)$ denote the first-block residue coordinate modulo $N_j$ at position $i$ after the $\ell$th MLP.

    We use the following coordinate extraction matrices. For
    \[
        x=(a_1,\ldots,a_m,b_1,\ldots,b_m,z_1,z_2)^\top\in\mathbb{R}^{2m+2},
    \]
    define
    \[
        E_1x:=(a_1,\ldots,a_m)^\top,\qquad
        E_2x:=(b_1,\ldots,b_m)^\top,\qquad
        E_px:=(z_1,z_2)^\top .
    \]
    Thus $E_1,E_2\in\mathbb{R}^{m\times d}$ and $E_p\in\mathbb{R}^{2\times d}$. Let $I_1:=E_1^\top$, $I_2:=E_2^\top$, and $I_p:=E_p^\top$ be the corresponding insertion matrices. Therefore $I_1u$ inserts $u\in\mathbb{R}^m$ into the first residue block, $I_2u$ inserts it into the auxiliary residue block, and $I_pv$ inserts $v\in\mathbb{R}^2$ into the final two positional coordinates.

    \paragraph{The attention matrices.}
    At layer $\ell\in\{1,\ldots,L\}$, set the stride $s_\ell:=2^{\ell-1}$. Since $L=\lceil\log_2(T+1)\rceil$, we have $s_\ell\le 2^{L-1}\le T$. 

    The construction uses two heads as in the construction of~\citet{Liu+23}.
    The first head attends to itself. Its query, key, value, and output matrices are
    \[
        W_Q^{(\ell,1)}
        :=
        \sqrt{\beta}
        \begin{pmatrix}
            0_{m\times d}\\
            E_p
        \end{pmatrix},
        \qquad
        W_K^{(\ell,1)}
        :=
        \sqrt{\beta}
        \begin{pmatrix}
            0_{m\times d}\\
            E_p
        \end{pmatrix},
    \]
    and
    \[
        W_V^{(\ell,1)}
        :=
        \begin{pmatrix}
            E_1\\
            E_p
        \end{pmatrix},
        \qquad
        W_C^{(\ell,1)}
        :=
        \begin{pmatrix}
            I_1 & I_p
        \end{pmatrix}.
    \]
    These matrices have dimensions $W_Q^{(\ell,1)},W_K^{(\ell,1)},W_V^{(\ell,1)}\in\mathbb{R}^{k\times d}$ and $W_C^{(\ell,1)}\in\mathbb{R}^{d\times k}$. For tokens $X_i,X_j$, the attention score is
    \[
        X_i^\top (W_Q^{(\ell,1)})^\top W_K^{(\ell,1)}X_j
        =
        \beta\,(E_pX_i)^\top(E_pX_j).
    \]
    Moreover,
    \[
        W_C^{(\ell,1)}W_V^{(\ell,1)}X_j
        =
        I_1E_1X_j+I_pE_pX_j,
    \]
    so the self head copies the first residue block and the positional coordinates.

    The second head looks back $s_\ell$ positions. Let $R_{-s_\ell}\in\mathbb{R}^{2\times 2}$ be the rotation matrix
    \[
        R_{-s_\ell}
        :=
        \begin{pmatrix}
            \cos(2\pi s_\ell/M) & \sin(2\pi s_\ell/M)\\
            -\sin(2\pi s_\ell/M) & \cos(2\pi s_\ell/M)
        \end{pmatrix},
    \]
    so that $R_{-s_\ell}p_i=p_{i-s_\ell}$ with indices being interpreted cyclically. Define
    \[
        W_Q^{(\ell,2)}
        :=
        \sqrt{\beta}
        \begin{pmatrix}
            0_{m\times d}\\
            R_{-s_\ell}E_p
        \end{pmatrix},
        \qquad
        W_K^{(\ell,2)}
        :=
        \sqrt{\beta}
        \begin{pmatrix}
            0_{m\times d}\\
            E_p
        \end{pmatrix},
    \]
    and
    \[
        W_V^{(\ell,2)}
        :=
        \begin{pmatrix}
            E_1\\
            0_{2\times d}
        \end{pmatrix},
        \qquad
        W_C^{(\ell,2)}
        :=
        \begin{pmatrix}
            I_2 & 0_{d\times 2}
        \end{pmatrix}.
    \]
    Thus
    \[
        X_i^\top (W_Q^{(\ell,2)})^\top W_K^{(\ell,2)}X_j
        =
        \beta\,(R_{-s_\ell}E_pX_i)^\top(E_pX_j),
    \]
    and
    \[
        W_C^{(\ell,2)}W_V^{(\ell,2)}X_j
        =
        I_2E_1X_j.
    \]
    Therefore the second head copies the first residue block of the attended token into the auxiliary residue block.

    \paragraph{The ideal hard-attention computation.}
    First replace each causal softmax row by its argmax one-hot vector. We later justify this replacement using Lemma~\ref{lem:arg-softmax} and robustness of the MLP.

    Under hard attention, the first head attends from position $i$ to position $i$. For every non-dummy position $i\ge T+1$, the second head attends from position $i$ to position $i-s_\ell$. Therefore the summed two-head output at position $i\ge T+1$ is
    \[
        Y_i^{(\ell)}
        =
        \left(
            r_1^{(\ell-1)}(i),\ldots,r_m^{(\ell-1)}(i),
            r_1^{(\ell-1)}(i-s_\ell),\ldots,r_m^{(\ell-1)}(i-s_\ell),
            p_i^\top
        \right)^\top .
    \]
    On the first $T$ dummy positions, the second head may attend arbitrarily when $i-s_\ell<1$, but this does not affect the residue computation.

    The MLP maps
    \[
        (a_1,\ldots,a_m,b_1,\ldots,b_m,z_1,z_2)^\top
    \]
    to
    \[
        \left(
            A_{N_1}(a_1,b_1),\ldots,A_{N_m}(a_m,b_m),
            \underbrace{0,\ldots,0}_{m\text{ auxiliary coordinates}},
            z_1,z_2
        \right)^\top ,
    \]
    where each $A_{N_j}$ is the robust modular-adder network from Lemma~\ref{lem:robust-mod-add} with robustness parameter $\alpha=1/16$. The two identity coordinates $z_1,z_2$ are implemented by $z=\operatorname{ReLU}(z)-\operatorname{ReLU}(-z)$, using two hidden units per coordinate. Hence the MLP width is at most $4\sum_{j=1}^m N_j+2m+4$.

    For hard attention, this gives the exact recurrence, for every non-dummy position $i\ge T+1$,
    \[
        r_j^{(\ell)}(i)
        =
        r_j^{(\ell-1)}(i)
        +
        r_j^{(\ell-1)}(i-s_\ell)
        \pmod{N_j}.
    \]
    By induction on $\ell$,
    \[
        r_j^{(\ell)}(i)
        =
        \sum_{u=i-2^\ell+1}^{i} y_u
        \pmod{N_j}
    \]
    for every non-dummy position $i\ge T+1$, with the convention $y_u=0$ for $u<1$. The base case $\ell=0$ is the embedding. For the induction step, the current term $r_j^{(\ell-1)}(i)$ gives the window $[i-2^{\ell-1}+1,i]$. The shifted term is at $k=i-s_\ell=i-2^{\ell-1}$. If $k\ge T+1$, the induction hypothesis gives the window $[i-2^\ell+1,i-2^{\ell-1}]$. If $k\le T$, then the shifted position is dummy and its residue is zero; the same shifted window also has zero sum, because all its indices are at most $T$ or below $1$. Thus the two windows combine to $[i-2^\ell+1,i]$.

    At the last position $M=2T+1$ and depth $L$, the window length is $2^L\ge T+1$. Therefore the window ending at $M$ contains all nonzero semantic tokens $\sigma_1,\ldots,\sigma_T,q_0$ and may contain only dummy zeros before them. Hence
    \[
        r_j^{(L)}(M)
        =
        q_0+\sum_{t=1}^T\sigma_t
        \pmod{N_j}
    \]
    for every $j\in[m]$.

    \paragraph{Softmax approximation.}
    It remains to justify replacing hard attention by ordinary causal softmax attention with finite weights. The two heads are analyzed by the same argument.

    For the ideal positional vectors, if the target position is $t^\star$, then $p_{t^\star}^\top p_{t^\star}=1$, while for every other position $j\neq t^\star$,
    \[
        p_{t^\star}^\top p_j
        \leq
        \cos\frac{2\pi}{M}
        =
        1-\Delta_M .
    \]
    Thus the ideal unscaled dot-product gap is $\Delta_M$.

    Because the positional coordinates are copied through softmax attention, they are not exact after the first layer (note that this is because we do not use a residual connection in our theoretical construction). However, we prove the following invariant. After layer $\ell$, all first-block residue coordinates are exact, all auxiliary residue coordinates are zero, and each positional vector differs from its ideal $p_i$ by at most $\ell\eta$ in Euclidean norm.

    The invariant is true at $\ell=0$. Suppose it holds at layer $\ell-1$. Since $(\ell-1)\eta\le L\eta\le \Delta_M/12$, the perturbation of any positional dot product is at most $3L\eta$, and the perturbation of the difference between a target score and a non-target score is at most $6L\eta\le \Delta_M/2$. Therefore, after multiplying by $\beta$, every relevant causal row has argmax gap at least
    \[
        \gamma=\frac{\beta\Delta_M}{2}.
    \]
    For the first head this holds for every row. For the second head this holds for every non-dummy row $i\ge T+1$, whose target is $i-s_\ell$. On dummy rows, all causal residue values are zero, so no selection guarantee is needed.

    Applying Lemma~\ref{lem:arg-softmax} to a causal row of length at most $M$ gives
    \[
        \left\|\texttt{CausalSoftmax}(A)_{i,1:i}-e_{t^\star}\right\|_1
        \leq
        2M e^{-\beta\Delta_M/2}
        \leq
        \eta
    \]
    by the definition of $\beta$.

    Consequently, every copied residue coordinate is within $\eta N_{\max}\le 1/16$ of the corresponding exact residue. The MLP modular adders from Lemma~\ref{lem:robust-mod-add} therefore output the exact modular sums. This proves exactness of the first residue block and zeroing of the auxiliary residue block at layer $\ell$.

    For positional coordinates, the first head outputs an $\eta$-accurate convex combination around the target position. More explicitly, if the previous positional error was at most $(\ell-1)\eta$, then
    \[
        \left\|
            \sum_{j\le i} w_j E_pX_j^{(\ell-1)} - p_i
        \right\|_2
        \leq
        (\ell-1)\eta+\eta
        =
        \ell\eta,
    \]
    where $w$ is the first head's attention row and $\|w-e_i\|_1\le \eta$. The MLP passes the two positional coordinates through the identity. Thus, the invariant holds at layer $\ell$.

    By induction, the finite-softmax transformer has the same exact residue coordinates as the hard-attention transformer after every MLP layer.

    \paragraph{Decoding.}
    At the last token, the first $m$ coordinates are exactly
    \[
        \left(
            q_0+\sum_{t=1}^T\sigma_t \bmod N_1,
            \ldots,
            q_0+\sum_{t=1}^T\sigma_t \bmod N_m
        \right)^\top .
    \]
    The unembedding $U$ applies the inverse Chinese remainder map, yielding the unique element of $\mathbb{Z}_n$ with these residues. Therefore
    \[
        U\,f_{\mathrm{TF}}(E(\sigma_1,\ldots,\sigma_T,q_0))
        =
        q_0+\sum_{t=1}^T\sigma_t \pmod n .
    \]
    Finally, the nonzero entries in the attention matrices are bounded by $\sqrt{\beta}$ or $1$, and the MLP weights are bounded by $C_{\mathrm{MLP}}N_{\max}$ for a universal constant $C_{\mathrm{MLP}}$. Since $\Delta_M^{-1}=O(T^2)$ and $\eta^{-1}=O(T^2L N_{\max})$, we have $\sqrt{\beta}=O(T\sqrt{\log(TN_{\max})})$. Hence all weights have $\ell_\infty$ norm at most $C\max\{N_{\max},T\sqrt{\log(TN_{\max})}\}$ for a universal constant $C>0$. This completes the proof.
\end{proof}

\section{Learning Parity in 1-Layer Transformer with Linear Attention}\label{sec:app:proof-provable-parity-main}
We first start by specifying the architecture, curriculum, and details of the algorithm, and we then proceed to the proof.
\subsection{Architecture, Left Removal Curriculum, and Algorithm}\label{subsec:architecture}

\paragraph{Causal linear attention architecture.} Let $D$ be the maximum sequence length, which will be $d+k$ for our setup in internalizing parities. Our architecture, as specified below, is parameterized by a single attention head score matrix $W\in \R^{D \times D}$. We will fix other components of the architecture such as MLP weights, and value matrix. Each of the \(D\) tokens is represented in dimension \(D+1\). The first coordinate is a scalar content coordinate, and the remaining \(D\) coordinates are one-hot positional encodings.

Given a content vector
\( 
a=(a_1,\ldots,a_D)\in\mathbb{R}^D,
\)
we form the token matrix
\[
X(a)
=
\begin{pmatrix}
a_1 & 1 & 0 & \cdots & 0\\
a_2 & 0 & 1 & \cdots & 0\\
\vdots & \vdots & \vdots & \ddots & \vdots\\
a_D & 0 & 0 & \cdots & 1
\end{pmatrix}
\in\mathbb{R}^{D\times(D+1)}.
\]
Equivalently, the \(i\)-th token is
\[
X_i(a)=(a_i,e_i)\in\mathbb{R}^{1+D},
\]
where \(e_i\in\mathbb{R}^D\) is the \(i\)-th standard basis vector.

We use a single attention head without softmax normalization. The query-key scores are bilinear:
\[
S=XW_{QK}X^\top.
\]
Here
\( 
W_{QK}\in\mathbb{R}^{(D+1)\times(D+1)}.
\) Equivalently, if \(Q=XW_Q\) and \(K=XW_K\), then
\( 
QK^\top=XW_QW_K^\top X^\top,
\)
and we collapse the parameterization
\( 
W_{QK}:=W_QW_K^\top
\)
into a single parameter. We always fix the value matrix to be the identity,
\[
W_V=I_{D+1}.
\]

We constrain \(W_{QK}\) to ignore the scalar content coordinate and to be causal in the positional block. Namely, we write
\[
W_{QK}
=
\begin{pmatrix}
0 & 0\\
0 & W
\end{pmatrix}, \quad \text{ where } \quad W\in\mathbb{R}^{D\times D}\,,
\]
is an unconstrained trainable matrix. The first row and column of \(W_{QK}\) are zero, so the bilinear query-key scores depend only on the positional coordinates. For tokens \(i,j\in[D]\), this gives
\[
X_i(a)^\top W_{QK}X_j(a)
=
(a_i,e_i)^\top
\begin{pmatrix}
0 & 0\\
0 & W
\end{pmatrix}
(a_j,e_j)
=
e_i^\top W e_j
=
W_{ij}.
\]
Hence
\[
X(a)W_{QK}X(a)^\top=W.
\]

Causality is enforced by applying the lower-triangular mask
\[
M\in\{0,1\}^{D\times D},
\qquad
M_{ij}
=
\begin{cases}
1, & j\le i,\\
0, & j> i.
\end{cases}
\]
Thus token \(i\) may attend only to tokens at positions \(j\le i\), including itself. The masked score matrix is
\[
M\odot X(a)W_{QK}X(a)^\top
=
M\odot W,
\]
where \(\odot\) denotes entrywise multiplication. Since we fix the value matrix to be the identity, the causal linear-attention output is
\[
H(X(a))
=
\left(M\odot X(a)W_{QK}X(a)^\top\right)X(a)
=
(M\odot W)X(a).
\]

\paragraph{Output of Attention on shorter input length} The same construction is consistent under restriction to shorter sequences. If we use the same input representation but restrict the sequence to the first \(D'\leq D\) tokens, then the restricted token matrix is
\[
X_{[D']}(a)
=
X(a)_{1:D',:}
\in\mathbb{R}^{D'\times(D+1)}.
\]
The bilinear score matrix becomes
\[
X_{[D']}(a)W_{QK}X_{[D']}(a)^\top
=
W_{[D']\times[D']},
\]
the upper-left \(D'\times D'\) restriction of \(W\). After applying the \(D'\times D'\) causal mask \(M^{(D')}\), defined by
\( 
M^{(D')}_{ij} = \mathbf{1}\{j \leq i\}\)
the corresponding causal attention output is
\[
H_{[D']}(X(a))
=
\left(M^{(D')}\odot W_{[D']\times[D']}\right)X_{[D']}(a).
\]
Thus restricting the input to its first \(D'\) tokens simply restricts the positional score matrix \(W\) to its first \(D'\) rows and columns.

\paragraph{Final output with fixed MLP} For each output position \(i\in[D]\), we define the scalar prediction
\[
\widehat y_i:=f_{W}(a)_i=\sigma\!\left((H_i(X(a)))_1\right).
\]
This can be viewed as applying a fixed MLP/readout to $H_i(X(a))_1 \in \R$, followed by the scalar nonlinearity \(\sigma:\R \to \R\).
Equivalently, the vector of predictions is
\[
\widehat y=f_W(a)=\left(\sigma\!\left((H_1(X(a)))_1\right),\dots,\sigma\!\left((H_D(X(a)))_1\right)\right) \in \R^{D}.
\]

For any position \(i\),
\[
H_i(X(a))=\sum_{j\le i}W_{ij}X_j(a), \quad \text{ and hence } \quad (H_i(X(a)))_1=\sum_{j\le i}W_{ij}a_j\,.
\]

\subsection{Curriculum for Internalizing the Input}
\paragraph{Curriculum for learning parities}
We now describe how the above architecture is used for learning parities. The input is
\( 
x=(x_1,\ldots,x_d)\in\{0,1\}^d,
\)
and the CoT initially consists of \(k\) binary tokens \( 
z=(z_1,\ldots,z_k)\in\{0,1\}^k.
\)
Thus the full content vector has length \(D=d+k\), and is given by
\[
a^{[1]}=(x_1,\ldots,x_d,z_1,\ldots,z_k)\in\{0,1\}^{d+k}.
\]

We use the CoT tokens as intermediate supervision. In the full sequence, the token \(z_i\) appears at position \(d+i\). As to be specified, the prediction at a position is computed from the tokens up to that position. We train the model to predict the CoT entries
\( 
z_1,\ldots,z_k
\)
at the corresponding positions $
d+1,\ldots,d+k\,, $
in the first phase. The model learns to compute parity, by \emph{explicitly} outputting the entire CoT, in the first phase. Equivalently, for each \(i\in[k]\), the prediction at position \(d+i\) is trained against the target \(z_i\).

\paragraph{Phases of curriculum}  The curriculum proceeds by progressively shortening the CoT. In the first phase, the model is trained on the full sequence
\( 
(x_1,\ldots,x_d,z_1,\ldots,z_k).
\) In the next phase, we remove the first CoT token and train on
\( 
(x_1,\ldots,x_d,z_2,\ldots,z_k).
\)
In Phase $t$ for $1 \leq t \leq k$: we remove the first \(t-1\) CoT tokens and train on the shortened sequence
\[\text{Input Sequence in Phase } t: \quad \quad 
a^{[t]}=(x_1,\ldots,x_d,z_t,z_{t+1},\ldots,z_k).
\]
At the end of $k$ phases the input sequence only contains $(x_1,\dots,x_d)$. The desired goal is that, by the end of such a curriculum, the model learns to directly produce $z_k$ using our architecture.

\paragraph{Output of the Architecture} In the first phase ($t=1$), the content vector is
\[
a^{[1]}=(x_1,\ldots,x_d,z_1,\ldots,z_k)\in\{0,1\}^{d+k}.
\]
For every CoT position \(d+i\), where \(0\leq i \leq k\), the output of the architecture 
\[
\widehat y_{d+i}^{[1]} = f_W(a^{[1]})_{d+i}
=
\sigma\!\left(\sum_{j\le d+i} W_{d+i,j}a^{[1]}_j\right).
\]
In Phase \(t\), similarly, the content vector has CoT tokens removed and is
\begin{equation}\label{eq:input-phase-t}
        a^{[t]}=(x_1,\ldots,x_d,z_t,z_{t+1},\ldots,z_k) \in \{0,1\}^{d+k-t+1}.
\end{equation}
Since the first CoT token in phase \(t\) is now \(z_t\), the output for \(0 \leq i \leq k-t+1\) is
\begin{equation}\label{eq:output-at-d+i}
\widehat y_{d+i}^{[t]}
=f_W(a^{[t]})_{d+i}=
\sigma\!\left(\sum_{j\le d+i}W_{d+i,j}a^{[t]}_j\right)\,.
\end{equation}
At the end of final $(k^\mathrm{th})$ phase, after all the CoT tokens have been removed, the content vector and the output at $d^\mathrm{th}$ position are respectively given by 
\[
a^{[k+1]}=(x_1,\ldots,x_d) \quad \text{ and } \quad  
\widehat y_{d}^{[k+1]}
=
\sigma\!\left(\sum_{j=1}^{d}W_{d,j}x_j\right).
\]

\paragraph{Fixed non-linear activation}
We use a fixed piecewise-linear activation \(\sigma:\R\to\R\) that agrees with parity on the integers:
\( 
\sigma(m)= m \mod 2\) for $m \in \mathbb{Z}$. Between consecutive integers, \(\sigma\) is defined by linear interpolation. Equivalently, for \(s\in[m,m+1]\),
\begin{equation}\label{eq:non-linearity}
    \sigma(s)=
\begin{cases}
s-m, & m \text{ even},\\
1-(s-m), & m \text{ odd}.
\end{cases}
\end{equation}
At integer breakpoints, \(\sigma\) is not differentiable. We fix the left-derivative convention:
\begin{equation}\label{eq:non-linearity-derivative}
    \sigma'(m)
=
\lim_{\epsilon\downarrow 0}
\frac{\sigma(m)-\sigma(m-\epsilon)}{\epsilon}
=
\begin{cases}
-1, & m \text{ even},\\
+1, & m \text{ odd}.
\end{cases}
\end{equation}
Away from integer points, \(\sigma'\) is the usual derivative.

\begin{algorithm}
\caption{\textsc{Internalize-Parity} via CoT Shortening}
\label{alg:internalize-parity-shortening}
\begin{algorithmic}[1]
\Require Input dimension \(d\), CoT length \(k\), step size \(\eta\), per-phase number of steps \(T\), batch size \(B\)
\Require Training distribution over \((x,z_1,\ldots,z_k)\in\{0,1\}^d\times\{0,1\}^k\)
\State Initialize \(W^{(0)}\gets 0^{(d+k)\times(d+k)}\) and then reset the rows for phase-one next-token CoT prediction:
\[
W^{(0)}_{d+l-1,d+l-1}\gets -1
\qquad\text{for all }l\in\{2,\ldots,k\}.
\]
\For{\(t=1,\ldots,k\)}
    \State Sample a phase-\(t\) batch
    \( 
    \mathcal{B}_t=\{(x^{(b)},z^{(b)}_1,\ldots,z^{(b)}_k)\}_{b=1}^{B}
    \)
    \State For each \(b\in[B]\), form the shortened phase-\(t\) input
    \[
    a^{[t,b]}=(x^{(b)}_1,\ldots,x^{(b)}_d,z^{(b)}_t,\ldots,z^{(b)}_k).
    \]
    \If{\(t=1\)}
        \State The empirical loss in Phase-\(1\) is next-token CoT loss
        \[
        \widehat L_1(W)
        =
        \frac1{2B}
        \sum_{b=1}^{B}
        \sum_{l=1}^{k}
        \left(
        f_W(a^{[1,b]})_{d+l-1}
        -
        a^{[1,b]}_{d+l}
        \right)^2 .
        \]
    \Else
        \State The empirical loss is only on $d^\mathrm{th}$ token
        \[
        \widehat L_t(W)
        =
        \frac1{2B}
        \sum_{b=1}^{B}
        \left(
        f_W(a^{[t,b]})_{d}
        -
        a^{[t,b]}_{d+1}
        \right)^2
        =
        \frac1{2B}
        \sum_{b=1}^{B}
        \left(
        f_W(a^{[t,b]})_{d}
        -
        z^{(b)}_{t}
        \right)^2 .
        \]
    \EndIf  
    \State Make \(T\) updates starting \(W\gets W^{(t-1)}\) on empirical Phase-\(t\) loss $\widehat{L_t}(\cdot)$.
     \For{\(s=1,\ldots,T\)}
            \State Update the row corresponding to the first CoT/output position:
            \[
            W
            \gets
            W
            -
            \eta \nabla_{W}\widehat L_t(W)
            \]
        \EndFor
   
    \State Let \(W^{(t)}\) be the final iterate of Phase \(t\), and round entries to nearest integers \(W^{(t)}\gets \mathtt{Round}(W^{(t)})\).
\EndFor
\State \Return \(W^{(k)}\)
\end{algorithmic}
\end{algorithm}

\subsection{Analysis}\label{subsec:parity-analysis}

We prove \Cref{thm:provable-parity-main} for
\Cref{alg:internalize-parity-shortening}. Throughout, the support coordinates
$i_1,\ldots,i_k$ are distinct, $1\leq k\leq d$, and
$x\sim\operatorname{Unif}(\{0,1\}^d)$. Write
\[
 z_0(x)=0,\qquad z_t(x)=\bigoplus_{h=1}^t x_{i_h},\quad t\in[k].
\]
The activation $\sigma$ is the piecewise-linear interpolation of parity on
$\mathbb Z$. At an integer $m$, gradients use its left derivative,
\begin{equation}\label{eq:analysis-left-derivative}
 \sigma'_-(m)=2\sigma(m)-1\in\{-1,1\}.
\end{equation}
At all other arguments, they use its usual derivative. We use the loss
normalization in the algorithm: each supervised row contributes a batch
average with factor $1/(2B)$, and Phase~1 sums these contributions without
an additional average over rows.

Recall that the output at position $j$ is
\begin{equation}\label{eq:analysis-row-model}
 f_W(a)_j=\sigma\!\left(\sum_{h\leq j}W_{jh}a_h\right).
\end{equation}
Consequently, the supervised rows have disjoint trainable parameters. Updating
all of $W$ is therefore equivalent to updating each supervised row on its own
loss; unsupervised rows and causally masked entries have zero gradient.
Let $W^{(t,s)}$ denote the iterate after $s$ updates in Phase~$t$, with
$W^{(t,0)}=W^{(t-1)}$, and let $W^{(t)}$ be the result of rounding
$W^{(t,T)}$ coordinatewise to the nearest integer.

For row $d$, the available features are $x$ in every phase. For the row
predicting $z_l$ in Phase~1, the available features are
\begin{equation}\label{eq:analysis-feature-families}
 q_l(x)=(x_1,\ldots,x_d,z_1,\ldots,z_{l-1})
 \in\{0,1\}^{d+l-1},\qquad l\in[k],
\end{equation}
where $q_1(x)=x$. We first analyze a common row-learning problem for these
feature vectors, and then apply it to all supervised rows and phases.

\paragraph{Quadratic regions of the squared parity loss.}
The squared loss remains quadratic on an interval centered at the target
bit, including across a breakpoint where the prediction matches that bit.
This observation allows us to control the actual gradient-descent trajectory.
\begin{lemma}[Quadratic loss identity]
\label{lem:analysis-quadratic-loss}
For $m\in\mathbb Z$, $p\in\{0,1\}$, and $|v-p|\leq1$,
\begin{equation}\label{eq:analysis-quadratic-identity}
g_{m,p}(v):=
\frac12\bigl(\sigma(m-v)-\sigma(m-p)\bigr)^2
=\frac12(v-p)^2.
\end{equation}
With the prescribed derivative convention for $\sigma$, the chain-rule
derivative equals $v-p$ for $|v-p|<1$, and equals $-p$ at $v=0$.
\end{lemma}

\begin{proof}
For every integer $n$ and $|h|\leq1$, the definition of $\sigma$ gives
\[
|\sigma(n+h)-\sigma(n)|=|h|.
\]
Taking $n=m-p$ and $h=p-v$, then squaring, proves the identity.
Differentiation gives the interior derivative. At $v=p$, the residual
vanishes, so the prescribed chain rule also gives zero.

At $v=0$, we have
$\sigma(m)-\sigma(m-p)=p\,\sigma'_-(m)$.
Thus the prescribed chain-rule derivative is
\[
-\bigl(\sigma(m)-\sigma(m-p)\bigr)\sigma'_-(m)
=-p\bigl(\sigma'_-(m)\bigr)^2=-p.
\]
\end{proof}

\paragraph{Second moments of the available features.}
We first establish the matrix bounds needed to analyze the regression
updates. Recall that row $d$ uses $q_1(x)=x$ in every curriculum phase.
The additional rows trained in Phase~1 use
\[
q_l(x)=(x_1,\ldots,x_d,z_1,\ldots,z_{l-1}),
\qquad l=2,\ldots,k.
\]
Thus $l$ specifies which CoT token is being predicted in Phase~1;
it does not denote the curriculum phase.

Fix one of these feature vectors $q=q_l$, and let $N=d+l-1$.
Its population and empirical second moments are
\[
\Sigma_q:=\mathbb E[q(x)q(x)^\top],\qquad
\widehat\Sigma_q:=\frac1B\sum_{b=1}^B
q(x^{(b)})q(x^{(b)})^\top,
\]
where the training examples are independent and uniform on $\{0,1\}^d$.

For $l\geq2$, the only repeated feature is $z_1=x_{i_1}$.
To calculate the second moment, let
\[
\phi(x)=
\begin{cases}
x, & l=1,\\
(x_1,\ldots,x_d,z_2,\ldots,z_{l-1}), & l\geq2,
\end{cases}
\]
and write $n=\dim(\phi)$. Let $R\in\{0,1\}^{N\times n}$ restore the
removed copy: it inserts $\phi_{i_1}=x_{i_1}$ at position $d+1$,
leaving the other entries in their original order. For $l=1$,
take $R=I_d$. In both cases, $q=R\phi$.

The repeated features receive identical gradient updates. We therefore
need a positive eigenvalue bound only on the subspace containing
these updates:
\[
U:=\operatorname{range}(R)=
\begin{cases}
\mathbb R^d, & l=1,\\
\{v\in\mathbb R^N:v_{i_1}=v_{d+1}\}, & l\geq2.
\end{cases}
\]
Every sample vector belongs to $U$, and every gradient update is a
linear combination of sample vectors. The following lemma bounds
the population second moment on $U$ and its empirical approximation.

\begin{lemma}[Feature geometry and concentration]
\label{lem:analysis-feature-moments}
For the feature maps above,
\begin{equation}\label{eq:analysis-feature-moments}
\Sigma_q=\frac14 R(I_n+\mathbf1\mathbf1^\top)R^\top.
\end{equation}
Both $\Sigma_q$ and $\widehat\Sigma_q$ annihilate $U^\perp$ and map into $U$.
Moreover,
\[
\lambda_{\min}(\Sigma_q|_U)\geq\frac14,\qquad
\|\Sigma_q\|_{\mathrm{op}}\leq N,\qquad
\|\widehat\Sigma_q\|_{\mathrm{op}}\leq N.
\]
For every $\varepsilon>0$,
\begin{equation}\label{eq:analysis-moment-concentration}
\Pr\!\left[
\|\widehat\Sigma_q-\Sigma_q\|_{\mathrm{op}}
>\frac{\varepsilon}{\sqrt N}
\right]
\leq2N^2\exp\!\left(-\frac{2B\varepsilon^2}{N^3}\right).
\end{equation}
\end{lemma}

\begin{proof}
The entries of $\phi$ are parities on distinct nonempty supports.
Under the uniform input distribution, each is a fair bit and any
two are independent. Consequently,
\[
\mathbb E[\phi\phi^\top]
=\frac14(I_n+\mathbf1\mathbf1^\top).
\]
Since $q=R\phi$, this proves
\eqref{eq:analysis-feature-moments}.

Every sample vector belongs to $U$, so both second-moment matrices
map into $U$ and vanish on $U^\perp$. Moreover,
\[
\Sigma_q\succeq\frac14RR^\top.
\]
The nonzero eigenvalues of $RR^\top$ are those of $R^\top R$.
This latter matrix is diagonal, with an entry equal to the number
of copies of each feature: one for every feature, except two for
the duplicated feature when $l\geq2$.
Thus every nonzero eigenvalue is at least one, giving
$\lambda_{\min}(\Sigma_q|_U)\geq1/4$.
The upper bounds follow from $\|q(x)\|_2^2\leq N$.

For concentration, fix $i,j$ and set
$Y_b=q_i(x^{(b)})q_j(x^{(b)})$.
These variables are independent across $b$ and belong to $[0,1]$.
The scalar two-sided Hoeffding inequality gives
\[
\Pr\!\left[
|(\widehat\Sigma_q-\Sigma_q)_{ij}|>t
\right]\leq2e^{-2Bt^2}.
\]
Taking a union bound over the $N^2$ entries and setting
$t=\varepsilon/N^{3/2}$ yields
\[
\Pr\!\left[
\max_{i,j}|(\widehat\Sigma_q-\Sigma_q)_{ij}|
>\frac{\varepsilon}{N^{3/2}}
\right]
\leq2N^2\exp\!\left(-\frac{2B\varepsilon^2}{N^3}\right).
\]
Finally, $\|A\|_{\mathrm{op}}\leq N\max_{i,j}|A_{ij}|$ for every
$N\times N$ matrix $A$, proving
\eqref{eq:analysis-moment-concentration}.
\end{proof}

\paragraph{Population and empirical regression trajectories.}
We now study the regression trajectories used to analyze the addition
of one new input bit. We fix one row-learning task:
\[
(q,r)=
\begin{cases}
(x,i_t),
& \text{for row }d\text{ in curriculum Phase }t,\\
(q_l,i_l),
& \text{for the row predicting }z_l\text{ in Phase }1,\quad l\geq2.
\end{cases}
\]
Thus $x_r$ is the new input bit to be learned, and $q$ is the vector
available to that row. The first $d$ entries of $q$ are the raw inputs,
so $q_r(x)=x_r$. Write $N=\dim(q)$.

In both cases, $x_r$ occurs only once in $q$. For $q=x$, this is
immediate. For $q=q_l$, the only repeated feature is $x_{i_1}=z_1$,
whereas the new bit is $x_{i_l}$ with $i_l\neq i_1$.
Consequently, $e_r\in U$: it either belongs to the full space
$U=\mathbb R^d$, or has zero entries at both duplicate positions.

The desired regression coefficient vector is $e_r$, since
$\langle e_r,q(x)\rangle=x_r$.
We compare two auxiliary sequences, both initialized at zero.
The population sequence uses the population second moment:
\begin{equation}\label{eq:analysis-population-regression}
u_{s+1}=u_s-\eta\Sigma_q(u_s-e_r),\qquad u_0=0.
\end{equation}
The empirical sequence uses the second moment of the fixed training batch:
\begin{equation}\label{eq:analysis-empirical-regression}
\widehat u_{s+1}
=\widehat u_s-\eta\widehat\Sigma_q(\widehat u_s-e_r),
\qquad \widehat u_0=0.
\end{equation}
These are the population and empirical gradient-descent updates for
predicting $x_r$ by the linear prediction $\langle u,q(x)\rangle$.

Our goal is to bound the prediction error
\[
\langle u,q(x)\rangle-x_r
=\langle u-e_r,q(x)\rangle.
\]
An error strictly smaller than one allows us to apply the quadratic-loss
identity. We first prove this bound for the population sequence, then
show that the empirical sequence stays close enough to preserve it.

\begin{lemma}[Population prediction margin]
\label{lem:analysis-population-margin}
For either choice of $(q,r)$ above, suppose $0<\eta\leq1/N$.
Then, for every $s\geq0$,
\begin{equation}\label{eq:analysis-population-margin}
\|u_s-e_r\|_2\leq1,\qquad
|\langle u_s-e_r,q(x)\rangle|
\leq(1-\eta/4)^s
\quad\text{for every }x\in\{0,1\}^d.
\end{equation}
\end{lemma}

\begin{proof}
By \Cref{lem:analysis-feature-moments}, the eigenvalues of
$\Sigma_q|_U$ belong to $[1/4,N]$.
Since $u_0-e_r=-e_r\in U$,
\[
u_{s+1}-e_r=(I-\eta\Sigma_q)(u_s-e_r)
\]
is a contraction on $U$. This proves the norm bound.

For the prediction bound, recall that $q=R\phi$, where $\phi$ lists
the distinct features. Define
\[
b_s:=R^\top u_s,\qquad A_s:=\sum_j b_{s,j}.
\]
The vector $b_s$ gives the total coefficient on each distinct feature:
when a feature appears twice, its coefficient in $b_s$ is the sum
of the two coefficients in $u_s$. In particular,
\[
\langle u_s,q(x)\rangle=b_s^\top\phi(x).
\]
The first $d$ entries of $\phi$ are still $x_1,\ldots,x_d$, so the
target bit $x_r$ also has index $r$ in $\phi$.

Let $\mu_j$ be the number of copies of $\phi_j$ in $q$, and write
\[
\Lambda:=R^\top R
=\operatorname{diag}(\mu_1,\ldots,\mu_n).
\]
Each $\mu_j$ is one or two. The target occurs once, so $\mu_r=1$,
and $\sum_j\mu_j=N$.
The population recursion and the second-moment formula give
\begin{equation}\label{eq:analysis-aggregated-dynamics}
b_{s+1}
=b_s-\frac{\eta}{4}\Lambda
\bigl(b_s-e_r+(A_s-1)\mathbf1\bigr).
\end{equation}
Here $e_r$ in this equation is the $r$th basis vector in
$\mathbb R^n$.

We first show that the coefficients remain nonnegative and their
sum remains at most one. Both properties hold at $s=0$.
If they hold at step $s$, then
\[
b_{s+1,j}
=(1-\eta\mu_j/4)b_{s,j}
+\frac{\eta\mu_j}{4}
\bigl(1-A_s+\mathbf1\{j=r\}\bigr)
\geq0.
\]
Also, summing the updates gives
\[
1-A_{s+1}
=\left(1-\frac{\eta(N+1)}4\right)(1-A_s)
+\frac{\eta}{4}\sum_j(\mu_j-1)b_{s,j}
\geq0.
\]
The coefficients in these expressions are nonnegative because
$\eta\leq1/N$. Induction therefore gives
$b_{s,j}\geq0, A_s\leq1 $ for every $s$. The coefficient on the target bit satisfies
\[
b_{s+1,r}
=(1-\eta/4)b_{s,r}+\frac{\eta}{4}(2-A_s)
\geq(1-\eta/4)b_{s,r}+\eta/4.
\]
Since $b_{0,r}=0$, it follows that
\[
b_{s,r}\geq1-(1-\eta/4)^s.
\]

These coefficient bounds imply the prediction bound on every input.
If $x_r=0$, then
\[
0\leq b_s^\top\phi(x)
\leq A_s-b_{s,r}\leq1-b_{s,r}.
\]
If $x_r=1$, then
\[
b_{s,r}\leq b_s^\top\phi(x)\leq A_s\leq1.
\]
In either case,
\[
|\langle u_s,q(x)\rangle-x_r|
\leq1-b_{s,r}\leq(1-\eta/4)^s,
\]
as required.
\end{proof}

The population prediction error is therefore strictly below one
after every positive number of steps. We next transfer this property
to the empirical sequence. Since the batch is reused, all empirical
updates depend on the same matrix $\widehat\Sigma_q$.
It suffices to control this matrix once.

Fix $\varepsilon_0=1/128$, and define
\begin{equation}\label{eq:analysis-second-moment-event}
\mathcal E(q):=
\left\{
\|\widehat\Sigma_q-\Sigma_q\|_{\mathrm{op}}
\leq\frac{\varepsilon_0}{\sqrt N}
\right\}.
\end{equation}

\begin{lemma}[Empirical prediction margin and convergence]
\label{lem:analysis-empirical-margin}
For either choice of $(q,r)$ above, suppose $0<\eta\leq1/N$
and $\mathcal E(q)$ holds. Then, for every $s\geq0$,
\begin{equation}\label{eq:analysis-tracking-bound}
\|\widehat u_s-u_s\|_2
\leq\frac{\varepsilon_0}{\sqrt N}\min\{s\eta,8\}.
\end{equation}
Consequently, for every $s\geq1$ and every input $x$,
\begin{equation}\label{eq:analysis-empirical-margin}
|\langle\widehat u_s-e_r,q(x)\rangle|<1.
\end{equation}
Moreover, for every $T\geq0$,
\begin{equation}\label{eq:analysis-regression-convergence}
\|\widehat u_T-e_r\|_2\leq(1-\eta/8)^T.
\end{equation}
\end{lemma}

\begin{proof}
On $\mathcal E(q)$, the minimum eigenvalue of
$\widehat\Sigma_q|_U$ is at least
\[
\frac14-\frac{\varepsilon_0}{\sqrt N}\geq\frac18.
\]
Its maximum eigenvalue is at most $N$ by
\Cref{lem:analysis-feature-moments}. Hence
\[
\|(I-\eta\widehat\Sigma_q)|_U\|_{\mathrm{op}}
\leq1-\eta/8.
\]

To compare the two sequences, set
$\Delta_s:=\widehat u_s-u_s$.
Subtracting their recursions gives
\[
\Delta_{s+1}
=(I-\eta\widehat\Sigma_q)\Delta_s
-\eta(\widehat\Sigma_q-\Sigma_q)(u_s-e_r).
\]
All vectors in this recursion belong to $U$.
Using $\|u_s-e_r\|_2\leq1$ and $\Delta_0=0$, we obtain
\[
\|\Delta_s\|_2
\leq\frac{\varepsilon_0}{\sqrt N}\,
\eta\sum_{j=0}^{s-1}(1-\eta/8)^j
\leq\frac{\varepsilon_0}{\sqrt N}\min\{s\eta,8\}.
\]
This proves \eqref{eq:analysis-tracking-bound}.

Since $\|q(x)\|_2\leq\sqrt N$, the population prediction bound gives
\[
\begin{aligned}
|\langle\widehat u_s-e_r,q(x)\rangle|
&\leq
|\langle u_s-e_r,q(x)\rangle|
+\|\Delta_s\|_2\|q(x)\|_2\\
&\leq e^{-s\eta/4}+\varepsilon_0\min\{s\eta,8\}.
\end{aligned}
\]
This is strictly less than one for every $s\geq1$.
Indeed, writing $t=s\eta$, if $0<t\leq1$, then
\[
1-e^{-t/4}\geq t/8>\varepsilon_0t.
\]
If $t\geq1$, then
\[
1-e^{-t/4}\geq1-e^{-1/4}>1/8>8\varepsilon_0.
\]
Thus \eqref{eq:analysis-empirical-margin} holds for every input.

Finally, the empirical error itself satisfies
\[
\widehat u_{s+1}-e_r
=(I-\eta\widehat\Sigma_q)(\widehat u_s-e_r).
\]
Since $\widehat u_0-e_r=-e_r\in U$ has norm one, contraction on $U$
gives \eqref{eq:analysis-regression-convergence}.
\end{proof}

\paragraph{Exact recovery of one row.}
We now identify actual gradient descent on the parity loss with the auxiliary
empirical regression trajectory. The next lemma applies to any integer-valued
initial row, provided its target is obtained by adding the parity contribution
of one uniquely occurring feature. 

\begin{lemma}[One-row recovery]\label{lem:analysis-one-row}
Fix $(q,r)$ as in the preceding regression lemmas, and let $N=\dim(q)$.
Fix $w_0\in\mathbb Z^N$, and suppose the target is
\[
y(x)=\sigma\!\left(\langle w_0,q(x)\rangle-x_r\right).
\]
Run gradient descent from $w_0$ on the fixed-batch loss
\[
 \widehat\ell(w)=\frac1{2B}\sum_{b=1}^B
 \bigl(\sigma(\langle w,q(x^{(b)})\rangle)-y(x^{(b)})\bigr)^2
\]
using $0<\eta\leq1/N$ and the specified breakpoint convention.
On $\mathcal E(q)$, its iterates satisfy
\begin{equation}\label{eq:analysis-exact-row-dynamics}
 w_s=w_0-\widehat u_s\qquad(s\geq0),
\end{equation}
where $\widehat u_s$ is defined by
\eqref{eq:analysis-empirical-regression}. In particular, if
$(1-\eta/8)^T<1/2$, then
\begin{equation}\label{eq:analysis-one-row-rounding}
 \mathsf{Round}(w_T)=w_0-e_r.
\end{equation}
\end{lemma}
\begin{proof}
For a batch example, set
$m_b=\langle w_0,q(x^{(b)})\rangle\in\mathbb Z$ and
$p_b=q_r(x^{(b)})\in\{0,1\}$.
Suppose $w_s=w_0-\widehat u_s$. Its scalar prediction on this example is
$\sigma(m_b-\langle\widehat u_s,q(x^{(b)})\rangle)$, and its target is
$\sigma(m_b-p_b)$.

At $s=0$, \Cref{lem:analysis-quadratic-loss} gives the regression gradient
under the stipulated breakpoint convention. At every $s\geq1$,
\Cref{lem:analysis-empirical-margin} gives
\[
 \left|\langle\widehat u_s,q(x^{(b)})\rangle-p_b\right|<1.
\]
Thus \Cref{lem:analysis-quadratic-loss} again identifies the selected gradient
with the quadratic regression gradient. Applying the chain rule for
$w=w_0-u$ yields
\[
 \nabla_w\widehat\ell(w_s)
 =-\frac1B\sum_{b=1}^B
   \bigl(\langle\widehat u_s,q(x^{(b)})\rangle-p_b\bigr)q(x^{(b)})
 =-\widehat\Sigma_q(\widehat u_s-e_r).
\]
Consequently,
\[
 w_{s+1}=w_s+\eta\widehat\Sigma_q(\widehat u_s-e_r)
 =w_0-\widehat u_{s+1}.
\]
Starting from $w_0=w_0-\widehat u_0$, induction proves
\eqref{eq:analysis-exact-row-dynamics}.

By \Cref{lem:analysis-empirical-margin},
\[
 \|w_T-(w_0-e_r)\|_\infty
 =\|\widehat u_T-e_r\|_\infty
 \leq\|\widehat u_T-e_r\|_2
 \leq(1-\eta/8)^T.
\]
If this is less than $1/2$, nearest-integer rounding recovers the integer
vector $w_0-e_r$ exactly.
\end{proof}

\paragraph{Proof of the main theorem.}
\begin{proof}[Proof of \Cref{thm:provable-parity-main}]
We may take $c=1/4$, so $\eta=1/(4d)$, and choose a sufficiently large
universal constant $C$. Since every available feature vector has dimension
$N\leq2d$, this step size satisfies $\eta\leq1/N$. For $C\geq64$ and
$T=\lceil Cd\rceil$, we also have
\begin{equation}\label{eq:analysis-rounding-accuracy}
 (1-\eta/8)^T\leq\exp\!\left(-\frac{T}{32d}\right)
 \leq e^{-2}<\frac12.
\end{equation}

First we construct one event that covers every row-learning task.
For Phase~1, require $\mathcal E(q_l)$ on its batch for all $l\in[k]$.
For each Phase~$t\geq2$, require $\mathcal E(q_1)$ on that phase's batch.
There are $2k-1$ row--batch pairs. Denote the intersection of their events by
$\mathcal E_{\mathrm{all}}$. By
\Cref{lem:analysis-feature-moments}, $N\leq2d$, and
$\varepsilon_0=1/128$, a union bound gives
\begin{equation}\label{eq:analysis-joint-probability}
 \Pr[\mathcal E_{\mathrm{all}}^c]
 \leq16kd^2\exp\!\left(-\frac{B}{65536d^3}\right).
\end{equation}
No independence between the events for different rows on the same batch is
needed. Under the theorem's batch-size assumption, any fixed $C\geq2^{17}$
makes the right-hand side at most
\[
 16kd^2\left(\frac{\delta}{kT}\right)^2
 \leq\frac{16\delta^2}{kC^2}\leq\delta.
\]
Hence $\Pr[\mathcal E_{\mathrm{all}}]\geq1-\delta$.
All the following statements are deterministic on this event.

\emph{Phase-1 learning and autoregressive CoT computation.}
The row predicting $z_1$ is row $d$. Its feature vector is $q_1=x$, its
initial weights are zero, and $z_1=\sigma(-x_{i_1})$.
By \Cref{lem:analysis-one-row} and
\eqref{eq:analysis-rounding-accuracy}, rounding at the end of Phase~1 gives
\[
 W^{(1)}_{d,1:d}=-e_{i_1}.
\]
For $l\geq2$, consider row $d+l-1$. Its feature vector is $q_l$, and its
initialized coefficient on the last feature $z_{l-1}$ is $-1$, with all other
coefficients zero. Its initial logit is therefore $-z_{l-1}(x)$.
Since
\[
 z_l(x)=\sigma(-z_{l-1}(x)-x_{i_l}),
\]
this is another instance of \Cref{lem:analysis-one-row}, with target feature
$x_{i_l}$. This feature occurs once: the only possible duplicate in $q_l$
is $z_1=x_{i_1}$, and $i_l\neq i_1$.
Consequently the rounded weights satisfy
\begin{equation}\label{eq:analysis-phase-one-rows}
 W^{(1)}_{d+l-1,1:d}=-e_{i_l},\qquad
 W^{(1)}_{d+l-1,d+l-1}=-1,\qquad 2\leq l\leq k,
\end{equation}
with every other coefficient in these rows equal to zero. All causal
coefficients were allowed to vary during training; the initialized CoT
coefficient is recovered by rounding. Masked entries and unsupervised rows
remain at their initialized integer values.

Define the autoregressive rollout recursively by
\[
 \widehat z_l(x)=
 f_{W^{(1)}}(x,\widehat z_1(x),\ldots,\widehat z_{l-1}(x))_{d+l-1},
 \qquad l\in[k],
\]
with an empty generated prefix for $l=1$.
The recovered first row gives $\widehat z_1(x)=\sigma(-x_{i_1})=z_1(x)$.
If $\widehat z_{l-1}(x)=z_{l-1}(x)$, then
\eqref{eq:analysis-phase-one-rows} gives
\[
 \widehat z_l(x)=\sigma(-x_{i_l}-z_{l-1}(x))=z_l(x).
\]
Induction proves exact generation of every CoT token for every Boolean input,
as asserted in the first claim of the theorem.

\emph{Internalization across phases.}
We show that the first prediction row satisfies
\begin{equation}\label{eq:analysis-prefix-invariant}
 W^{(t)}_{d,1:d}=-\sum_{h=1}^t e_{i_h},\qquad t\in[k].
\end{equation}
The Phase-1 analysis establishes the base case.
Suppose \eqref{eq:analysis-prefix-invariant} holds at the end of Phase~$t-1$.
In Phase~$t$, row $d$ still has feature vector $q_1=x$, initial weights
$w_0=-\sum_{h<t}e_{i_h}$, and target
\[
 z_t(x)=\sigma\!\left(\langle w_0,x\rangle-x_{i_t}\right).
\]
All input features occur once. Applying \Cref{lem:analysis-one-row} on the
Phase-$t$ event and using \eqref{eq:analysis-rounding-accuracy} yields
\[
 W^{(t)}_{d,1:d}=w_0-e_{i_t}=-\sum_{h=1}^t e_{i_h}.
\]
This completes the induction. Although each gradient step updates the full
matrix, the loss in these phases depends only on row $d$, so all other rows
have zero gradient and remain unchanged.

At the end of Phase~$k$, for every $x\in\{0,1\}^d$,
\[
 f_{W^{(k)}}(x)_d
 =\sigma\!\left(-\sum_{h=1}^k x_{i_h}\right)
 =\bigoplus_{h=1}^k x_{i_h}=z_k(x).
\]
The Phase-1 rollout also gives
$f_{W^{(1)}}^{\mathrm{e2e}-k}(x)=\widehat z_k(x)=z_k(x)$.
Thus, under the internalization map in the theorem,
\[
 f_{W^{(k)}}(x)_d
 =P\!\left(f_{W^{(1)}}\right)(x)
 =f_{W^{(1)}}^{\mathrm{e2e}-k}(x)
 \qquad\text{for every }x\in\{0,1\}^d.
\]
Both claims hold on the same event $\mathcal E_{\mathrm{all}}$, whose
probability is at least $1-\delta$. Finally, taking $B$ at the displayed
polynomial threshold, $k\leq d$ and $T=O(d)$ give runtime polynomial in
$d$ and $\log(1/\delta)$.
\end{proof}

\subsection{Experiments for Learning Parity with Simplified Transformer}
We also experimentally see the learning dynamics by training the same transformer architecture as in \Cref{alg:internalize-parity-shortening} on the parity task, but without the rounding step. We set \(d=20\) and \(k=9\), use batch size \(B=1024\), step size \(\eta=10^{-3}\). Rather than switching phases after a fixed iteration budget \(T\), we advance to the next phase once the training loss falls below a prescribed threshold. The learned attention matrices are reported in \Cref{fig:attention-parity}. Consistent with the theoretical mechanism, the attention weights progressively concentrate on the relevant coordinates, and by the end of the curriculum the model selects the desired coordinates from the hidden parity support.
\begin{figure}
    \centering
\includegraphics[width=1.0\linewidth]{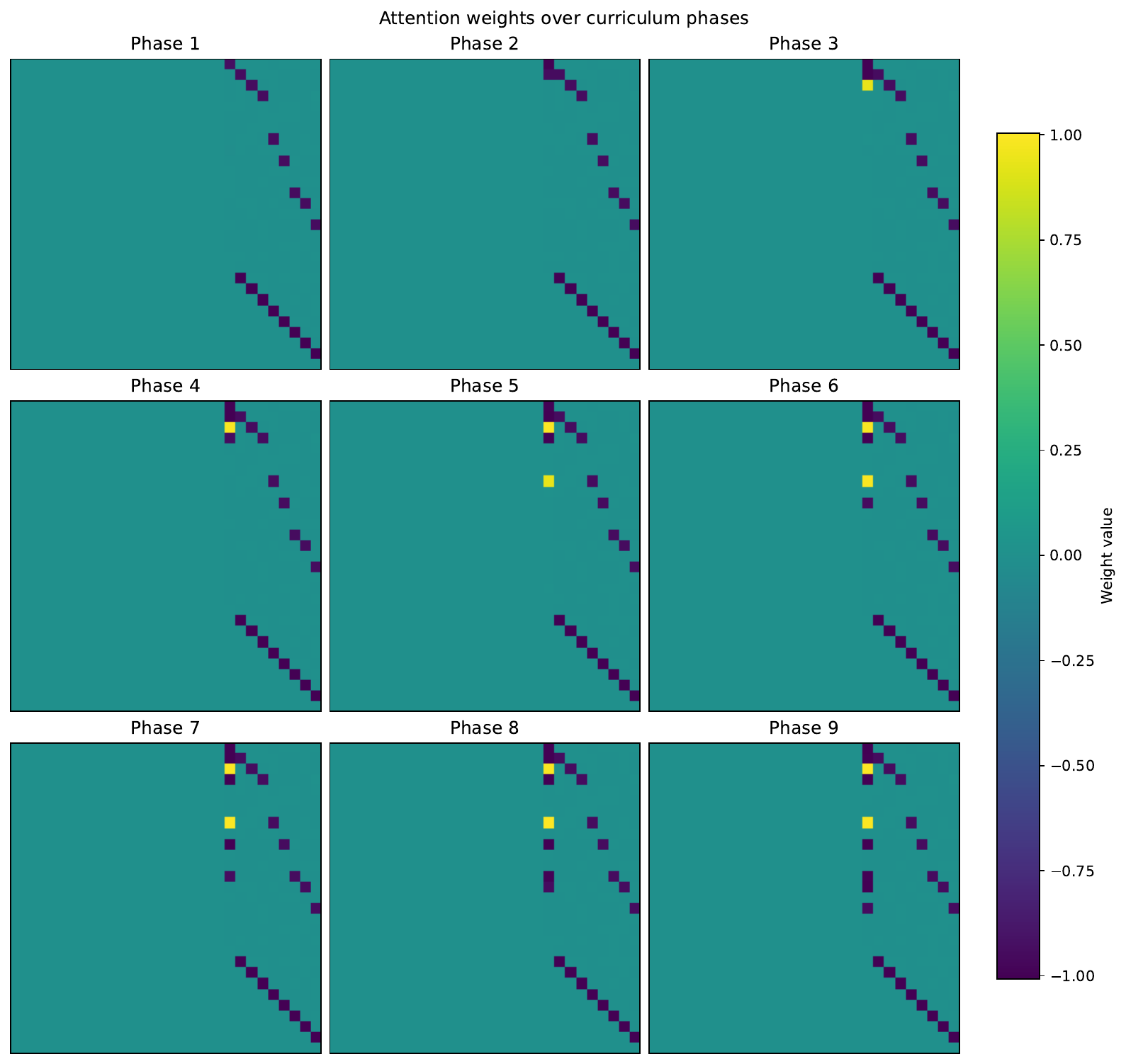}
    \caption{Evolution of the learned attention weights during curriculum training for the parity task. Each panel shows the attention-weight matrix after one internalization phase, where the model is trained after progressively removing one CoT token. The visualization suggests that the same architecture gradually reorganizes its attention weights to recover the hidden parity support as the explicit CoT is removed.}
    \label{fig:attention-parity}
\end{figure}

\newpage

\section{Internalizing semiautomata: experimental details and additional figures}\label{sec:app:experimental-details-semiautomata}

In this section, we provide more experimental details and results on internalizing the simulation of semiautomata.
We describe our learning setting in more detail.

\paragraph{Setup} A semiautomaton $\mathcal{A}$ is a tuple $\mathcal{A} = \left(Q, \Sigma, \delta\right)$ where $Q$ is a set of \textit{states}, $\Sigma$ is an input alphabet and $\delta: Q \times \Sigma \to Q$ denotes a \textit{transition function}. We are interested in predicting the state of the semiautomaton after $T$ transition steps, starting from a given initial state $q_0 \in Q$. Equivalently, we are interested in learning the \textit{extended} transition function $\delta^{(t)}: Q \times \Sigma^T \to Q$. 
We consider a uniform distribution over input symbols $\sigma_1, \ldots, \sigma_T \sim_{iid} \mathrm{Unif} \left( \Sigma \right)$. Each input sequence is concatenated with a fixed initial state $q_0$, that is $z =(\sigma_1, \ldots, \sigma_T, q_0) \in \Sigma^{T} \times Q$. The learning goal is to find a predictor $h: \Sigma^{T} \times Q \to Q$ that minimizes the probability of predicting the wrong final state:
\begin{equation}
    L_{\mathcal{A}}(h) = \mathbb{P}_{\sigma_1, \ldots, \sigma_T} \left[ h\left( \sigma_1, \ldots, \sigma_T, q_0 \right) \neq \delta^{(T)}\left( q_0,  \sigma_1 \ldots \sigma_T \right) \right].
\end{equation}

\paragraph{Autoregressive training} We consider various training procedures that differ in the type and structure of the training data:

\begin{itemize}
    \item[-] \textit{Chain-of-thought} (CoT) training: We consider sequences $z = (\sigma_1, \ldots, \sigma_T, q_0, q_1, \ldots, q_T)$ that contain all the intermediate states between input and final state, where $q_t = \delta(q_{t-1}, \sigma_t) = \delta^{(t)} \left(q_0, \sigma_1\ldots\sigma_t \right) $ for $t \in [T]$.
    \item[-] \textit{End-to-end} (E2E) training: Only the final state $q_T \in Q$ is appended to the input sequence. 
    \item[-] \textit{Internalization via left curriculum}: We consider $T$ stages of a curriculum, where at each stage $t$ we use data of the form $z^{(t)} = (\sigma_1, \ldots, \sigma_T, q_0, q_t, \ldots, q_T)$. That is, we omit the first $t-1$ states from the sequence.
    \item[-] \textit{Internalization via right curriculum}: Similarly, we consider $T$ stages of a curriculum, where at each stage $t$ we use data of the form $z^{(t)} = (\sigma_1, \ldots, \sigma_T, q_0, q_1, \ldots, q_{T-t}, q_T)$. In other words, we omit the last $t-1$ states from the sequence (starting from state $q_{T-1}$).
    \item[-] \textit{Internalization via inductive curriculum}: We consider $T$ stages of a curriculum, where at each stage $t$ we use data with a chain of thought that contains leaps of length $t$, i.e., $z^{(t)} = (\sigma_1, \ldots, \sigma_T, q_0, q_t, q_{2t}, \ldots, q_{kt})$ where $k = \arg \max \{n \in \mathbb{N}: n t \leq T \}$. Note that the last token might or might not be the final state $q_T$ in some of the stages.
\end{itemize}

\begin{figure}
    \centering
    \includegraphics[scale=0.4]{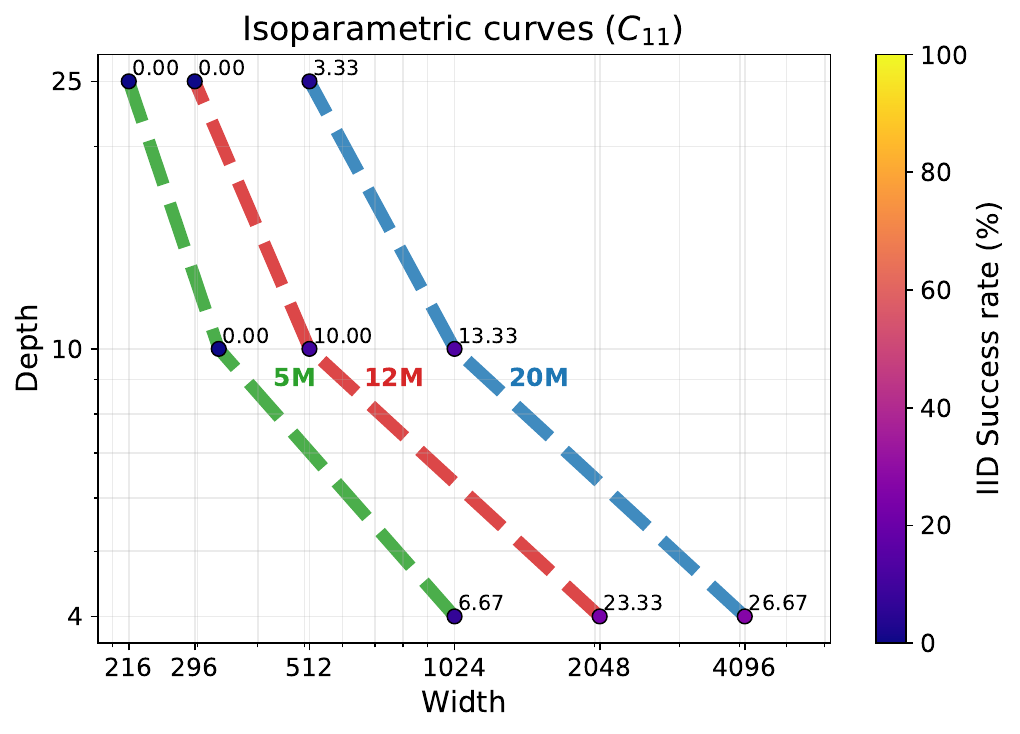}
    \includegraphics[scale=0.4]{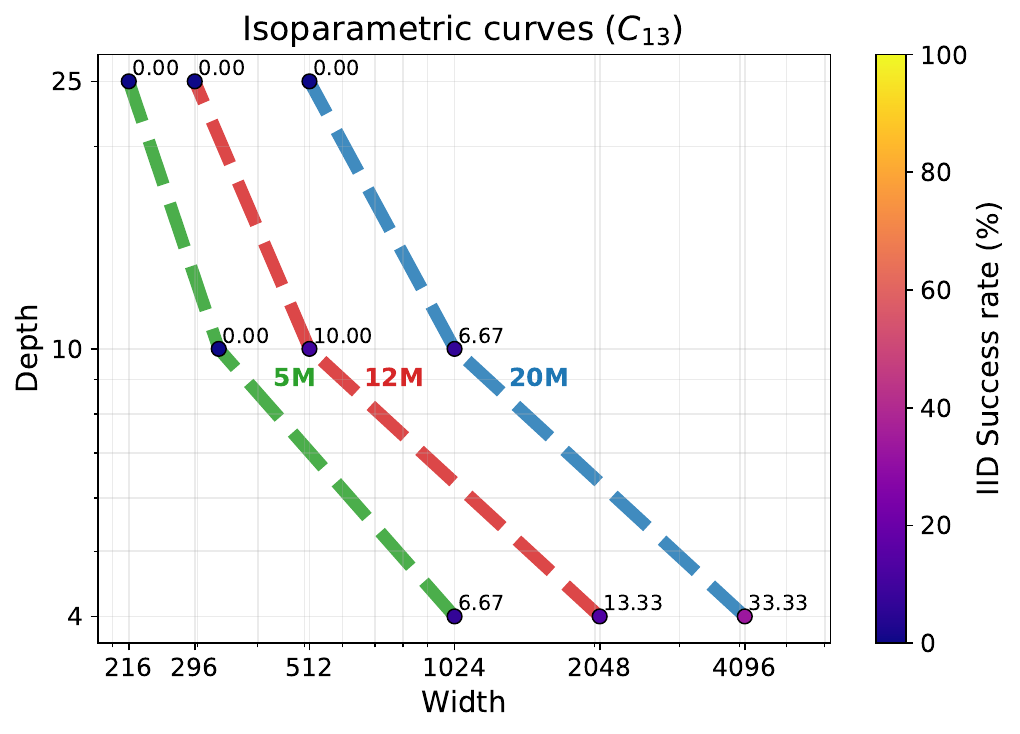}
    \caption{\textbf{Internalization success rates across isoparametric curves for $C_{11}$ and $C_{13}$.} We observe that increasing width, rather than depth, improves the chances of internalization.}
    \label{fig:isoparams_more}
\end{figure}

We train with a next-token prediction $\log$ loss applied only at the state positions of each sequence (i.e., the predictable part of the sequence). 
We start all curricula at stage 1 (when the data contain the full history of states $q_1, \ldots, q_T$) and increase the stage after a fixed number of iterations. 

For counter automata, the alphabet $\Sigma$ and state space $Q$ are the same, that is $\Sigma = Q = \mathbb{Z}_n$. We encode input symbols and state symbols using different symbols in our transformer implementation. For $S_3$, the input symbols are 4 permutations: identity, transposition $1 \leftrightarrow 2$, transposition $2 \leftrightarrow 1$ and a shift to the left by 1. We encoded each state using their Lehmer rank. We always append an EOS token at the end of the sequences.
For OOD evaluation on $C_n, \, n > 2$, each input token is sampled independently as $\sigma_t \sim \mathrm{Binomial}(n-1,p)$, where $n$ is the size of the input space and $p$ is a parameter. For the parity semiautomaton $C_2$, OOD is measured by simply changing the bias of sampling `1' vs `0' from 0.5 to $p$. Unless stated otherwise, we always use $p=0.8$ which induces a distribution concentrated on larger symbols and larger total sequence sums.

\paragraph{Experimental details}

We run experiments using Huggingface's transformers library. The architecture is a GPT2-type architecture with absolute positional embeddings, normalization layers, residual connections and uses the GeLU activation function. We use the AdamW optimizer with constant learning rate $3\cdot 10^{-4}$ and parameters $\left(\beta_1, \beta_2\right) = (0.9, 0.95)$, where the reduced momentum coefficients enable faster adaptation to the changing loss between curriculum stages. We reset the optimizer's state in the beginning of each curriculum stage.

Each run for a counter semiautomaton completes in $50,000$ training iterations of online training with 128 freshly sampled sequences per iteration. All curriculum stages consist of an equal amount of training iterations equal to $50, 000 / T$, where $T$ is the automaton simulation time. The experiments on $S_3$ run for $200,000$ training iterations. We deem an internalization run as \textit{successful} if in-distribution test accuracy is above $95\%$ at the end of training. This is how we compute successs rates across the section.

The main set of results of Figure~\ref{fig:cot_e2e_internalization_comparison} consists of 30 random seeds combination of semiautomaton and simulation length. The depth of the network is set to 4, the embedding dimension is set to $512$, the number of heads is set to $128$ and the MLP width at each layer is $2048$.

Figure~\ref{fig:isoparams_and_cprime} (left), which compares success rates across different semiautomata for $T=25$, presents simulations with networks of depth 4, embedding dimension 512, number of heads 128, and MLP width 256 and 2048.

Figure~\ref{fig:isoparams_and_cprime} (right) and Figure~\ref{fig:isoparams_more} show success rates across varying width and depth configurations for \(C_7\) and \(C_{11}, C_{13}\), respectively. In these figures, width is defined as the maximum of the embedding dimension and MLP width. We evaluate three parameter-matched architecture families:
\[
\text{20M}: [(4,512,4096), (10,512,1024), (25,344,512)],
\]
\[
\text{12M}: [(4,376,1024), (10,332,128), (25,216,64)],
\]
\[
\text{5M}: [(4,512,2048), (10,448,512), (25,296,256)],
\]
where each tuple denotes \((\text{depth}, \text{embedding dimension}, \text{MLP width})\). In all cases, we use per-head-dimension $4$ which determines the number of heads for each configuration.

Figure~\ref{fig:scatter_C2_C3_C5} presents scatter plots of in-distribution (ID) and out-of-distribution (OOD) test accuracies at the end of internalization across multiple random seeds. For \(C_2\), we use an embedding dimension of 64 with 32 attention heads under the inductive curriculum. For \(C_3\), the depth-1 configuration uses embedding dimension 2048, MLP width 8192, and 512 attention heads, also trained with the inductive curriculum. The depth-4 and depth-10 configurations for both \(C_3\) and \(C_5\) use embedding dimension 512 and 128 attention heads, and are trained using the left curriculum. We report depth-1 results with the inductive curriculum because it consistently achieved stronger internalization performance than the left curriculum in shallow architectures.

\paragraph{More results}

We found the number of attention heads to be highly important for successful internalization. In particular, for a fixed embedding dimension shared across heads, increasing the number of heads consistently improved the internalization success rate. Figure~\ref{fig:more_heads_help} demonstrates this effect for \(C_2\) across two embedding dimensions and multiple depth configurations.

\begin{figure}
    \centering
    \includegraphics[scale=0.45]{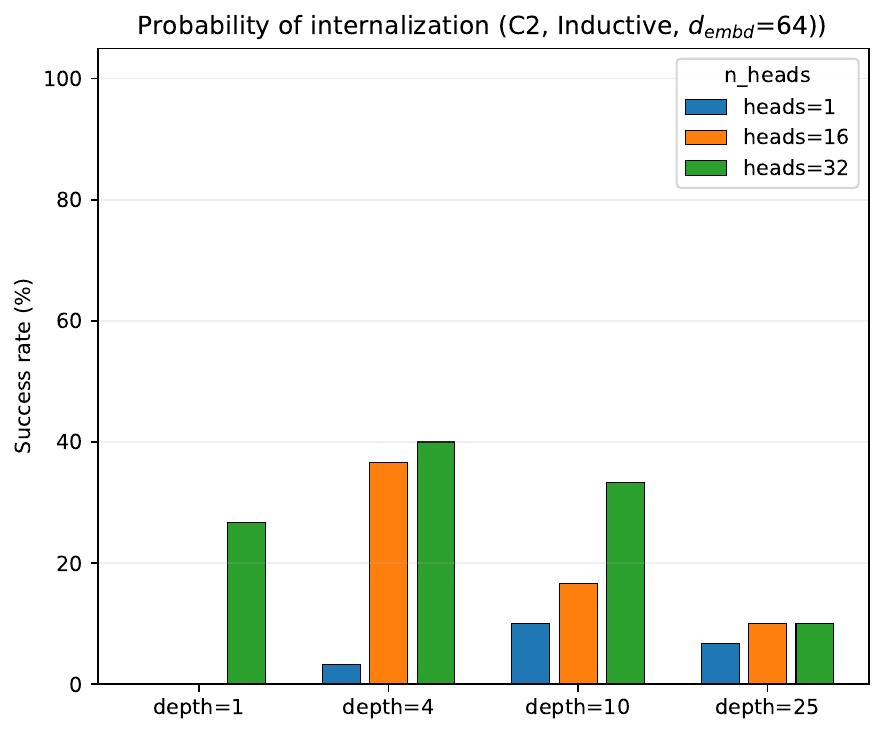}
    \includegraphics[scale=0.45]{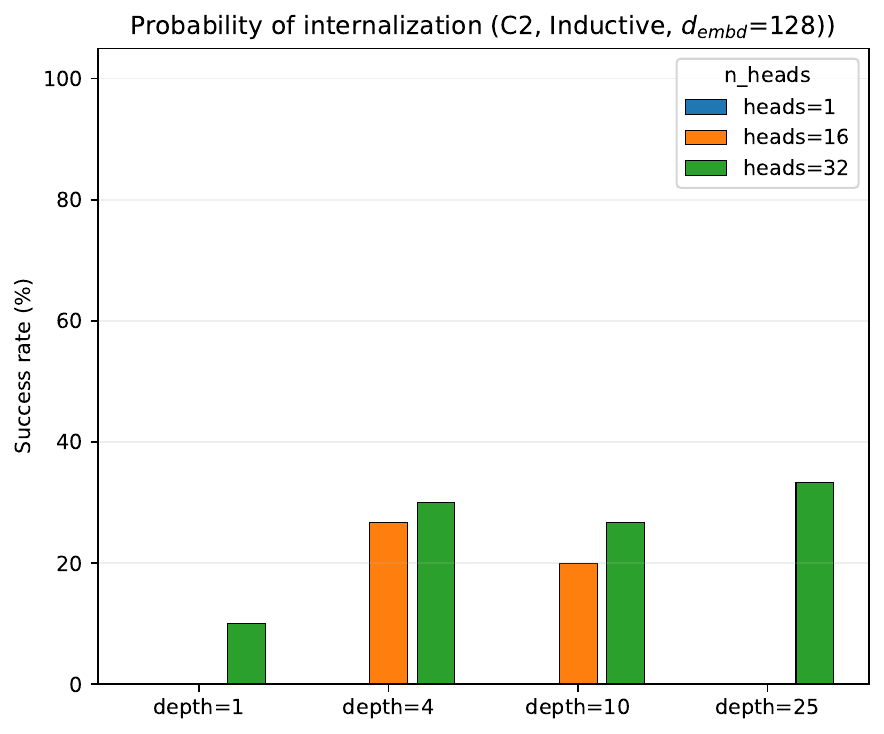}
    \caption{\textbf{Increasing the number of attention heads improves the success rate of internalization.} Results are shown for \(C_2\) under the inductive curriculum, with embedding dimensions 64 (left) and 128 (right), across multiple depth configurations. The MLP width is set to four times the embedding dimension.}
    \label{fig:more_heads_help}
\end{figure}

We found that a curriculum method that starts by removing tokens from the right is strictly worse at internalizing chain-of-thought reasoning than either a left-to-right or an inductive method. For example, in Figure~\ref{fig:left-vs-right}, we present success rates across three different semiautomata and observe that, while a left curriculum succeeded in internalizing approximately one-third of the time, the right curriculum failed in every case. Similar observations were made by \citet{DCS24}.

\begin{figure}
    \centering
    \includegraphics[scale=0.425]{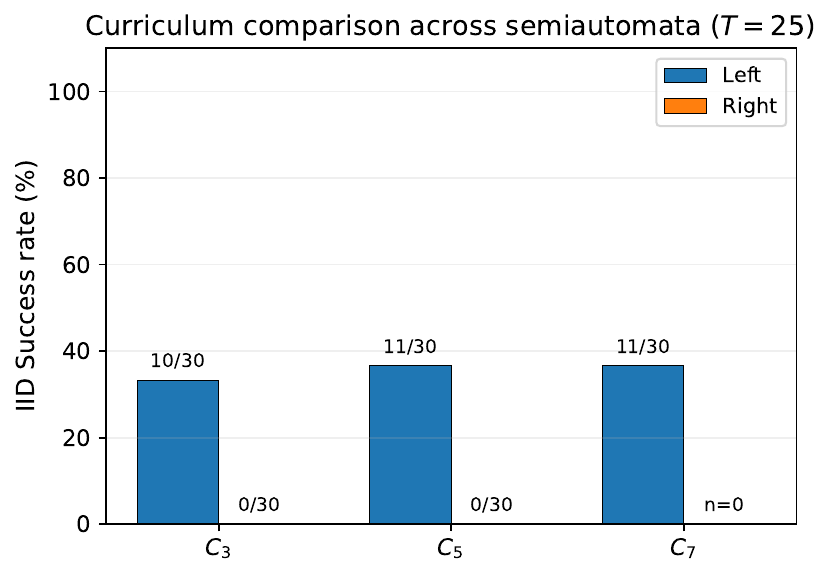}
    \caption{\textbf{Comparison across left and right curricula for internalization of $C_3, C_5$ and $C_7$.} This is for transformers of depth 4, embedding dimension 512, number of heads set to 128 and MLP width equal to 2048.}
    \label{fig:left-vs-right}
\end{figure}

We provide an intuitive explanation for this phenomenon. We believe that this shortcoming is largely attributable to the use of absolute positional embeddings in the architecture. As chain-of-thought tokens are internalized, the MLP -- which is shared across positions -- must play a dual role: it must implement both the simple one-step operation (computing $\bmod n$ over two numbers) and the increasingly complex aggregated operation of addition modulo $n$ over $t$ numbers, where $t$ denotes the current stage of the curriculum. Since our transformer uses absolute positional embeddings, it must use positional information to gate the MLP into one of these two operations. If we internalize from the left, then the position at which the complex operation occurs remains fixed at $T+1$. In contrast, if we internalize from the right, then this position changes at every stage, creating a difficult moving target for the MLP. We believe that other architectures, such as mixture-of-experts models, may not exhibit such a difference between the two methods.

\subsection{Learning Semiautomata with Mixture: Experimental Details}\label{app:mixture}

We also consider learning semiautomata from a single mixture distribution over CoT lengths, rather than from an ordered curriculum. Let
\[
    \mathcal D_1,\ldots,\mathcal D_T
\]
denote the stage distributions corresponding to one of the CoT-removal procedures defined above. The mixture distribution is $\mathcal D_{\mathrm{mix}}
    =
    \frac{1}{T}\sum_{t=1}^T \mathcal D_t$ .
In our implementation, at each training step we sample a stage \(t\sim\mathrm{Unif}([T])\) and draw the batch from the corresponding stage distribution \(\mathcal D_t\). Thus, marginally over training, the model is trained on a single distribution supported on examples with different CoT lengths. We use the same number of samples in total for the curriculum and mixture runs so that the comparisons are fair.  

The main difference from the curriculum experiments is having additional token for the answer length hints. In the original semiautomata experiments, an example consists of the input symbols, followed by the fixed initial state \(q_0\), followed by the retained state tokens and \(\texttt{EOS}\). In the mixture experiments, we insert one additional token after \(q_0\) and before the answer block:
\[
    (\sigma_1,\ldots,\sigma_T,q_0,\texttt{LEN}_L,
    q_{j_1},\ldots,q_{j_K},\texttt{EOS}).
\]
Here $1\le j_1<\cdots<j_K=T$,
so the final state \(q_T\) is always retained, and $L = K+1$
is the length of the answer block including the trailing \(\texttt{EOS}\). We add \(T\) new vocabulary tokens
\[
    \texttt{LEN}_2,\ldots,\texttt{LEN}_{T+1},
\]
disjoint from the input tokens, state tokens, and \(\texttt{EOS}\). The direct-answer case has \(K=1\), hence \(L=2\), corresponding to the output
\[
    (q_T,\texttt{EOS}).
\]

 The token \(\texttt{LEN}_L\) tells the model which output format is requested for the current example. The hint is part of the prompt and is not included in the loss; as above, the next-token loss is applied only to the predictable answer/state positions.

At evaluation time, we can either use the natural length hint associated with a mixture component or force a desired answer length. In particular, forcing $\texttt{LEN}_2$ asks the model to answer directly, while forcing larger values of \(L\) asks it to generate a longer CoT before the final answer. We report final-state accuracy, where the prediction is the last generated state token before \(\texttt{EOS}\). This lets us evaluate the same trained model both in a direct-answer regime and in longer-CoT regimes.

\section{Additional Details for Learning Through Internalization as Positive Distribution Shift}\label{app:pds-view}

\subsection{General Positive Distribution Shift}

Let $\mathcal X$ be the input space, and $(\mathcal D, f^*)$ be the target distribution, where $f^*:\mathcal X\to \Delta(\mathcal Y)$ is the target function. We aim to find a predictor $h$ that has low error $L_{(\mathcal D,f^*)}$ specified by loss $\ell:\mathcal Y \times \mathcal Y \to \mathbb R_{\ge 0}$  (we extend it to a randomized prediction by $\bar{l}(p,q) = \mathbb E_{\hat{a}\sim p,a\sim q} l(\hat{a},a), ~p,q\in \Delta(\mathcal Y)$). We are allowed to train on a (potentially larger) context space $\mathcal C$, which could include added hints, padding, prompts, or other information, and for which there is an extraction map $\pi: \mathcal C \to \mathcal X$ which maps the context $c$ back to an input $x = \pi(c)$. The label depends only on the input $\pi(c)$. We train on a distribution $\mathcal D_{\mathcal C}' \in \Delta(\mathcal C)$ labeled by $f^* \circ \pi$. For a class of models $\mathcal M$, we define an efficient transport map $Tr_{b}: \mathcal M \to \Delta(\mathcal Y)^{\mathcal X}$ for an inference budget $b$. For a hypothesis $h\in \mathcal H\subseteq \Delta(\mathcal Y)^{\mathcal X}$, we define $L_{(\mathcal D,f^*)}(h) = \mathbb E_{x\sim \mathcal D} \bar{l}(h(x),f^*(x))$ and $L_{(\mathcal D,f^*)}(\mathcal H) = \inf_{h\in \mathcal H} L_{(\mathcal D,f^*)} (h)$. We evaluate the loss of a model $M\in \mathcal M$ on the target function accuracy as $L_{(\mathcal D,f^*)}^b(M) = L_{(\mathcal D,f^*)}(Tr_b(M))= E_{x\sim \mathcal D} \bar{\ell}\left(Tr_{b}(M)(x),f^*(x) \right) $.

\begin{definition}[Generalized PDS learnability]
A learning rule $A$ generalized-PDS-learns $(\mathcal D,f^*)$ and a hypothesis class \(\mathcal H\subseteq\Delta(\mathcal Y)^{\mathcal X}\) with inference budget \(b\), training context $\mathcal C$, training distribution $\mathcal D_{\mathcal C}'$, extraction map $\pi$, model class $\mathcal M$, and efficient transport map $Tr_{b}$ with sample complexity \(m(\varepsilon)\) and runtime
\(T(\varepsilon)\) if for every $\varepsilon>0$ and for
\(
S=\{(c_i,y_i)\}_{i=1}^{m(\varepsilon)}
\sim
(\mathcal D'_{\mathcal C},f^\star\circ\pi)^{m(\varepsilon)},
\)
the learner outputs
\(
\hat M=A(S)\in\mathcal M
\)
in time at most \(T(\varepsilon)\), and
\[
\mathbb E_{S,A}
L_{(\mathcal D,f^\star)}
\bigl(
\operatorname{Tr}_b(\hat M)
\bigr)
\le
L_{(\mathcal D,f^\star)}(\mathcal H)+\varepsilon.
\]
\end{definition}

Depending on what $\mathcal C$ and $\mathcal D_{\mathcal C}$ (and the elements in its support) are allowed to depend on, we can define generalized-f-PDS and generalized-D(R)S-PAC learnabilities. In particular, for f-PDS, we allow the training distribution $\mathcal D_{\mathcal C}'$ to depend on the target $f^*$.

\paragraph{Original PDS.}
The original PDS setting from \cite{Med+26} is recovered by taking the training and
evaluation spaces to be the same. Let
\(
f^*:\mathcal X\to\Delta(\mathcal Y)
\)
be the target function, let \(\mathcal D\in\Delta(\mathcal X)\) be
the target input distribution, and let
\(
\mathcal H\subseteq \mathcal Y^{\mathcal X}
\)
be the hypothesis class, which is deterministic in this case. Set
\(
\mathcal C=\mathcal X,
\pi=\mathrm{id}_{\mathcal X}.
\)
Let \(\mathcal M=\mathcal H\). We only need to consider $b=0$ and we define
\(
\operatorname{Tr}_b(h)=h.
\)
We take 
\(
\mathcal D'_{\mathcal C}=\mathcal D'
\)
and use zero-one loss,
\(
\ell(\hat y,y)=\mathbf 1\{\hat y\neq y\}.
\)
We have
\[
L_{(\mathcal D,f^\star)}(\operatorname{Tr}_b(h))
=
\mathbb E_{x\sim \mathcal D}
\mathbb E_{y\sim f(\cdot\mid x)}
\mathbf 1\{h(x)\neq y\}
=
L_{\mathcal D,f}(h).
\]
Thus the generalized PDS guarantee becomes
\[
\mathbb E_{S\sim(D',f)^m}
L_{D,f}(A(S))
\le
L_{D,f}(\mathcal H)+\varepsilon,
\]
which is the original PDS guarantee.

\paragraph{Mixture over hinted parities in MLPs.}
For the mixture experiments with MLPs in \Cref{subsec:internalization-MLP}, we need to take $\mathcal C$ to be the expanded input space with the hints, and $Tr_{b}$ to be the operation that extracts the subnetwork from the bigger network. 

We instantiate $\mathcal X= \{0,1\}^d$, $\mathcal D=\text{Unif}(\{0,1\}^d)$ and $f^*(\cdot \mid x)$ is $\chi_{S}(x)$ with probability $1-\eta$ and $1-\chi_{S}(x)$ with probability $\eta$, where $S$ is the hidden parity support of size $k$. The loss is zero one loss $\ell(\hat y,y)=\mathbf 1\{\hat y\neq y\}$. We then take $\mathcal C=\{0,1\}^{d+k-1}$ and $\pi(c) = (c_1,\dots, c_d)$. The hints are given by $z_t(x) = \oplus_{j=1}^{t} x_{i_j}$, for $t=1,\dots,k-1$ and for $r=0,\dots,k-1$ let $c_r(x) = (x_1,\dots,x_d,z_1(x),\dots, z_{k-1-r}(x), 0,\dots,0) \in \mathcal C$. The hypothesis $\mathcal H_k$ here is the $k$-sparse parities. Let $\mathcal D_r$ be the distribution of $c_r(x)$ for $x\sim \mathcal D$ and take the training distribution to be $\mathcal D_{\mathcal C}' = \frac{1}{k} \sum_{r=0}^{k-1} \mathcal D_r$. The model class $\mathcal M$ is the class of expanded input MLPs $g_{\theta}(\tilde x) = F_\theta(\tilde x)
=
W_L\sigma\!\left(
W_{L-1}\sigma\!\left(
\cdots
\sigma(W_1\tilde x+b_1)
\cdots
\right)+b_{L-1}
\right)+b_L$. The transport map takes the subetwork $Tr_{0}(g_{\theta})(x) = \left(
W_L\sigma\!\left(
W_{L-1}\sigma\!\left(
\cdots
\sigma(W_1[:,1\!:\!d]\,x+b_1)
\cdots
\right)+b_{L-1}
\right)+b_L
\right)$. The learning algorithm $A$ here is the standard SGD on the expanded input MLP $g_{\theta}$.

Then the PDS learning condition is $\mathbb E_{S,A}
L_{(\mathcal D,f^\star)}
\bigl(
Tr_0(A(S))
\bigr)
\le
L_{(\mathcal D,f^\star)}(\mathcal H_{k})
+\varepsilon
=
\eta+\varepsilon.$

So our results with MLPs in \Cref{sec:pds-learning-internalization} imply generalized-PDS-learnability of the noisy parity task.

\paragraph{CoT mixture for learning semiautomata.}

For a semiautomaton $(Q,\Sigma, \delta)$ and horizon $T$, we take the input space $\mathcal X$ to be given by elements of the form $(\sigma_1,\dots, \sigma_T, q_0, \text{LEN}_{L})$, i.e. \[
\mathcal X = \Sigma^T \times Q \times \{\text{LEN}_2,\dots, \text{LEN}_{T+1} \}. 
\]
Let $u=(\sigma_1,\dots, \sigma_T,q_0)$. Given the answer length $K\in [T]$, if we have CoT state indices $1\le j_1 < \dots < j_K =T$ left after CoT removal, the answer is $a_K^*(u) = (q_{j_1}(u),\dots, q_{j_K}(u), \text{EOS})$ where $q_t(u) = \delta(q_{t-1}(u),\sigma_t)$ is the state trajectory $t=1,\dots, T$. The target function is $f^*(x) = f^*((u,\text{LEN}_{K})) = a_K^*(u)$, and the target distribution is the final state evaluation $\mathcal D = \mathcal D_{U} \otimes \delta_{\text{LEN}_2}$, where $\mathcal D_U$ is the distribution of $u=(\sigma_1,\dots, \sigma_T, q_0)$. The output space is $\mathcal Y= \cup_{L=2}^{T+1} (Q^{L-1} \times \{ \text{EOS} \})$. The loss is $l(\hat{a},a)=1_{\text{last}(\hat{a})\neq \text{last}(a)}$ where $\text{last}(a)$ is the last token before $\text{EOS}$ in $a$. The context distribution is given by \[
c_k(x) = (x, q_{j_1}(u),\dots, q_{j_K}(u),\text{EOS}).
\]
Let $\mathcal D_{K}$ be the distribution of $c_K(u)$ for $u\sim \mathcal D_U$ and let $\mathcal D_{\mathcal C}' = \frac{1}{T} \sum_{K=1}^{T} \mathcal \mathcal D_{K}$. Here $\mathcal M$ is the set of autoregressive transformers. The direct answering transport map is for $b=2$, $Tr_2(M)(x)$ given by the distribution of answers over outputs generated from $x=(u,\text{LEN}_2)$ and the full-CoT inference we take $b=T+1$ analogously. The learning algorithm $A$ here is standard next-token prediction training of a transformer on the mixture distribution. Thus this case also falls under the generalized-PDS learnability definition.


\subsection{Learning Through Internalization with MLPs: Experimental Details}\label{app:mlp-cot-details}

All notation follows \Cref{subsec:internalization-MLP}. In the experiments, we use the bipolar version of the parity task, which is equivalent to the binary domain: \(\oplus\) is represented by multiplication, the original data input has dimension \(d=25\), and the support size is \(k=12\). After relabeling the support, the target is \(y=\prod_{j=1}^{12}x_j\), with \(x\in\{\pm1\}^{25}\). The augmented input has \(d+k-1=36\) coordinates: the first \(25\) are data coordinates, and the remaining \(11\) are the hint coordinates \(z_1,\ldots,z_{11}\). 

For implementation, we index stages by \(r\in\{0,\ldots,k-1\}\), the number of removed hints. At stage \(r\), the \(r\) strongest/rightmost hints are removed and replaced by the neutral point ($0$ for binary input, $1$ for bipolar input)  
\[
    \tilde x^{(r)}
    =
    (x_1,\ldots,x_d,z_1,\ldots,z_{k-1-r},1,\dots,1).
\]
To make the task slightly harder and a better proxy for learning, instead of "zeroing out" the hints, we use independent random noise
\[
    \tilde x^{(r)}
    =
    (x_1,\ldots,x_d,z_1,\ldots,z_{k-1-r},\xi_{k-r},\ldots,\xi_{k-1}),
    \qquad
    \xi_i\sim \mathrm{Unif}\{\pm1\}.
\]
The plot in \Cref{fig:mixture} is shown with independent random noise.
We denote the corresponding input distribution by \(\mathcal D_r\). Thus \(\mathcal D_0\) is the full-hint distribution, while \(\mathcal D_{k-1}\) has no informative hints.

We compare three training conditions with matched total compute. The direct baseline trains only on \(\mathcal D_{k-1}\). The curriculum condition trains sequentially on \(\mathcal D_0,\mathcal D_1,\ldots,\mathcal D_{k-1}\), with equal budget per stage. The mixture condition trains on the single distribution $\mathcal D_{\mathrm{mix}}
    =
    \frac{1}{k}\sum_{r=0}^{k-1}\mathcal D_r$.
In the implementation of the mixture condition, each minibatch first samples \(r\sim\mathrm{Unif}\{0,\ldots,k-1\}\), and then all examples in that minibatch are sampled from \(\mathcal D_r\). Marginally, this is training on the fixed mixture distribution \(\mathcal D_{\mathrm{mix}}\), rather than on an ordered curriculum.

All models are evaluated on fresh samples from \(\mathcal D_{k-1}\), where every hint coordinate is independent noise. This measures performance when no informative hint is available, while keeping the implemented input dimension fixed at \(d+k-1\). When evaluating the formal subnetwork predictor from \Cref{subsec:internalization-MLP}, we instead evaluate \(h_\theta\) directly on \(x\in\{\pm1\}^{25}\).

We use fully connected ReLU MLPs with raw scalar output. The  configuration has depth \(2\) and width \(2048\), with approximately \(4.27\)M parameters. Here depth counts hidden ReLU layers: an input layer \(36\to w\), followed by depth-1 hidden layers \(w\to w\), and a final linear layer \(w\to 1\).

We train with mean-squared error against the \(\pm1\) target using plain SGD with learning rate \(3\times10^{-3}\), batch size \(64\), no momentum, no weight decay, and no learning-rate schedule. Training is online: every minibatch consists of freshly sampled examples. Each run uses \(600\) epochs with \(1000\) SGD steps per epoch, for \(6\times10^5\) total SGD steps. In the curriculum condition, this budget is divided evenly across the \(12\) stages, giving \(50\) epochs per stage. Accuracy is computed periodically on \(5000\) fresh evaluation samples, with a prediction counted as correct when its sign agrees with \(y\). Each condition is run with \(3\) random seeds.


\end{document}